\documentclass[journal,compsoc]{IEEEtran} 
\ifCLASSOPTIONcompsoc
  \usepackage[nocompress]{cite}
\else
  \usepackage{cite}
\fi
 \pdfoutput=1
\usepackage{array}
\usepackage{marginnote}
\usepackage{soul}
\usepackage{hyperref}

\usepackage{amsmath}
\usepackage{amssymb}
\usepackage{amsthm}
\usepackage{amsfonts}
\usepackage[linesnumbered,ruled,vlined]{algorithm2e}
\SetAlCapNameFnt{\small}
\usepackage{makecell}
\usepackage{algpseudocode}
\usepackage{verbatim}
\usepackage{multirow}
\usepackage{arydshln}
\usepackage{bm}
\usepackage{mathtools}
\usepackage{environ}
\usepackage{etoolbox}
\usepackage{dsfont}
\usepackage{latexsym} 
\usepackage{comment}
\usepackage{xcolor}
\usepackage{graphicx}
\usepackage[shortlabels]{enumitem}
\usepackage{diagbox}
\usepackage{hyphenat}
\usepackage[font={footnotesize,sf},subrefformat=simple,labelformat=simple,justification=centering]{subcaption}
\usepackage{indentfirst}
\usepackage{setspace}
\usepackage{booktabs}
\usepackage[numbers]{natbib}
\usepackage[capitalize]{cleveref}
\usepackage[font={footnotesize,sf},justification=justified,singlelinecheck=false]{caption}
\usepackage{stfloats}
\usepackage{nicefrac}
\usepackage{physics}
\usepackage{marginnote}
\usepackage{soul}

\definecolor{blue}{rgb}{0,0,1}
\definecolor{green}{rgb}{0,1,0}
\definecolor{orange}{rgb}{0.9,0.4,0}

\newcommand{\reva}[3]{%
{#2}%
}
\newcommand{\revb}[3]{%
{#2}%
}

\graphicspath{{./figures/}}

\DeclareMathOperator*{\argmax}{arg\,max}
\DeclareMathOperator*{\argmin}{arg\,min}

\newtheorem*{theorem*}{Theorem}

\newtheorem*{definition*}{Definition}
\newtheorem*{corollary*}{Corollary}

\newcommand*\widebar[1]{\overline{#1}}
\newcommand*{\smallsquare}{     
     \vcenter{\hbox{\scalebox{0.45}{$\;\mathbin{ \blacksquare }\;$}}}
}

\DeclarePairedDelimiterX{\kldivx}[2]{\big[}{\big]}{%
  #1\;\delimsize\|\;#2%
}
\newcommand{\kldiv}{D_{\text{KL}}\kldivx}

\newcolumntype{L}[1]{>{\raggedright\let\newline\\\arraybackslash\hspace{0pt}}m{#1}}
\newcolumntype{C}[1]{>{\centering\let\newline\\\arraybackslash\hspace{0pt}}m{#1}}
\newcolumntype{R}[1]{>{\raggedleft\let\newline\\\arraybackslash\hspace{0pt}}m{#1}}

\newlength{\myl}
\let\origequation=\equation
\let\origendequation=\endequation
\RenewEnviron{equation}{
  \settowidth{\myl}{$\BODY$}                       
  \origequation
  \ifdimcomp{\the\linewidth}{>}{\the\myl}
  {\ensuremath{\BODY}}                             
  {\resizebox{\linewidth}{!}{\ensuremath{\BODY}}}  
  \origendequation
}

\begin{document}

\title{Bayesian Unsupervised Disentanglement \\ of Anatomy and Geometry for Deep \\ Groupwise Image Registration}

\author{Xinzhe~Luo, Xin~Wang, Linda~Shapiro,~\IEEEmembership{Fellow,~IEEE,} Chun~Yuan, Jianfeng~Feng, Xiahai~Zhuang
  \IEEEcompsocitemizethanks{
    \IEEEcompsocthanksitem X. Luo and X. Wang contributed equally to this work.
    \IEEEcompsocthanksitem Corresponding author: Xiahai Zhuang.
    \IEEEcompsocthanksitem X. Luo and X. Zhuang are with the School of Data Science, Fudan University, Shanghai 200433, China.  
    E-mail: \{xzluo19, zxh\}@fudan.edu.cn. 
    \IEEEcompsocthanksitem X. Luo is now with the Department of Electrical and Electronic Engineering and I-X, Imperial College London, London SW7 2AZ, United Kingdom.
    E-mail: x.luo@imperial.ac.uk.
    \IEEEcompsocthanksitem X. Wang is with the Department of Electrical and Computer Engineering, University of Washington, Seattle 98195, United States.
    E-mail: xwang99@uw.edu.
    \IEEEcompsocthanksitem L. Shapiro is with the Paul G. Allen School of Computer Science and Engineering, University of Washington, Seattle 98195, United States.
    E-mail: shapiro@cs.washington.edu.
    \IEEEcompsocthanksitem C. Yuan is with the Department of Radiology and Imaging Sciences, University of Utah, and the Department of Radiology, University of Washington, Seattle 98109, United States.
    E-mail: cyuan@uw.edu.
    \IEEEcompsocthanksitem J. Feng is with the Shanghai Center for Mathematical Sciences and the Institute of Science and Technology for Brain-Inspired Intellengence, Fudan University, Shanghai 200433, China.
    Email: jffeng@fudan.edu.cn
    \IEEEcompsocthanksitem X. Luo and X. Zhuang were funded by the National Natural Science Foundation of China (grant No. 62372115) Shanghai Municipal Education Commission-Artificial Intelligence Initiative to Promote Research Paradigm Reform and Empower Disciplinary Advancement Plan (grant no. 24KXZNA13).
  }
  \thanks{Manuscript received ..; revised ..}
}

\IEEEtitleabstractindextext{%
\begin{abstract}
This article presents a general Bayesian learning framework for multi-modal groupwise image registration.
The method builds on probabilistic modelling of the image generative process, where the underlying common anatomy and geometric variations of the observed images are explicitly disentangled as latent variables.
Therefore, groupwise image registration is achieved via hierarchical Bayesian inference.
We propose a novel hierarchical variational auto-encoding architecture to realise the inference procedure of the latent variables, where the registration parameters can be explicitly estimated in a mathematically interpretable fashion.
Remarkably, this new paradigm learns groupwise image registration in an unsupervised closed-loop self-reconstruction process, sparing the burden of designing complex image-based similarity measures. 
The computationally efficient disentangled network architecture is also inherently scalable and flexible, allowing for groupwise registration on large-scale image groups with variable sizes.
Furthermore, the inferred structural representations from multi-modal images via disentanglement learning are capable of capturing the latent anatomy of the observations with visual semantics.
Extensive experiments were conducted to validate the proposed framework, including four different datasets from cardiac, brain, and abdominal medical images.
The results have demonstrated the superiority of our method over conventional similarity-based approaches in terms of accuracy, efficiency, scalability, and interpretability.

\end{abstract}

}

\maketitle
\IEEEdisplaynontitleabstractindextext
\IEEEpeerreviewmaketitle

\IEEEraisesectionheading{\section{Introduction}\label{sec:intro}}
\IEEEPARstart{G}{roupwise} image registration aims to find the hidden spatial correspondence that aligns multiple observations.
For medical images, when the observations reflect some common anatomy, their intrinsic structural correspondence, which can be independent of the multi-modal imaging acquisition protocol, is of particular interest.
However, conventional methods on multi-modal groupwise registration usually rely on intensity-based similarity measures to iteratively optimise the spatial transformations.
For instance, multivariate joint entropy and mutual information were proposed as groupwise similarity measures in \cite{conference/miccai/boes1999,journal/prl/studholme2004,journal/tip/spiclin2012,journal/tpami/wachinger2012}, followed by more computationally favourable template-based approaches \cite{journal/media/polfliet2018,journal/tpami/luo2022}.
Nevertheless, devising proper similarity measures and choosing the correct registration hyperparameters for heterogeneous medical images can be tedious and challenging.
The high computational burden and the instability in registration accuracy may also prevent real-world applications of conventional groupwise registration algorithms.

In this work, we seek to establish a new unified and interpretable learning framework for large-scale unsupervised multi-modal groupwise registration.
Inspired by recent progress in disentangled representation learning \cite{journal/tpami/bengio2013,journal/arxiv/higgins2018}, we propose to disentangle the underlying common anatomy and geometric variations from the observed images, which can be regarded as reverting the data generating process of the observations.
To this end, a probabilistic generative model is constructed, where the common anatomy and spatial transformations are disentangled as latent variables.
Thus, the problem of groupwise registration (recovery of the intrinsic structural correspondence) devolves to the estimation of the latent posterior distribution, which is then solved via variational inference.
Besides, since unsupervised learning of disentangled representation is impossible without proper inductive biases \cite{conference/icml/locatello2019,conference/aistats/khemakhem2020}, we have designed a novel hierarchical variational auto-encoding architecture to realise the inference procedure of latent variables, where the decoder network complies with the equivariance assumption of the imaging process.
Accordingly, groupwise registration is learnt in an unsupervised closed-loop self-reconstruction process:
i) the encoder extracts the intrinsic structural representations from the multi-modal observations, based on which the common anatomy and the spatial correspondence are estimated;
ii) the decoder emulates the equivariant image generative process from the common anatomy, and reconstructs the observations via the inverse spatial transformations.

To improve the efficiency and scalability of groupwise registration,
particularly in achieving deep groupwise registration with variable (and potentially
very large) group sizes for the first time, we make the following contributions:
\begin{itemize}
  \item We propose a new learning paradigm for multi-modal groupwise registration based on Bayesian inference and disentangled representation learning. 
  Remarkably, we can achieve registration in an unsupervised closed-loop self-reconstruction process, which spares the burden of designing complex image-based similarity measures.
  \item We propose a novel hierarchical variational auto\hyp{}encoding architecture for the joint inference of latent variables.
  The network is able to reveal the underlying structural representations from the observations with visual semantics, based on which the intrinsic structural correspondence can be explicitly estimated in a mathematically interpretable fashion.
  \item Based on the equivariance assumption and the decomposition of the symmetry group actions on the observational space, we prove that under certain conditions, our model can identify the desired registration parameters, which constitutes the theoretical underpinnings of the established framework.
  \item Our registration model, while trained using small image groups, can be readily adapted to large-scale and variable-size test groups, significantly enhancing its computational efficiency and applicability.
  \item We validated the proposed framework on four publicly available medical image datasets, demonstrating its superiority over similarity-based methods in terms of accuracy, efficiency, scalability, and interpretability.
\end{itemize}

The remainder of the article is organised as follows.
\cref{sec:related_work} discusses the related work to this study.
\cref{sec:bayes_reg} elucidates the proposed Bayesian inference framework for groupwise registration.
\cref{sec:bayes_network} describes the proposed hierarchical variational auto-encoder architecture for learning and estimation of the latent variables.
\cref{sec:experiment} presents the experimental setups and evaluation results of our method on four different public datasets.
\cref{sec:discussion} notes implications of the proposed framework and concludes the study.
For conciseness, in \cref{tab:notation} we specify the main mathematical symbols used in the rest of this paper.

\begin{table*}[t]
  \small
  \centering
  \caption{Definition of the main mathematical symbols used in this paper.}
  \begin{tabular}{p{0.18\columnwidth}p{0.75\columnwidth}|p{0.18\columnwidth}p{0.75\columnwidth}}
    \toprule
    Symbol & Description & Symbol & Description \\
    \midrule
    $L$ & the number of levels of latent variables & $N$ & the number of images in the group \\
    $\bm{Z}$ & random variable of the common anatomy & $\bm{U}$ & random vector of the image group $\{U_j\}_{j=1}^N$ \\
    $\Omega$ & the coordinate space of the common anatomy & $\Omega_j$ & the $j$-th image coordinate space \\
    $\bm{\phi}$ & the set of diffeomorphic maps $\bm{\phi}\triangleq\{\phi_j\}_{j=1}^N$ & $\phi_j$ & the spatial transformation from $\Omega$ to $\Omega_j$ \\ 
    $\bm{z}$ & the latent common structural representation & $\bm{u}$ & an observed sample of the image group \\
    $\bm{z}^l$ & the latent common structural representation at the $l$-th level & $\bm{v}^l$ & the latent stationary velocity fields registering the image group at the $l$-th level \\ $q^*(\bm{z}^l\,|\,\bm{u},\bm{v})$ & the geometric mean variational distribution of $\bm{z}^l$ given $\bm{u}\circ\bm{\phi}\triangleq\{u_j\circ\phi_j\}_{j=1}^N$ & $\bm{v}_j^l$ & the stationary velocity field of the $j$-th diffeomorphism at the $l$-th level, $\bm{v}^l=\{\bm{v}_j^l\}_{j=1}^N$ \\ 
    $q_j^{\diamond}(\bm{z}^l\,|\, u_j,\bm{v}_j)$ & the $j$-th single-view variational distribution of $\bm{z}^l$ given $u_j\circ\phi_j$  & $\bm{\mu}_{\bm{v},j}^l,\bm{\Sigma}_{\bm{v},j}^l$ & parameters of the Gaussian variational distribution $q(\bm{v}_j^l\mid \bm{u},\bm{v}^{<l})$ \\
    $\widetilde{q}_{j}(\bm{z}^l\,|\, u_j)$ & the $j$-th single-view variational distribution of $\bm{z}^l$ given $u_j$  & $\bm{v}_j^+$ & the total velocity field aggregating all $\{\bm{v}_j^l\}_{l=1}^L$ upsampled to the top level \\
    $\bm{\psi}$ & parameters of the variational model $\{\bm{\psi}_j\}_{j=1}^N$ & $\bm{\theta}$ & parameters of the generative model \\
    \bottomrule
  \end{tabular}
  \label{tab:notation}
\end{table*}

\section{Related Work}\label{sec:related_work}
In this section, we review related work to the proposed framework, including similarity-based and deep feature-based approaches to (groupwise) image registration, as well as literature on multi-modal representation learning that facilitates disentanglement of latent variables from multiple modalities. 
We hope to expose the connections between them that motivate our established framework.

\subsection{Groupwise Image Registration}
Let $\bm{U}=\{U_j\}_{j=1}^N$ be the random vector representing the image group and $\bm{u}=\{u_j\}_{j=1}^N$ an observed sample of $\bm{U}$, with $U_j:\mathds{R}^d\supset\Omega_j\rightarrow\mathds{R}$ the intensity mapping, $\Omega_j$ the image domain, and $d$ the dimensionality.
In the classical pattern matching and computational anatomy regime \cite{book/oxford/grenander1993,journal/qam/grenander1998,journal/ijcv/trouve1998,journal/qam/dupuis1998,journal/ijcv/miller2001}, the observed homogeneous image group is considered as samples from the orbit of a transformation group $\mathcal{G}$ acting on a deformable template $U_0$ (\emph{a.k.a.} atlas), namely
\begin{equation}
  \bm{U} \subset \mathcal{G}\cdot U_0\triangleq \{U_0\circ\phi_j^{-1}:\phi_j\in\mathcal{G}\},
\end{equation}
where the transformation $\phi_j^{-1}:\Omega_j\rightarrow\Omega$ maps spatial locations in the image domain to a common coordinate space $\Omega\subset\mathds{R}^d$, and $\mathcal{G}$ is often taken as the group of diffeomorphisms.
This concept has motivated the development of population averaging methods that estimate the deformations $\mathcal{G}$ and the template $U_0$ simultaneously, based on the observations $\bm{U}$ \cite{journal/cviu/guimond2000,journal/ni/avants2004,journal/ni/joshi2004,journal/ni/christensen2006,conference/miccai/bhatia2007,journal/jrssb/allassonnire2007,journal/ni/ma2008,journal/ni/geng2009,journal/media/zhang2015,conference/neurips/dalca2019,conference/cvpr/ding2022}.
The resultant transformations encode the structural variability and the template provides a statistical representative of the images.
In particular, \emph{a priori} guess of the template as one of the observed images was proposed in \cite{journal/cviu/guimond2000,journal/ni/avants2004,journal/ni/christensen2006}, with the template updated by the average deformation in each iteration.
\citet{journal/ni/joshi2004} and \citet{conference/miccai/bhatia2007} proposed using the intensity mean image as a template, avoiding reference selection by estimating transformations from the common space.
\citet{journal/jrssb/allassonnire2007} extended the setup to a Bayesian framework, with the template modelled as a linear combination of continuous kernel functions.
Recently, \citet{conference/neurips/dalca2019,conference/cvpr/ding2022} proposed learning-based template estimation methods, using spatial transformation networks to predict diffeomorphisms parameterized by stationary velocity fields \cite{conference/miccai/arsigny2006,journal/ni/ashburner2007}.
Nevertheless, groupwise registration based on the deformable template assumption could not readily accommodate the multi-modal nature of medical images.

\emph{Multi-modality} can introduce additional complexity, as the observed images are not related to an intensity template directly through spatial transformations.
In this case, based on the common anatomy assumption of multi-modal medical images, we can instead assume an \emph{anatomical} or structural representation $\bm{Z}$ from which each observed image is generated through the composition of an imaging functional $f_j$ and a spatial transformation $\phi_j^{-1}$, \emph{i.e.},
\begin{equation}\label{eq:img_generate}
  U_j = f_j(\bm{Z})\circ\phi_j^{-1}.
\end{equation}
More importantly, the imaging functional is assumed to be \emph{spatially equivariant} w.r.t. $\phi_j^{-1}$, a condition motivated by the demand to learn equivariant image features \cite{conference/icann/hinton2011,conference/icann/kivinen2011,conference/cvpr/schmidt2012,journal/tpami/qi2020}, which writes
\begin{equation}
  f_j(\bm{Z})\circ\phi_j^{-1} = f_j(\bm{Z}\circ\phi_j^{-1}),\quad \forall\,\phi_j\in \mathcal{G}.
\end{equation}

Inspired by inter-modality pairwise registration based on the joint intensity distribution (JID) or particularly the mutual information (MI) \cite{journal/ijcv/viola1997,journal/tmi/maes1997,conference/miccai/leventon1998,journal/pr/studholme1999,journal/ijist/roche2000}, initial attempts to realise \emph{multi-modal groupwise registration} focused on generalising this information-theoretic approach directly to high-dimensional cases \cite{conference/miccai/boes1999,journal/prl/studholme2004,conference/ipmi/zhang2005,journal/tip/spiclin2012,journal/tpami/wachinger2012}.
Later, concerning the curse of dimensionality in estimating high-dimensional JIDs, template-based groupwise registration revives through probabilistic modelling of the image generative process in \cref{eq:img_generate} \cite{journal/media/lorenzen2006,journal/tip/orchard2009,journal/ni/blaiotta2018,journal/media/polfliet2018,journal/tpami/zhuang2019,journal/tpami/luo2022}.
These methods assumed that the JID is modelled by the marginalisation of the joint distribution between the latent template and the observed images.
Thus, the spatial transformations were determined by maximum likelihood estimators (MLEs) while the template was predicted via maximum a posteriori (MAP).
We direct interested readers to our previous work \cite{journal/tpami/luo2022} for detailed elucidation of this maximum-likelihood perspective and its connection to information-theoretic metrics.

Over the past years, the shift from optimization to learning-based image registration \cite{journal/media/hu2018,journal/media/de2019,journal/tmi/balakrishnan2019,journal/media/dalca2019} has also motivated the development of learning\hyp{}based multi-modal groupwise registration methods which predict the desired spatial transformations using neural network estimation \cite{conference/ipmi/che2019,conference/miccai/luo2020}.
However, this deep-learning approach may inherent the same limitation of scalability from its optimisation-based counterpart.
That is, the trained network for a fixed input size can hardly be utilised to register image groups of variable sizes.
This hinders the potential of learning\hyp{}based frameworks on large-scale groupwise registration.

\subsection{Deep Feature-Based Image Registration}
A closely related line of research to our work is deep feature-based image registration \cite{conference/neurips/pielawski2020,conference/miccai/liu2021,conference/miccai/moyer2021,conference/miccai/siebert2021,journal/media/czolbe2023,journal/tip/deng2023}.
These methods usually adopt a two-stage pipeline: in the first stage a feature extraction network is pre-trained on a surrogate task, be it segmentation \cite{conference/miccai/siebert2021,journal/media/czolbe2023}, auto-encoding \cite{journal/media/czolbe2023}, or contrastive learning \cite{conference/neurips/pielawski2020,conference/miccai/liu2021}, followed by an instance optimisation procedure using conventional similarity measures with the learnt features in the second stage.
They may have the advantage of being interpretable and capable of multi-modal registration based on modality\hyp{}invariant features \cite{conference/neurips/pielawski2020,conference/miccai/liu2021,conference/miccai/siebert2021}, compared to previous end-to-end counterparts \cite{journal/media/de2019,journal/tmi/balakrishnan2019} that make predictions using a black-box network.
However, the two-stage prediction pipeline and the requirement for ground-truth labels (segmentation annotation or aligned image pairs) limit their applicability.
In contrast, \citet{conference/ipmi/qin2019} proposed using the multimodal unsupervised image-to-image translation framework \cite{conference/eccv/huang2018} to reduce the multi-modal registration problem to a mono-modal one by disentangling shape and appearance representations.
More recently, \citet{journal/tip/deng2023} proposed a unified framework where explicit modality-invariant feature extraction and unsupervised registration are learnt via an interpretable optimisation problem.

\subsection{Multi-Modal Representation Learning}
The underlying structural representations that facilitate the estimation of spatial correspondence can be regarded as some latent embedding, which encodes the information correlating the multi-modal observations. 
This viewpoint, in the context of heterogeneous data, brings the need for multi-modal representation learning which exploits the commonality and complementarity of multiple modalities, and the demand for data alignment which identifies the relation and correspondence between different modalities \cite{journal/tpami/baltruvsaitis2018}. 

Modern probabilistic generative models \cite{conference/iclr/kingma2014,conference/neurips/goodfellow2014,conference/icml/van2016} construct a general framework for representation learning, since the data formation procedure to which they adapt is optimally suited for an unsupervised manner, and are powerful in striking a balance between fitting and generalisability taking into account the uncertainty in characterising the data and latent distributions. 

One family of methods within this scope is variational autoencoders (VAEs) \cite{conference/iclr/kingma2014}, where the paradigm is to approximate the intractable latent distribution by a learnable variational posterior during the optimisation of an evidence lower bound (ELBO) of the log likelihood of observed data. 
It provides an information-theoretic way to establish the foundational underpinnings of data encoding for downstream tasks \cite{conference/iclr/alemi2017}.
The multi-modal variants of VAEs additionally imbue the capability of retrieving modality-invariant representation by factorising it into single-view posteriors (\emph{a.k.a.} experts) inferred from each modality. 
A typical factorisation specifies the combination of the experts explicitly (\emph{e.g.} mixtures \cite{conference/neurips/shi2019}, products \cite{conference/neurips/wu2018}, mixtures of products \cite{conference/iclr/sutter2021}) or implicitly (\emph{e.g.} cross-modal generation \cite{conference/neurips/tu2022}), compelling the summarisation of the common information from any individual modality. 

These approaches, nonetheless, all necessitate originally aligned inputs, such as the image and sentence that describe the same object \cite{conference/neurips/shi2019}, or multi-modal images that are well registered \cite{conference/miccai/havaei2016,conference/miccai/dorent2019}, which limits their applicability to distorted data. 
Besides, they often adopt a reductionist stance by assuming an overly simplistic prior. (\emph{e.g.}, the standard Gaussians or Laplacians), failing to respect the structure of latent manifold induced by real-world high-dimensional images, which is crucial for mitigating off-track regularisation and vanishing latent dimensions \cite{conference/neurips/hoffman2016,conference/aistats/tomczak2018}, as well as for ensuring suitable decomposition of latent encodings \cite{conference/icml/mathieu2019}. 
The mixture-based factorisation of the posterior also suffers from inevitable gap between the ELBO and the true likelihood, leading to undesirable suboptimal results \cite{conference/iclr/daunhawer2022}. 

In contrast, our framework embraces advancements in tackling all the limitations above: It allows learning from distorted images in a unified manner, through explicitly endowing the representations of spatial transformations and the common anatomy with a disentangled nature. 

\section{Bayesian Groupwise Registration}\label{sec:bayes_reg}
\begin{figure*}[t]
  \centering
  \begin{subfigure}{0.27\textwidth}
    \centering
    \includegraphics[width=0.55\textwidth]{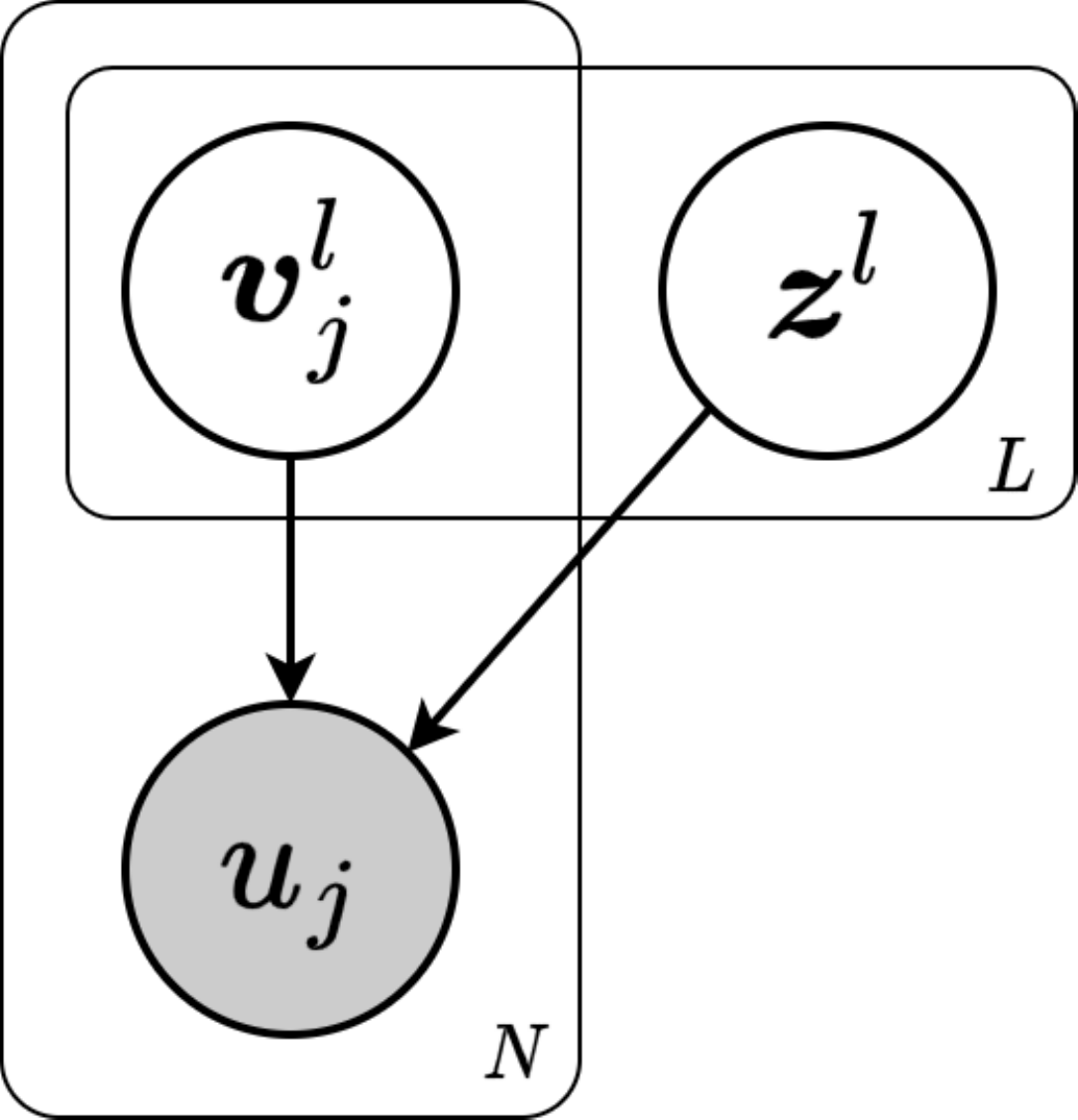}
    \caption{Generative model.}
    \label{fig:graphical_model}
  \end{subfigure}
  \begin{subfigure}{0.24\textwidth}
    \centering
    \includegraphics[width=0.85\textwidth]{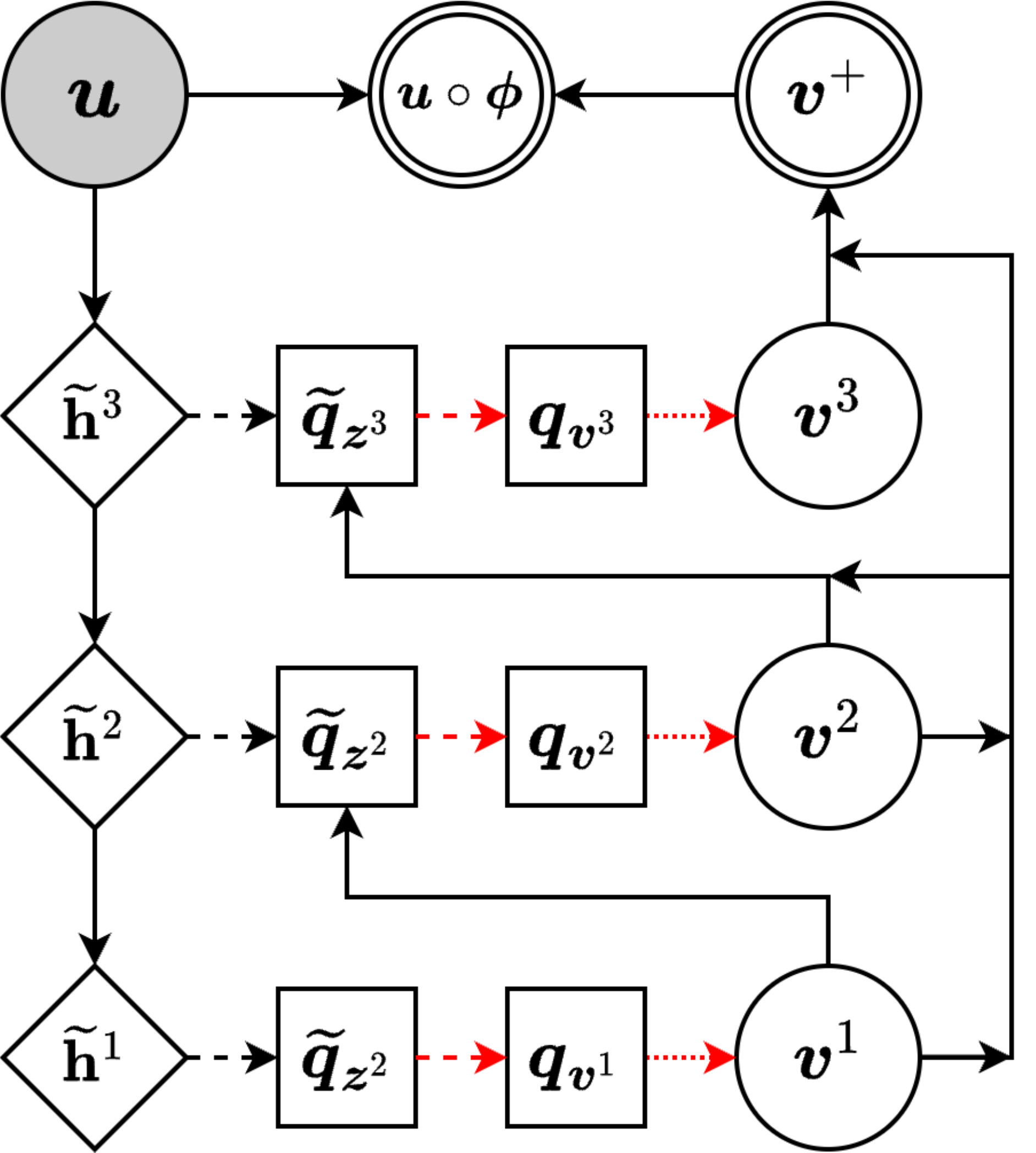}
    \caption{Inference steps \#1.}
    \label{fig:inference_steps1}
  \end{subfigure}
  \begin{subfigure}{0.23\textwidth}
    \centering
    \includegraphics[width=0.8\textwidth]{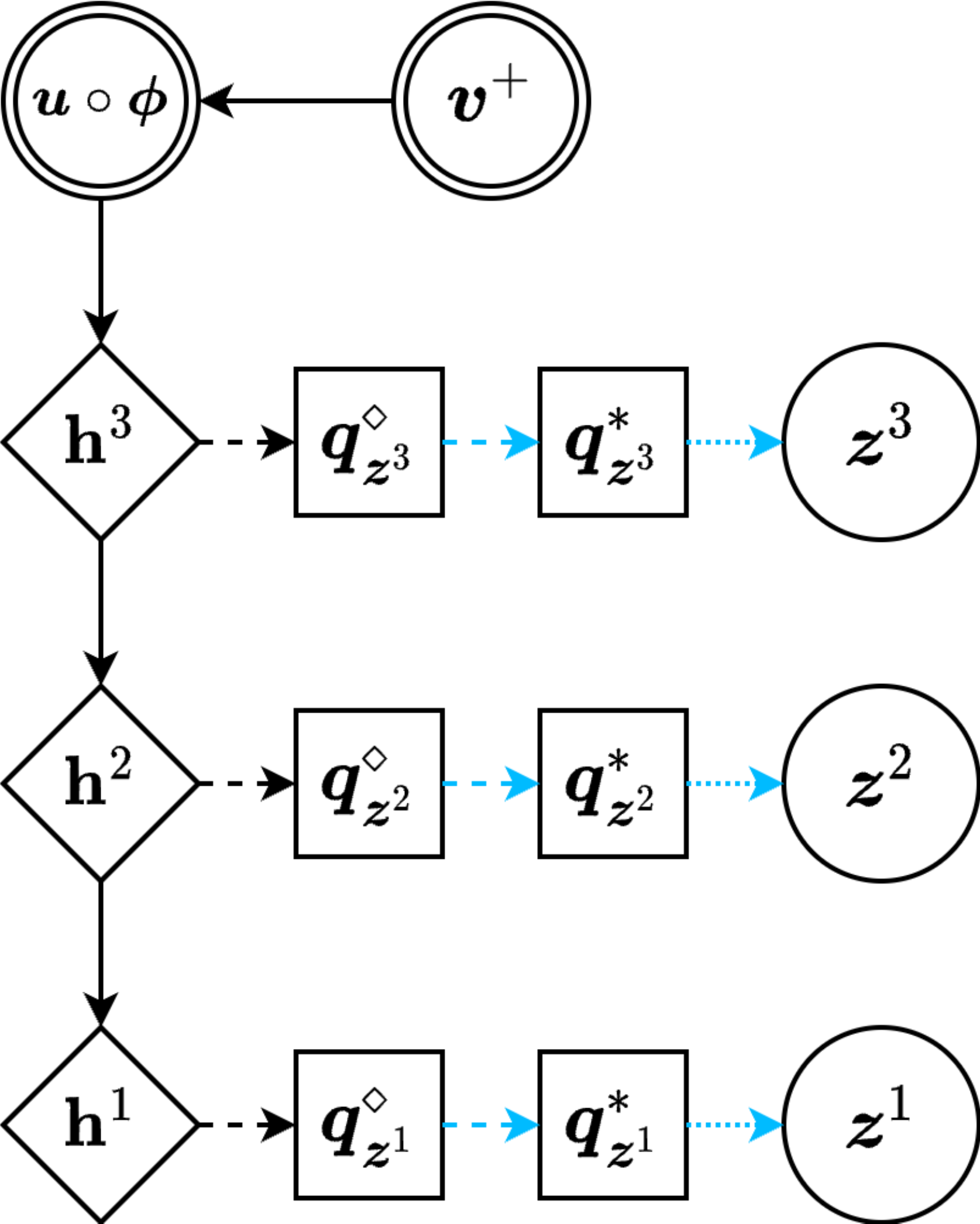}
    \caption{Inference steps \#2.}
    \label{fig:inference_steps2}
  \end{subfigure}
  \begin{subfigure}{0.23\textwidth}
    \centering
    \includegraphics[width=0.73\textwidth]{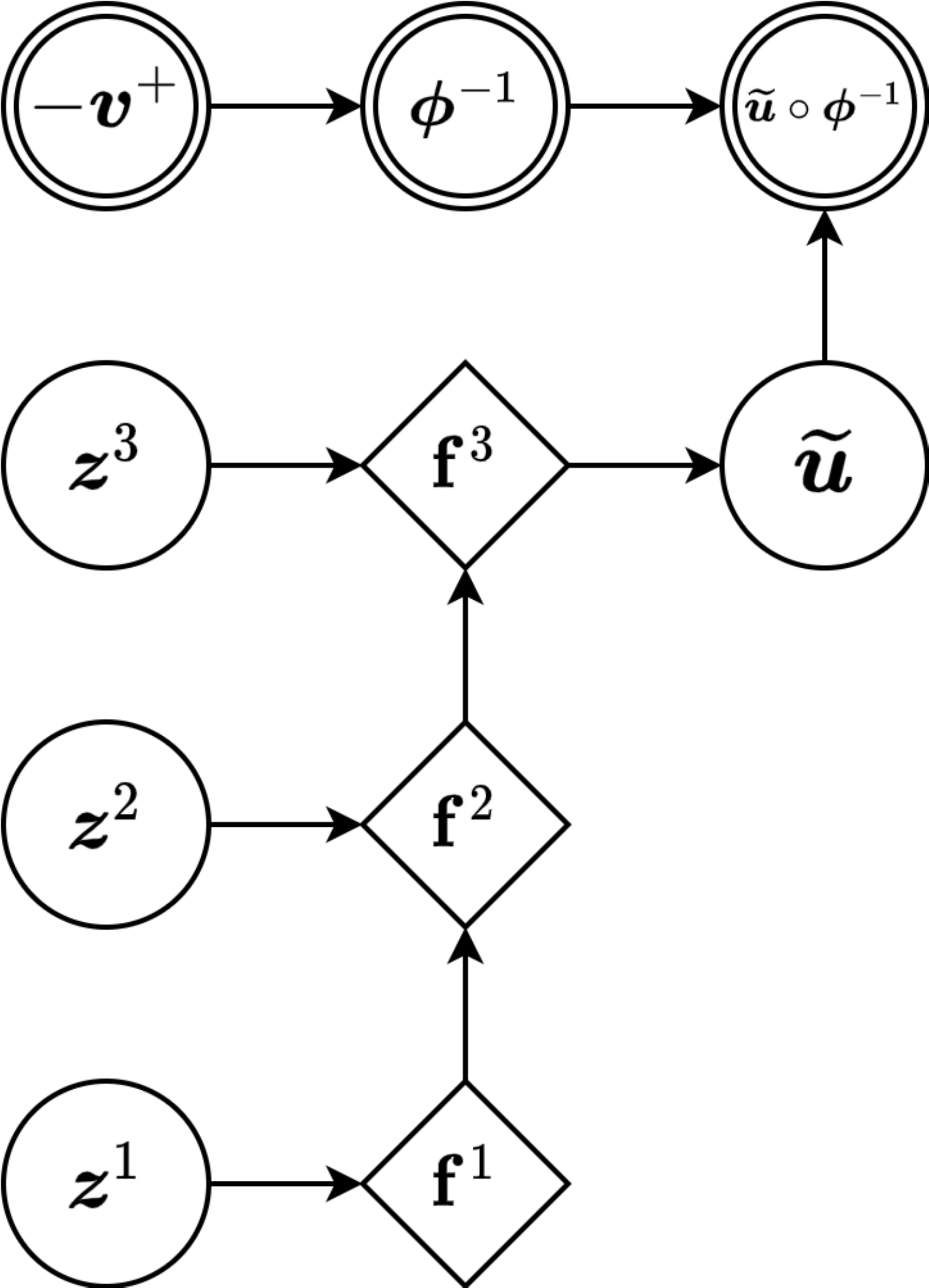}
    \caption{Generation steps.}
    \label{fig:generation_steps}
  \end{subfigure}
  \caption{The proposed hierarchical framework for Bayesian groupwise registration (3-layer example). 
  Random variables are in circles, deterministic variables are in double circles, and observed variables are shaded.
  Diamonds denote network feature maps, and squares represent variational distributions.
  (a) Probabilistic graphical model of the generative process. 
  (b) Inference steps \#1 that predict the hierarchical velocity fields, where we denote $\widetilde{\bm{q}}_{\bm{z}^l}\triangleq\{\widetilde{q}_j(\bm{z}^l\mid u_j;{\bm{\psi}_j})\}_{j=1}^N$ and $\bm{q}_{\bm{v}^l}\triangleq\{q(\bm{v}_j^l\mid\bm{u},\bm{v}^{<l};{\bm{\psi}})\}_{j=1}^N$.
  (c) Inference steps \#2 that predict the common structural representations base on the warped images, where we denote $\bm{q}_{\bm{z}^l}^{\diamond}\triangleq\{q_j^{\diamond}(\bm{z}^l\mid u_j,\bm{v}_j;{\bm{\psi}_j})\}_{j=1}^N$ and $\bm{q}_{\bm{z}^l}^*\triangleq q^*(\bm{z}^l\mid\bm{u},\bm{v};{\bm{\psi}})$.
  (d) Generation steps that reconstruct the original images.
  Note that the inference and generation steps form a closed-loop self-reconstruction process.}
  \label{fig:graph}
\end{figure*}

The fundamental principle driving the proposed approach is to disentangle the underlying common anatomy and geometric variations as latent representations \emph{w.r.t.} the decomposition of two symmetry groups acting independently in the observational image space.
More details on the underlying connection between disentangled representations and symmetry groups can be found in Section E.1 of the Supplementary Material.
To this end, we propose to formulate the estimation of these latent variables as a problem of Bayesian inference.
Specifically, let $\bm{u}=(u_j)_{j=1}^N$ be a sample of the random vector $\bm{U}$ of the image group.
From \cref{eq:img_generate}, each image $u_j$ is generated from the latent variable $\bm{z}$ representing the corresponding common anatomy, and the diffeomorphic transformation $\phi_j^{-1}$.
To ensure its invertibility, we assume that the transformations $\bm{\phi}\triangleq(\phi_j)_{j=1}^N$ is parameterized by stationary velocity fields $\bm{v}=(\bm{v}_j)_{j=1}^N$ \cite{conference/miccai/arsigny2006,journal/ni/ashburner2007} such that $\phi_j=\exp(\bm{v}_j)$, and
\begin{equation}
  \frac{\partial}{\partial t}\phi_j(\bm{\omega},t)=\bm{v}_j(\phi_j(\bm{\omega},t)),\quad \forall\,\bm{\omega}\in\Omega,\ t\in [0,1].
\end{equation}
The diffeomorphism of the transformations is further imposed by assigning proper priors to the velocity fields.

\subsection{Hierarchical Bayesian Inference}
In practice, registration is often performed at multiple levels to facilitate convergence of the algorithm \cite{conference/miccai/schnabel2001,journal/tmi/sotiras2013}.
Therefore, we express the latent variables with $L$ hierarchical levels, \emph{i.e.}, $\bm{z}=(\bm{z}^l)_{l=1}^L$ and $\bm{v}_j=(\bm{v}_j^l)_{l=1}^L$, in which higher levels indicate finer resolutions.
Thus, the graphical model of the image generative process can be described as \cref{fig:graphical_model}.
Note that while different levels of the latent variables are assumed to be independent, the inference procedure is performed \emph{hierarchically}.

Since we are going to parameterize the likelihood $p(\bm{u}\mid \bm{z},\bm{v};\bm{\theta})$ by neural networks in the context of Bayesian deep learning, exact inference and parameter estimation become intractable \cite{conference/iclr/kingma2014}.
Therefore, we resort to variational inference (VI) to approximate the maximum likelihood.
The objective function of VI is the evidence lower bound (ELBO) of the log-likelihood function, \emph{i.e.},
\begin{equation}\label{eq:ELBO}
  \begin{aligned}
    \mathcal{L}(\bm{\theta},\bm{\psi}\mid\bm{u}) &\triangleq
    \begin{aligned}[t]
      & \mathbb{E}_{q(\bm{z},\bm{v}\mid\bm{u};\bm{\psi})}[\log p(\bm{u}\mid\bm{z},\bm{v};\bm{\theta})] \\
      & - \kldiv{q(\bm{z},\bm{v}\mid\bm{u};\bm{\psi})}{p(\bm{z})p(\bm{v})},
    \end{aligned}
    \\
    &\leq\log p(\bm{u};\bm{\theta}) \triangleq \ell(\bm{\theta}\mid\bm{u}) 
  \end{aligned}
\end{equation}
where $q(\bm{z},\bm{v}\mid\bm{u};\bm{\psi})$ is the variational distribution, an approximation to the intractable true posterior $p(\bm{z},\bm{v}\mid\bm{u};\bm{\theta})$; $\bm{\psi}$ and $\bm{\theta}$ are the variational and generative parameters respectively. 
The expectation over the likelihood can be estimated by Monte-Carlo sampling \cite{conference/iclr/kingma2014}.

To simplify the KL divergence term, we first write 
\begin{equation}
\begin{aligned}
  &\kldiv{q(\bm{z},\bm{v}\mid\bm{u};\bm{\psi})}{p(\bm{z})p(\bm{v})} \\
  &= \mathbb{E}_{q(\bm{v}\mid\bm{u};\bm{\psi})}\Big[\kldiv{q(\bm{z}\mid\bm{u},\bm{v};\bm{\psi})}{p(\bm{z})}\Big] + \kldiv{q(\bm{v}\mid\bm{u};\bm{\psi})}{p(\bm{v})},
\end{aligned}
\end{equation}
where $q(\bm{v}\mid\bm{u};{\bm{\psi}})=\prod_{l=1}^L q(\bm{v}^l\mid\bm{u},\bm{v}^{<l};{\bm{\psi}})$, with $\bm{v}^l\triangleq (\bm{v}_j^l)_{j=1}^N$, and $\bm{v}^{<l}$ denotes the velocity fields in levels lower than $l$.
This factorisation of velocity fields can be illustrated by the inference steps \#1 in \cref{fig:inference_steps1}.
Besides, we assume that the multi-level variational posteriors of the velocity fields factorizes as
\begin{equation}\label{eq:factorize_v}
  q(\bm{v}^l\mid\bm{u},\bm{v}^{<l};{\bm\psi})=\prod_{j=1}^N q(\bm{v}_j^l\mid\bm{u},\bm{v}^{<l};{\bm\psi}).
\end{equation}
Then, the KL \emph{w.r.t.} the velocity fields can be decomposed as 
\begin{equation}
  \begin{aligned}
    &\kldiv{q(\bm{v}\mid\bm{u};\bm{\psi})}{p(\bm{v})} \\
    =\ & \sum_{j=1}^N\sum_{l=1}^{L}\mathbb{E}_{q(\bm{v}_j^{<l}\mid\bm{u};{\bm\psi})}\left[\kldiv{q(\bm{v}_j^l\mid\bm{u},\bm{v}^{<l};{\bm{\psi}})}{p(\bm{v}_j^l)}\right].
  \end{aligned}
\end{equation}
Likewise, the KL \emph{w.r.t.} the common anatomy can be simplified into 
\begin{equation}
  \begin{aligned}
    &\kldiv{q(\bm{z}\mid\bm{u},\bm{v};\bm{\psi})}{p(\bm{z})} \\
    =\ & \sum_{l=1}^{L}\mathbb{E}_{q(\bm{z}^{<l}\mid\bm{u},\bm{v};{\bm\psi})}\left[\kldiv{q(\bm{z}^l\mid\bm{u},\bm{v};{\bm\psi})}{p(\bm{z}^l)}\right],
  \end{aligned}
\end{equation}
where we have assumed that 
\begin{equation*}
  q(\bm{z}^l\mid\bm{u},\bm{v},\bm{z}^{<l};{\bm\psi})=q(\bm{z}^l\mid\bm{u},\bm{v};{\bm\psi}),
\end{equation*}
\emph{i.e.}, the common anatomy at level $l$ can be inferred directly from $\bm{u}$ and $\bm{v}$ without referring to lower-level representations $\bm{z}^{<l}$, which is illustrated by \cref{fig:inference_steps2}.

\begin{figure}[t]
  \centering
  \includegraphics[width=\linewidth]{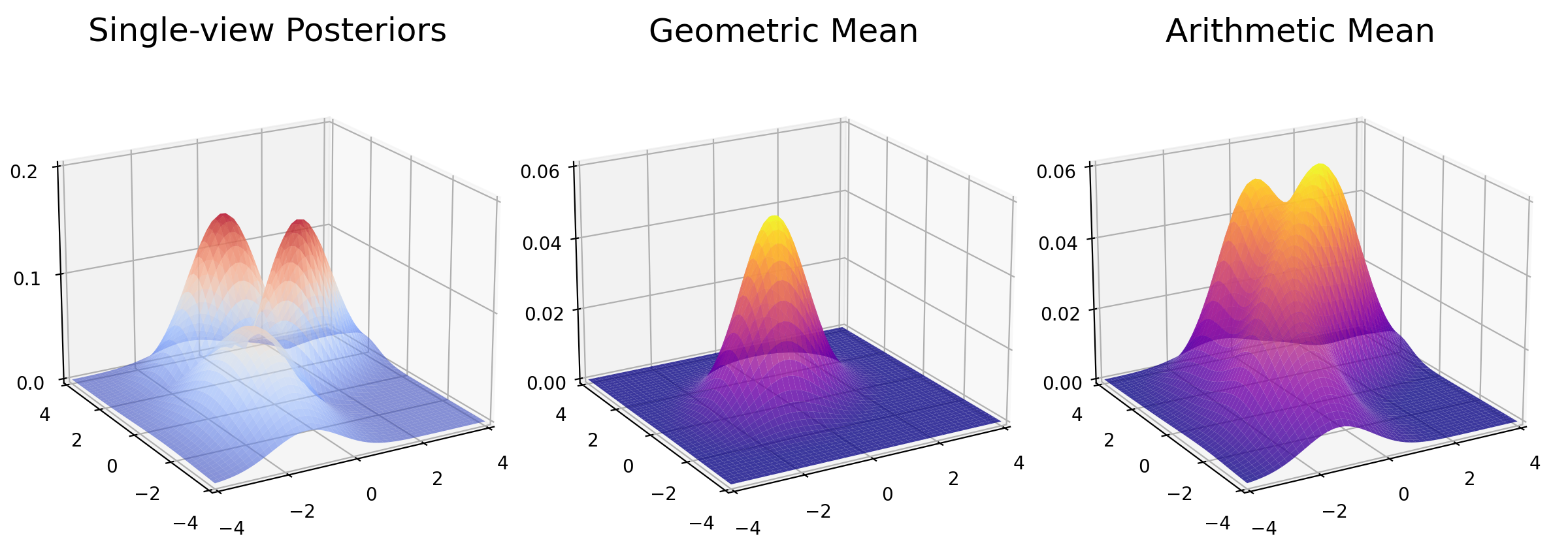}
  \caption{An example of the geometric and arithmetic mean for the single-view Gaussian posterior distributions.}
  \label{fig:geo_arith_mean}
\end{figure}

\begin{figure*}[!b]
  \normalsize
  \hrulefill
  \vspace*{4pt}
  \begin{equation}\label{eq:KL_decompose}
    \begin{aligned}
      \kldiv{q(\bm{z},\bm{v}\mid\bm{u};\bm{\psi})}{p(\bm{z})p(\bm{v})} =& \ \mathbb{E}_{q(\bm{v}\mid\bm{u};\bm{\psi})}\left\{ \sum_{l=1}^{L}\mathbb{E}_{q(\bm{z}^{<l}\mid\bm{u},\bm{v};{\bm\psi})}\left[\kldiv{q(\bm{z}^l\mid\bm{u},\bm{v};{\bm\psi})}{p(\bm{z}^l)}\right] \right\} \quad\quad\text{(i)} \\
      &\ + \sum_{j=1}^N\sum_{l=1}^{L}\mathbb{E}_{q(\bm{v}_j^{<l}\mid\bm{u};{\bm\psi})}\left[\kldiv{q(\bm{v}_j^l\mid\bm{u},\bm{v}^{<l};{\bm{\psi}})}{p(\bm{v}_j^l)}\right] \qquad\qquad
      \text{(ii)}
    \end{aligned}
  \end{equation}
\end{figure*}

\cref{eq:KL_decompose} summarises the decomposition of KL divergence \emph{w.r.t.} the latent variables.
Note that for simplicity, we have defined $q(\bm{v}_j^{<1}\mid\bm{u};{\bm{\psi}})=q(\bm{z}^{<1}\mid\bm{u},\bm{v};{\bm\psi})\triangleq 1$.
As will be elucidated in the following subsections, the key idea behind \cref{eq:KL_decompose} is: the overall KL divergence is decomposed \emph{w.r.t.} (i) the common structural representations $\bm{z}$, and (ii) the velocity fields $\bm{v}$.
The former is used to measure the intrinsic structural dissimilarity among the observed multi-modal images, while the latter serves as smoothness regularisation to enforce diffeomorphism of the transformations.

\subsection{Intrinsic Distance over Structural 
Representations}\label{sec:intrinsic_distance}

As mentioned above, the KL divergence \emph{w.r.t.} the common structural representations is used to measure the dissimilarity of the observed images for registration. 
Namely, instead of using similarity measures in the image space, we propose learning an intrinsic distance over the structural representations corresponding to the observed images.

We first extract the single-view structural representation maps corresponding to each individual image, represented by the distributions $q_j^{\diamond}(\bm{z}\mid u_j,\bm{v}_j;{\bm{\psi}_j})$, $j=1,\dots,N$, which is predicted from the inference steps \#2 (ref. \cref{fig:inference_steps2}). 
Note that $\bm{\psi}_j$ denotes the modality-specific variational parameters for each image, which are the same for images of the same modality.
Then we define the variational posterior of the common anatomy as the \emph{geometric mean} of these single-view posteriors, \emph{i.e.},
\begin{equation}\label{eq:geo_mean}
  q(\bm{z}^l\mid\bm{u},\bm{v};{\bm{\psi}}):= q^*(\bm{z}^l\mid\bm{u},\bm{v};{\bm{\psi}})\propto\left[\prod_{j=1}^{N}q_j^{\diamond}(\bm{z}^l\mid u_j,\bm{v}_j;{\bm{\psi}_j})\right]^{\nicefrac{1}{N}},
\end{equation}
which captures the common information among the multi-modal images.
On the other hand, the prior distribution of the common anatomy is given by the \emph{arithmetic mean} of the single-view posteriors, \emph{i.e.},
\begin{equation}
  p(\bm{z}^l;{\bm\psi}):=p^+(\bm{z}^l;{\bm\psi}) \triangleq \frac{1}{N}\sum_{j=1}^{N}q_j^{\diamond}(\bm{z}^l\mid u_j,\bm{v}_j;{\bm{\psi}_j}),
\end{equation}
which expresses \emph{a priori} knowledge of the common anatomy by the mixture of experts.
This setup is closely related to the \emph{variational mixture of posteriors prior} (VampPrior) introduced in \cite{conference/neurips/hoffman2016,conference/aistats/tomczak2018}, where the VampPrior with pseudo-inputs is chosen to maximise the ELBO.
For example, \cref{fig:geo_arith_mean} illustrates the geometric and arithmetic mean of Gaussian distributions when $N=3$.
One can observe that the three individual modes of single-view posteriors collapse to one after computing the geometric mean.
Indeed, if the experts are Gaussian, one can prove that the KL divergence from the arithmetic to the geometric mean distribution is minimised when the experts are identical.

On the other hand, applying Jensen's inequality to the KL divergence yields
\begin{equation}\label{eq:jensen_inequality}
  \begin{aligned}
    D_S&\triangleq\kldiv{q(\bm{z}^l\mid\bm{u},\bm{v};{\bm{\psi}})}{p^+(\bm{z}^l;{\bm\psi})} \\ 
    &\leq\frac{1}{N}\sum_{j=1}^{N}\kldiv{q(\bm{z}^l\mid\bm{u},\bm{v};{\bm{\psi}})}{q_j^{\diamond}(\bm{z}^l\mid u_j,\bm{v}_j;{\bm{\psi}_j})}\triangleq \widetilde{D}_S,
  \end{aligned}
\end{equation}
where $\widetilde{D}_S$ is an upper bound of the original (possibly) intractable KL divergence $D_S$ involving mixture distributions.
Thus, in practice we use this upper bound as a surrogate to minimise the KL divergence \emph{w.r.t.} the latent variable $\bm{z}$.

Note that when minimising $\widetilde{D}_S$ \emph{w.r.t.} the distribution $q(\bm{z}^l\mid\bm{u},\bm{v};{\bm{\psi}})$, we also obtain the geometric mean $q^*(\bm{z}^l\mid\bm{u},\bm{v};{\bm{\psi}})$ in \cref{eq:geo_mean}, \emph{i.e.},
\begin{equation*}
  q^*(\bm{z}\mid\bm{u},\bm{v};\bm{\psi}) =\argmin_{q}\widetilde{D}_S[q],
\end{equation*}
where $\widetilde{D}_S[q]$ is regarded a functional of the distribution $q$.
This also justifies the usage of the geometric mean as the variational posterior of the common anatomy.
In fact, since the KL divergence is a statistical distance between probability distributions, the $\widetilde{D}_S$ in \cref{eq:jensen_inequality} serves as an \emph{intrinsic distance} between the structural representations of the common anatomy and each observed image.
Therefore, since the geometric mean is the distribution that minimises this intrinsic distance, the optimisation of $\widetilde{D}_S^*\triangleq \widetilde{D}_S[q^*]$ \emph{w.r.t.} $q_j^{\diamond}$'s will then seek to force the single-view structural representations $q_j^{\diamond}(\bm{z}^l\mid u_j,\bm{v}_j;{\bm\psi_j})$'s (which are conditioned on the transformation variables) to be identical, thus driving the registration process.
Section A of the Supplementary Material provides more details for the above reasoning.

\subsubsection{Structural Representations for Variable Group Sizes}
In our previous work \cite{conference/ipmi/wang2023}, we have used Gaussian distributions to parameterize the structural representations.
That is, we assume $q_j^{\diamond}(\bm{z}\mid u_j,\bm{v}_j;{\bm{\psi}_j})=\mathcal{N}(\bm{z};\bm{\mu}_{\bm{z},j}^{\diamond},\bm{\Sigma}_{\bm{z},j}^{\diamond})$ where $\bm{\Sigma}_{\bm{z},j}^{\diamond}$ are diagonal, and \emph{for notational conciseness we omit the superscript $l$ indicating the level of the latent variables}.
We can calculate that the geometric mean is also a Gaussian distribution, \emph{i.e.}, $q^*(\bm{z}\mid\bm{u},\bm{v};{\bm{\psi}})=\mathcal{N}(\bm{z};\bm{\mu}_{\bm{z}}^*,\bm{\Sigma}_{\bm{z}}^*)$, where 
\begin{equation}
  \bm{\Sigma}_{\bm{z}}^*=N\left[\sum_{j=1}^{N}\bm{\Sigma}_{\bm{z},j}^{\diamond\,-1}\right]^{-1},
  \quad 
  \bm{\mu}_{\bm{z}}^*=\frac{\bm{\Sigma}_{\bm{z}}^*}{N}\sum_{j=1}^{N}\bm{\Sigma}_{\bm{z},j}^{\diamond\,-1}\bm{\mu}_{\bm{z},j}^{\diamond}.
\end{equation}
Therefore, the intrinsic distance has a closed-form expression, \emph{i.e.},
\begin{equation*}
  \begin{aligned}
    &\kldiv{q^*(\bm{z}\mid\bm{u},\bm{v};{\bm{\psi}})}{q_j^{\diamond}(\bm{z}\mid u_j,\bm{v}_j;{\bm{\psi}_j})} \\
  =\ & \frac{1}{2}\Bigg[\log\frac{\abs{\bm{\Sigma}_{\bm{z},j}^{\diamond}}}{\abs{\bm{\Sigma}_{\bm{z}}^*}}+\operatorname{tr}\left(\bm{\Sigma}_{\bm{z},j}^{\diamond\,-1}\bm{\Sigma}_{\bm{z}}^*\right) \\
  &\quad\; +(\bm{\mu}_{\bm{z},j}^{\diamond}-\bm{\mu}_{\bm{z}}^*)^{\intercal}\bm{\Sigma}_{\bm{z},j}^{\diamond\,-1}(\bm{\mu}_{\bm{z},j}^{\diamond}-\bm{\mu}_{\bm{z}}^*)\Bigg] + \text{const.}
  \end{aligned}
\end{equation*}
where the quadratic term is essentially a Mahalanobis distance between the structural representations.

Here, we extend the framework using categorical distribution to improve the interpretability of the structural representations.
Particularly, we propose to model them by independent categorical latent variables, \emph{i.e.},
$q_j^{\diamond}(\bm{z}\mid u_j,\bm{v}_j;{\bm{\psi}_j})=\operatorname{Cat}(\bm{z};\bm{\pi}_{j,1}^{\diamond},\dots,\bm{\pi}_{j,K}^{\diamond})$, where $\bm{\pi}_{j,k}^{\diamond}=(\pi_{\bm{\omega},j,k}^{\diamond})_{\bm{\omega}\in\Omega} \in[0,1]^{\abs{\Omega}}$ and $\sum_{k=1}^{K}\bm{\pi}_{j,k}^{\diamond}=\bm{1}\in\mathds{R}^{\abs{\Omega}}$.
Thus, the geometric mean becomes another categorical distribution 
\begin{equation}
  \begin{aligned}
    q^*(\bm{z}\mid\bm{u},\bm{v};{\bm{\psi}}) &= \operatorname{Cat}(\bm{z};\bm{\pi}_1^*,\dots,\bm{\pi}_K^*) = \prod_{\bm{\omega}\in\Omega}\prod_{k=1}^K (\pi_{\bm{\omega},j,k}^*)^{z_{\bm{\omega},k}},
  \end{aligned}
\end{equation}
with 
\begin{equation*}
  \bm{\pi}_k^* = \frac{\left[\prod_{j=1}^{N}\bm{\pi}_{j,k}^{\diamond}\right]^{\nicefrac{1}{N}}}{\sum_{k=1}^{K}\left[\prod_{j=1}^{N}\bm{\pi}_{j,k}^{\diamond}\right]^{\nicefrac{1}{N}}}\in [0,1]^{\abs{\Omega}},
\end{equation*}
and the intrinsic distance takes the form 
\begin{equation}
  \kldiv{q^*(\bm{z}\mid\bm{u},\bm{v};{\bm{\psi}})}{q_j^{\diamond}(\bm{z}\mid u_j,\bm{v}_j;{\bm{\psi}_j})} = \sum_{k=1}^{K}\bm{\pi}_k^*\log\frac{\bm{\pi}_k^*}{\bm{\pi}_{j,k}^{\diamond}}.
\end{equation}
\emph{Note that this intrinsic distance can be applied to image groups with variable group sizes, 
as the geometric mean can be computed from an arbitrary number of structural representations.
}

\subsubsection{Categorical Reparameterization Using Gumbel-Rao}
Unfortunately, the Gumbel-Max trick \cite{journal/tpami/huijben2022} for sampling discrete random variables is not differentiable \emph{w.r.t.} its parameterization, making it unsuitable for stochastic gradient estimators.
Concurrent works \cite{conference/iclr/jang2017,conference/iclr/maddison2017} have then introduced the Gumbel-Softmax (GS) distribution as a continuous relaxation of the discrete categorical variable, which admits a biased reparameterization gradient estimator.
Here, we use a variant of the GS estimator, namely the Gumbel-Rao (GR) estimator \cite{conference/iclr/paulus2021}, to further reduce the variance in gradient estimation via Rao-Blackwellisation.

Specifically, for objective functions like ELBO, the gradient \emph{w.r.t.} the distribution parameters of an expectation over a function $f(\bm{z})$ must be computed, namely $\nabla_{\bm{\psi}}\mathbb{E}_{q(\bm{z};{\bm\psi})}[f(\bm{z})]$, where $f(\bm{z})$ can be the likelihood function of the observed variables.
To overcome the challenge of gradient computation with discrete stochasticity, the Straight-Through Gumbel-Softmax (ST-GS) estimator \cite{conference/iclr/jang2017,conference/iclr/paulus2021} uses a continuous relaxation, giving rise to a biased Monte-Carlo gradient estimator of the form 
\begin{equation}
  \nabla_{\text{ST-GS}}^{\text{MC}}\triangleq\frac{\partial f(\bm{z})}{\partial \bm{z}}\frac{\operatorname{d}\operatorname{softmax}_{\tau}(\bm{g}+\log\bm{\pi}(\bm{\psi}))}{\operatorname{d}\bm{\psi}},
\end{equation}
where the forward pass in $f(\cdot)$ is computed using the non-relaxed discrete samples.
Based on the ST-GS, the GR estimator takes the form as
\begin{equation}
  \nabla_{\text{GR}}^{\text{MC}}\triangleq \mathbb{E}\left[\nabla_{\text{ST-GS}}^{\text{MC}}\mid\bm{z}\right]
  \approx\frac{\partial f(\bm{z})}{\partial \bm{z}}\left[\frac{1}{S}\sum_{s=1}^S \frac{\operatorname{d}\operatorname{softmax}_{\tau}(\bm{G}^s(\bm{\psi}))}{\operatorname{d}\bm{\psi}}\right], 
\end{equation}
where $\bm{G}^s\stackrel{\text{i.i.d.}}{\sim}\bm{g}+\log\bm{\pi}\mid\bm{z}$ for $s=1,\dots,S$, and $\bm{g}$ is a random vector with its entries \emph{i.i.d.} sampled from the Gumbel distribution.
More details on the Gumbel-Rao estimator can be found in Section B of the Supplementary Material.

\subsection{Spatial Regularization for Diffeomorphisms}
To impose spatial smoothness of the velocity fields, we define its prior distribution by $p(\bm{v}_j^l)=\mathcal{N}(\bm{v};\bm{0},\bm{\Lambda_v}^{-1})$, where the precision matrix $\bm{\Lambda_v}$ is given by the scaled Laplacian matrix of a neighbourhood graph on the voxel grid \cite{journal/media/dalca2019}, \emph{i.e.}, $\bm{\Lambda_v}=\lambda\bm{L_v}=\lambda(\bm{D_v}-\bm{A_v})$, with $\bm{D_v}$ and $\bm{A_v}$ the degree and adjacency matrices, respectively.
On the other hand, the variational posterior of the velocity fields is given by a mean-field Gaussian distribution, \emph{i.e.}, $q(\bm{v}_j^l\mid\bm{u},\bm{v}^{<l};{\bm{\psi}})=\mathcal{N}(\bm{v};\bm{\mu}_{\bm{v},j}^l,\bm{\Sigma}_{\bm{v},j}^l)$ where $\bm{\Sigma}_{\bm{v},j}^l$ is diagonal.
Therefore, the KL divergence for the velocity fields has an explicit formula
\begin{equation}
\begin{aligned}
  &\kldiv{q(\bm{v}_j^l\mid\bm{u},\bm{v}^{<l};{\bm\psi})}{p(\bm{v}_j^l)} \\ 
  =\ & \frac{1}{2}\left[\operatorname{tr}(\lambda\bm{D_v}\bm{\Sigma}_{\bm{v},j}^l-\log\bm{\Sigma}_{\bm{v},j}^l)+\frac{\lambda}{2}\sum_{r}\sum_{q\in\mathcal{N}(r)}(\bm{\mu}_{\bm{v},j}^l[r]-\bm{\mu}_{\bm{v},j}^l[q])^2\right] \\
  &+ \operatorname{const.}
\end{aligned}
\end{equation}
where $\mathcal{N}(r)$ are the neighbours of voxel $r$.
Thus, the quadratic term over $\bm{\mu}$ enforces the velocity fields to be spatially smooth, encouraging diffeomorphic transformations.

\section{An Interpretable Registration Architecture via Bayesian Disentanglement Learning}\label{sec:bayes_network}
\begin{figure}
  \centering
  \includegraphics[width=\linewidth]{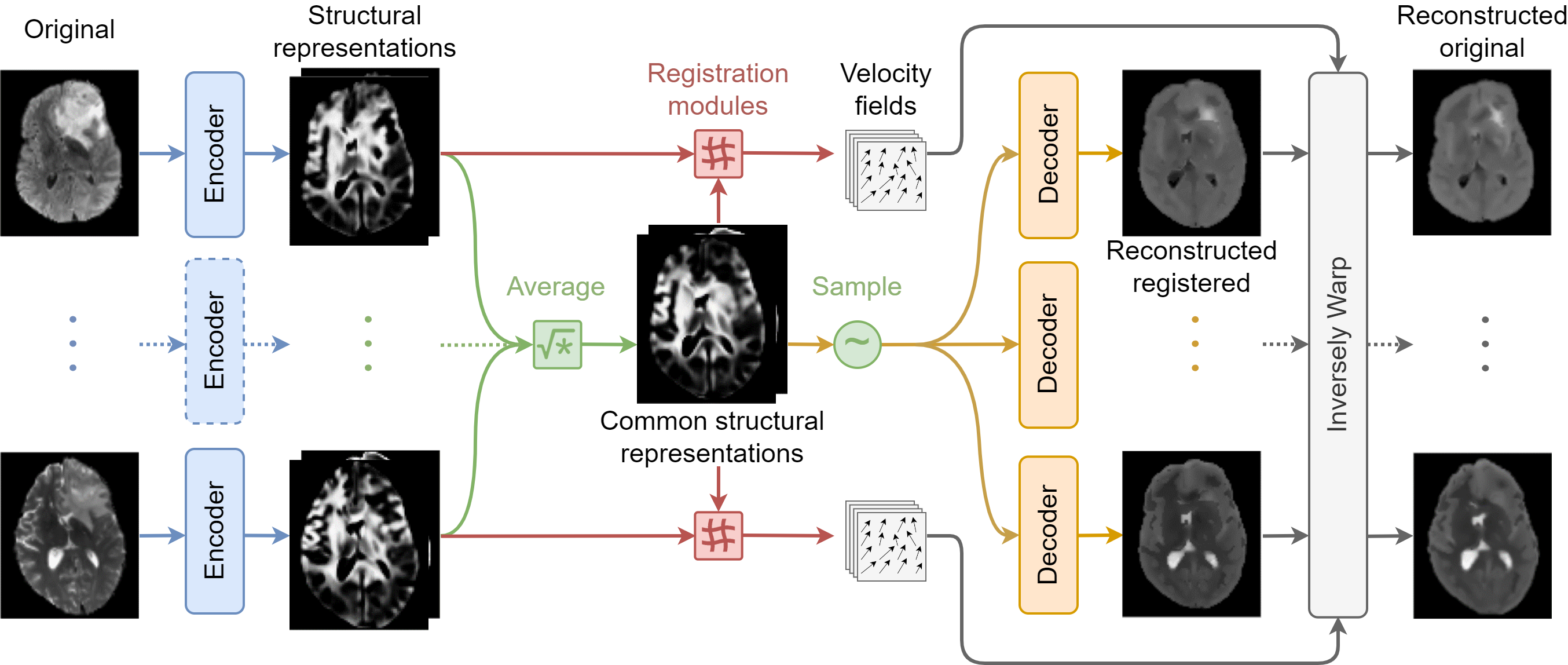}
  \caption{An overview of the proposed interpretable groupwise registration architecture via Bayesian disentanglement learning.
  }
  \label{fig:network_outline}
\end{figure}

\begin{figure*}
  \centering
  \includegraphics[width=\textwidth]{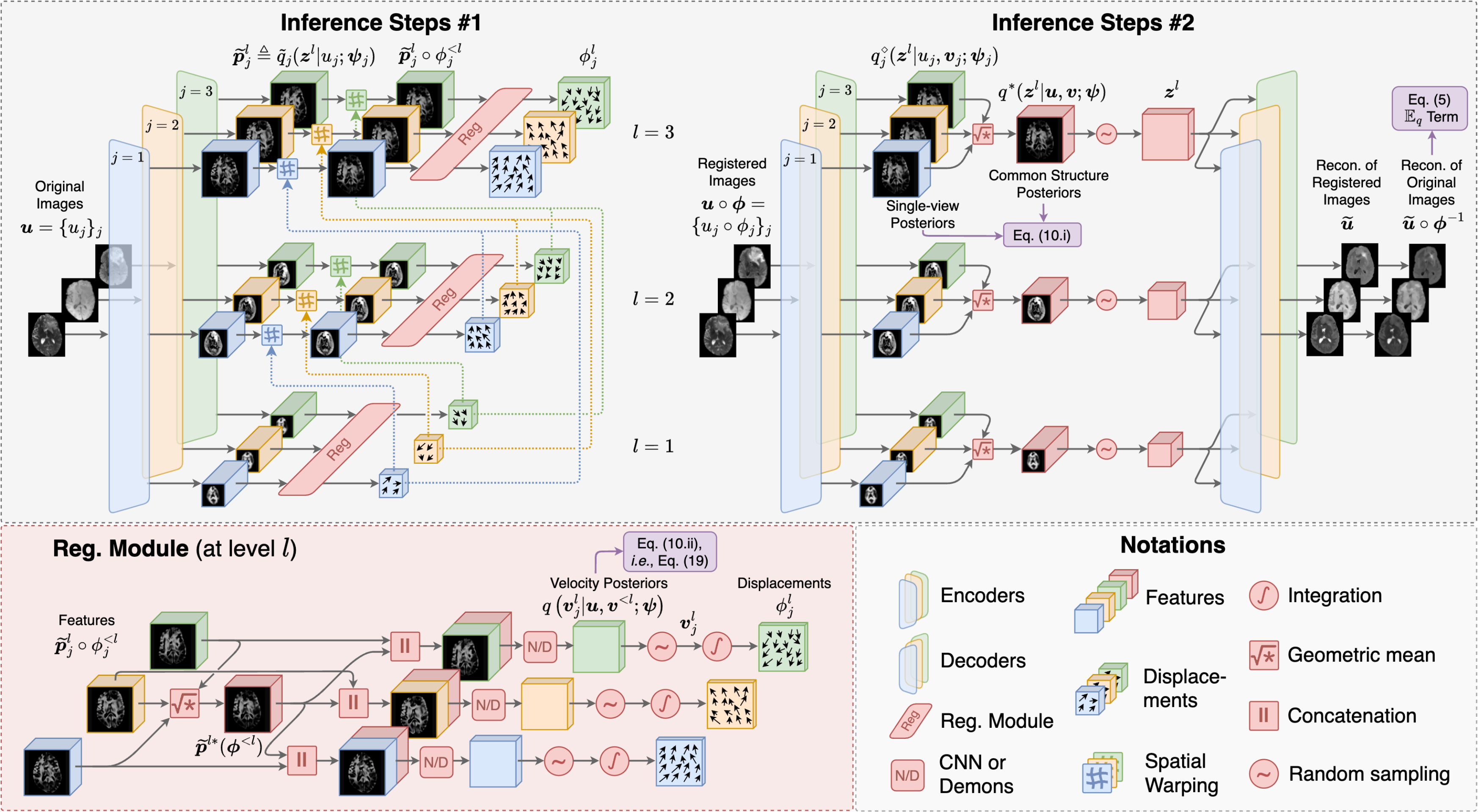}
  \caption{The network architecture for the proposed Bayesian groupwise registration, composed of the encoders that extract categorical structural representation maps, the registration modules that calculate multi-scale velocity fields, and the decoders that reconstruct the original images based on the common structural representations. Without loss of generality, the illustration is with $L=3$ levels and $N=3$ images to co-register.
  Note that inference steps \#2 are performed only in the training stage, while in the test stage the encoder is only fed with the original image group to predict groupwise registration. The purple boxes indicate the calculation of related terms in the ELBO.
   }
  \label{fig:network_structure}
\end{figure*}

Most learning-based registration methods in the recent literature are based on an end-to-end training pipeline, where a black-box neural network directly predicts the spatial transformation by optimising some image-level similarity measures \cite{journal/media/de2019,journal/tmi/balakrishnan2019,journal/media/dalca2019,journal/media/kang2022,journal/media/chen2022}.
However, not only do black-box end-to-end networks lack interpretability, learning directly the complex mapping from multi-modal images to their spatial correspondence disregards the underlying structural relationship and may be prone to generalisation issues \cite{journal/media/gao2013}.

Bayesian deep learning \cite{journal/acm/wang2020}, however, could be a remedy.
It unifies probabilistic graphical models (PGM) with deep learning, and integrates the inference and perception tasks, enabling them to benefit from each other.
Particularly in our context, the \emph{perception component} includes the encoders that extract single-view posteriors from individual modalities, while the \emph{task-specific component} comprises the registration modules and reconstruction decoders that learn jointly to infer the spatial correspondence among the input images.

\cref{fig:network_outline} presents a schematic overview of the proposed interpretable registration architecture, specially designed based on the optimisation procedure for Bayesian inference outlined in the previous sections and \cref{fig:graph}.
The encoders first disentangle the input multi-modal images into structural representation maps with categorical distribution. 
Then an average representation of the common anatomy is computed within the latent space, so that spatial correspondence between the observed images and the common anatomy is predicted from these structural representations using the Diffeomorphic Demons algorithm.
The disentanglement of geometric and anatomical variations is facilitated by inversely simulating the image generative process with spatially equivariant decoders, \emph{i.e.}, reconstructing the original images from the common anatomy representation.
\cref{fig:network_structure} depicts the entire network architecture, composed of the encoders, the registration modules and the decoders.
The following subsections will detail the construction of these modules.

\subsection{Inference of Structural Representations}
In the inference steps \#1 (ref. \cref{fig:inference_steps1}), given the original misaligned image group $\bm{u}$, the encoders first estimate the multi-level categorical distributions $\widetilde{q}_j(\bm{z}^l\mid u_j;{\bm{\psi}_j})$ for each original image $u_j$, where $l$ indicates the spatial resolution of the representation.
Note that $\widetilde{q}_j(\bm{z}^l\mid u_j;{\bm{\psi}_j})\triangleq\operatorname{Cat}(\bm{z};\widetilde{\bm{p}}_{j,1}^l,\dots,\widetilde{\bm{p}}_{j,K}^l)$ differs from $q_j^{\diamond}(\bm{z}^l\mid u_j,\bm{v}_j;{\bm{\psi}_j})=\operatorname{Cat}(\bm{z};\bm{\pi}_{j,1}^{\diamond\, l},\dots,\bm{\pi}_{j,K}^{\diamond\,l})$ regarding the network input.
The former takes $\bm{u}=\{u_j\}_{j=1}^N$ as input, while the latter takes $\bm{u}\circ\bm{\phi}=\{u_j\circ\phi_j\}_{j=1}^N$.
The architecture of the encoders is adapted from the Attention U-Net \cite{conference/midl/oktay2018},  which enhances the vanilla U-Net \cite{conference/miccai/ronneberger2015} with additive attention gates.
Moreover, domain-specific batch normalisation \cite{conference/cvpr/chang2019} is utilised in the contracting path of the encoders to cancel out modality-specific appearance variations, while all convolutional layers are shared across modalities, embedding multi-modal images into modality-invariant representations.

\subsection{Inference of Velocity Fields}
Based on the structural representation maps extracted from the original images by the encoders, the registration modules seek to estimate the transformations that spatially register the image group.
This is realised in a coarse-to-fine approach, starting from prediction of the lowest-resolution velocity fields (ref. the red arrows in \cref{fig:inference_steps1}).

Particularly, denoting $\widetilde{\bm{p}}_j^l\triangleq\widetilde{q}_j(\bm{z}^l\mid u_j;{\bm{\psi}_j})$, the registration module at level $l$ takes as input the partially registered single-view structural probability maps $\widetilde{\bm{p}}_j^l\circ\phi_j^{<l}$ and their geometric mean for the common anatomy $\widetilde{\bm{p}}^{l*}(\bm{\phi}^{<l})$,
based on which the velocity field distribution $q(\bm{v}_j^l\mid\bm{u},\bm{v}^{<l};{\bm{\psi}})$ is predicted.
Note that the spatial transformations $\bm{\phi}^{<l}\triangleq\{\phi_j^{<l}\}_{j=1}^N$ are integrated by $\phi_j^{<l}=\exp(\bm{v}_j^1+\dots+\bm{v}_j^{l-1})$, with $\bm{v}_j^l$ sampled from $q(\bm{v}_j^l\mid\bm{u},\bm{v}^{<l};{\bm{\psi}})=\mathcal{N}(\bm{v};\bm{\mu}_{\bm{v},j}^l,\bm{\Sigma}_{\bm{v},j}^l)$.
Finally, the total transformations that register the observed image group are given by $\phi_j=\exp(\bm{v}_j^+)$ with $\bm{v}_j^+\triangleq\sum_{l=1}^L\bm{v}_j^l$, where $\bm{v}_j^l$ is a random sample during training while $\bm{v}_j^l=\bm{\mu}_{\bm{v},j}^l$ for testing.
Note that the velocity fields are upsampled to the finest resolution before summation.

As the structural representation maps $\widetilde{\bm{p}}_j^l\circ\phi_j^{<l}$ encode the structural features of the partially registered input images, we can in fact compute the mean $\bm{\mu}_{\bm{v},j}^l$ of the variational posterior $q(\bm{v}_j^l\mid\bm{u},\bm{v}^{<l};{\bm{\psi}})$ directly from them based on the Diffeomorphic Demons algorithm 
\cite{journal/media/thirion1998,journal/ni/vercauteren2009}.
Particularly, for each grid location $\bm{\omega}$ at level $l$, we define $\bm{\mu}_{\bm{v},j}^l$ by
\begin{equation*}
  \bm{\mu}_{\bm{v},j}^l(\bm{\omega}) \triangleq \alpha\cdot {\widetilde{\bm{\mu}}_{\bm{v},j}^l(\bm{\omega})}\Big/\max_{\bm{\omega}\in\Omega^l}{\norm{\widetilde{\bm{\mu}}_{\bm{v},j}^l(\bm{\omega})}_2},
\end{equation*}
where
\begin{equation*}
\widetilde{\bm{\mu}}_{\bm{v},j}^l(\bm{\omega}) \triangleq - \Big[\bm{J}_{\bm{\varphi}_j^l}^{\intercal}(\bm{\omega})\bm{J}_{\bm{\varphi}_j^l}(\bm{\omega})+\sigma_{\bm{\varphi}_j^l}^2(\bm{\omega})\bm{I}_d\Big]^{-1}\!\bm{J}_{\bm{\varphi}_j^l}^{\intercal}(\bm{\omega})\bm{\varphi}_{j,\bm{\omega}}^l(\bm{0}) 
\end{equation*}
is the \emph{symmetric Demons force}, and $\alpha \in (0, \alpha_0)$ is a learnable parameter controlling the magnitude of the velocity fields.
Specifically, 
$$\bm{\varphi}_{j,\bm{\omega}}^l({\bm{u}}) \triangleq \widetilde{\bm{p}}^{l*}(\bm{\omega})-\widetilde{\bm{p}}_j^l\circ\phi_j^{<l}\circ(\operatorname{id}+\bm{u})(\bm{\omega})$$ 
is the feature difference given a displacement $\bm{u}\in\mathds{R}^d$,
\begin{equation*}
    \bm{J}_{\bm{\varphi}_j^l}(\bm{\omega})\triangleq -\frac{1}{2}\left[\nabla_{\bm{\omega}}^{\intercal}\,\widetilde{\bm{p}}^{l*}(\bm{\omega})+\nabla_{\bm{\omega}}^{\intercal}\,\widetilde{\bm{p}}_j^l\circ\phi_j^{<l}(\bm{\omega})\right]
\end{equation*}
is the symmetrized Jacobian matrix at the origin, and 
\begin{equation*}
  \sigma_{\bm{\varphi}_j^l}^2 = \frac{1}{\abs{\Omega^l}-1}\sum_{\bm{\omega}\in\Omega^l}\norm{\bm{\varphi}_{j,\bm{\omega}}^l({\bm{\bm{0}}})-\widebar{\bm{\varphi}}_j^l(\bm{0})}_2^2
\end{equation*}
is the sample variance of the feature difference norm.
Detailed derivation can be found in Section C of the Supplementary Material, where the Demons algorithm is established from a mathematically interpretable optimisation framework.
The variance $\bm{\Sigma}_{\bm{v},j}^l$ of the variational posterior is predicted by a convolutional block from the concatenation of probability maps $[\widetilde{\bm{p}}_j^l\circ\phi_j^{<l}; \widetilde{\bm{p}}^{l*}]$, where the network parameters are shared among different modalities.

In addition, to avoid the degeneracy that deforms the observed images to an arbitrary coordinate space, we constrain the sum of the final velocity fields to be zero by subtracting their average \cite{journal/ni/avants2004}, \emph{i.e.}, $\bm{v}_j^+ \leftarrow \bm{v}_j^+ - \frac{1}{N}\sum_{j=1}^N \bm{v}_j^+.$

\subsection{Bayesian Disentangled Representation Learning}
Given the perceptual modules capturing intrinsic structural representations from multi-modal images and their spatial correspondence, we are going to discuss how such mapping functions are learnt in the first place.
This is achieved by optimising the network parameters $\bm{\psi}$ and $\bm{\theta}$ in a closed-loop self-reconstruction process that maximises ELBO.
The single-view structural distributions for estimating the common anatomy are defined by
\begin{equation*}
  q_j^{\diamond}(\bm{z}^l\mid u_j,\bm{v}_j;{\bm{\psi}_j})\triangleq\widetilde{q}(\bm{z}^l\mid u_j\circ\phi_j;{\bm{\psi}_j}),
\end{equation*}
which is obtained by feeding the transformed images $u_j\circ\phi_j$ into the encoders in the inference steps \#2 (ref. \cref{fig:inference_steps2}), effectively reverting the image generative process in \cref{eq:img_generate}.

Based on the common structural representations sampled from the geometric mean $q^*(\bm{z}^l\mid \bm{u},\bm{v};{\bm{\psi}})$, the decoders reconstruct the original images in two steps (ref. \cref{fig:generation_steps}).
The images corresponding to the common anatomy $\widetilde{u}_j=f(\bm{z};\bm{\theta}_j)$ are decoded first.
Then, estimation of the original images $\widehat{\bm{u}}=\{\widehat{u}_j\}_{j=1}^N$ is obtained by inverse warping, \emph{i.e.}, $\widehat{u}_j=\widetilde{u}_j\circ\phi_j^{-1}$ with $\phi_j^{-1}=\exp(-\bm{v}_j^+)$, simulating \cref{eq:img_generate}.

The network architecture of the decoders is given by multi-level convolutional blocks and linear upsampling.
At level $l$, the output feature maps from upsampling are multiplied by the sampled categorical variables $\bm{z}^l$.
Besides, all convolutional layers are assigned with a kernel size of 1 isotropically, to suppress spatial correlation and to form \emph{spatial equivariance}.
Thus, the decoded image $\widetilde{u}_j$ is in the same spatial orientation as the common anatomy $\bm{z}$, encouraging the registration module to identify the correct velocity fields $\bm{v}_j^+$ for reconstructing $\widehat{u}_j$.
This connection between spatial equivariance and registration identifiability is formally established in Section E of the Supplementary Material.
Furthermore, the batch normalisation layers in the decoders are set as the inverse of the counterpart BN functions along the contracting path of the encoders, \emph{i.e.}, 
\begin{equation*}
  \bm{c}_j^{\text{out}} = \frac{\bm{c}_j^{\text{in}}-\bm{\beta}_{m}}{\bm{\gamma}_{m}+{\varepsilon}}\sqrt{\bm{\sigma}_{m}^2+\varepsilon}+\bm{\mu}_{m},
\end{equation*}
where for each image index $j$ and its corresponding modality $m$, $\bm{c}_j^{\text{in}},\bm{c}_j^{\text{out}}\in\mathds{R}^{B_m\times C\times H\times W\times D}$ are the input and output feature maps, $\bm{\mu}_{m},\bm{\sigma}_{m}\in\mathds{R}^C$ are the batch statistics, and $\bm{\beta}_{m},\bm{\gamma}_{m}\in\mathds{R}^C$ are the affine parameters.
The arithmetic is computed element-wise. 
In fact, such preference for invertible and spatially decomposable decoders may promote learning of the true disentangled representations \cite{conference/aistats/khemakhem2020,conference/neurips/reizinger2022}.
An illustration of the encoder and decoder architecture can be found in Section D of the Supplementary Material.

Finally, the expectation $\mathbb{E}_{q(\bm{z},\bm{v}\mid\bm{u};\bm{\psi})}[\log p(\bm{u}\mid\bm{z},\bm{v};{\bm\theta})]$ is approximated by Monte-Carlo ancestral sampling, where the likelihood is modelled by a Laplace distribution, \emph{i.e.},
\begin{equation}
  \begin{aligned}
    p(\bm{u}\mid\bm{z},\bm{v};\bm{\theta}) 
    &\propto \prod_{j=1}^N\prod_{\bm{\omega}\in\Omega_j} \exp\left\{ -\frac{\abs{u_j(\bm{\omega})-f(\bm{z};\bm{\theta}_j)\circ\phi_j^{-1}(\bm{\omega})}}{b} \right\},
  \end{aligned}
\end{equation}
where the scale parameter $b$ is set to 1 in all experiments.

The proposed framework also encourages learning of disentangled representations formally defined by \citet{journal/arxiv/higgins2018}, which formulates disentangled representations as a reflection on the decomposition of certain symmetry groups acting on the world states (or observations).
This explains why we choose to decompose the latent variables of our model into the common anatomy and the corresponding spatial diffeomorphisms: because they reflect two separate symmetry structures of the observed images, \emph{i.e.}, the ontological transformation changing the \emph{underlying anatomy}, and the diffeomorphic transformation characterising \emph{geometric variation}.
A more thorough discussion of disentangled representations and their identifiability can be found in Section E of the Supplementary Material.

\subsection{Learning, Inference and Scalability}
Learning of our model is performed in an end-to-end fashion. 
All model parameters, including the auto-encoders and the registration modules, are trained simultaneously by stochastic gradient ascent to maximise the ELBO.
In the inference stage, the transformations that register the image group are given by $\widehat{\bm{\phi}}=\{\widehat{\phi}_j\}_{j=1}^N$, where 
$\widehat{\phi}_j = \exp(\widehat{\bm{\mu}}_j^+)$
and $\widehat{\bm{\mu}}_j^+$ is aggregated from the means of the velocity distributions at all levels predicted by the registration modules.

Remarkably, the proposed disentanglement of anatomy and geometry endows our model with \emph{powerful scalability}. Since geometric and arithmetic averaging can be performed on any number (greater than 1) of representations, our model is capable of processing image groups of arbitrary sizes during training and testing. 
Let the number of images in an input group be $N'=\sum_{m=1}^M{N_m}$, where $N_m$ is the number of images of modality $m$. Each image of modality $m$ is processed by the encoder for the corresponding modality to extract structural representations, and then the geometric and arithmetic means are calculated on the representations of all images to estimate the KL divergence and predict deformations for each image. For decoding, the reconstruction from the decoder of modality $m$ is inversely warped for $N_m$ times by the deformations for the $N_m$ images of modality $m$ to reconstruct original images. Thus, our model allows $N_m$'s to be different or even 0, and only requires $N'\geq 2$ (for calculating averages) and at least two modalities are available in each group, \emph{i.e.}, $\sum_m\text{1}(N_m>0)\geq 2$ (for effectively learning modality-invariant representations).

Therefore, our model can handle groups with missing modalities or variable sizes, which are common occurrences in real-world scenarios. Particularly, the scalability of our model is of great value in two aspects: 
\begin{itemize}
    \item During training, we can drop several modalities/images in each training group, thereby reducing the computational cost and/or allowing for learning from less data effectively. Typically, there are two experimental settings: 1) \emph{complete learning}, the standard setting, where each group consists of all modalities, with one image per modality, 2) \emph{partial learning}, where each group only contains two modalities, with one image per modality. Experiments in both settings are discussed in \cref{sec:experiment}. 
    \item During evaluation, our model can register groups with arbitrary sizes, enabling large-scale and variable-size multi-modal groupwise registration, as explored in \cref{sec:scalability}. In contrast, conventional iterative methods optimise similarity metrics with computational complexities that escalate rapidly with group size. Deep learning baselines, on the other hand, often have fixed channel numbers to take as input the concatenation of images \cite{conference/icip/he2020}, necessitating groups to be of uniform size. These limitations significantly restrict their applicability.
\end{itemize}

\section{Experiments and Results}
\label{sec:experiment}
We evaluated our framework on a total of four publicly available datasets, including BraTS-2021 \cite{report/arxiv/baid2021, journal/tmi/menze2014,journal/sd/bakas2017}, MS-CMRSeg \cite{journal/media/zhuang2022}, Learn2Reg Abdominal MR-CT \cite{journal/tmi/hering2022,journal/tbe/xu2016,journal/media/kavur2021}, and the OASIS dataset \cite{journal/jcn/marcus2010}.
The experiments were implemented in PyTorch \cite{report/arxiv/paszke2017} and conducted on $\text{NVIDIA}^\circledR$ $\text{RTX}^\textnormal{TM}$ 3090 GPUs.
The experimental materials, baselines, evaluation metrics, and results are detailed as follows.

\subsection{Materials}
To show that our method applies to various groupwise registration tasks, especially in multi-modal scenarios, we validated it on four different datasets: 
\begin{enumerate}
  \item \textbf{MS-CMRSeg.} 
  This dataset provides multi-sequence cardiac MR images from 45 patients, including LGE (Late Gadolinium Enhanced), bSSFP (balanced-Steady State Free Precession), and T2-weighted scans.
  The images were preprocessed by affine co-registration, slice selection, resampling to $1\times 1\ \text{mm}$ and ROI extraction, giving 39, 15 and 44 slices for training, validation and test, respectively.
  Additional misalignment simulated by random FFDs using isotropic control point spacings of $5/10/20/40\ \text{mm}$ were made to further demonstrate registration.
  The ROI regions were obtained by dilating the foreground mask using a circular filter of 15-pixel radius.
  Evaluation was conducted on the warped manual segmentations of the myocardium, left and right ventricle.
  \emph{The major challenge of this dataset comes from the complex intensity patterns of the cardiac region, which can misguide the registration.}
  \item \textbf{BraTS-2021.} 
  This dataset contains multi\hyp{}parametric MRI scans of glioma, including native T1, T1Gd (post-contrast T1-weighted), T2-weighted, and T2-FLAIR (Fluid Attenuated Inversion Recovery).
  The multi-parametric MRIs of the same patients were co-registered to the same anatomical template, interpolated to the same resolution ($1\times 1\times 1\ \text{mm}$) and skull-stripped.
  We randomly selected 300, 50 and 150 patient cases for training, validation and test, respectively.
  The images were downsampled into $2\times 2\times 2\ \text{mm}$ with volume size of $80\times 96\times 80$.
  To demonstrate registration, four synthetic free-form deformations (FFDs) for each image with isotropic control point spacings of $5/10/15/20\ \text{mm}$ were generated to simulate misalignment.
  \emph{The major challenge of this task stems from the complex intensity patterns around regions of tumor, yielding ambiguous anatomical correspondences among images.}
  \item \textbf{Learn2Reg Abdominal MR-CT.}
  This dataset collects 3D T1-weighted MR and CT abdominal images.
  The data were resampled to $3\times 3 \times 3\ \text{mm}$ and cropped to size of $112\times 96\times 112$.
  Training was performed on the unpaired 40 MR and 50 CT images \cite{journal/media/kavur2021,journal/tbe/xu2016}, while 8 MR-CT pairs (16 images) were used for test \cite{journal/jdi/clark2013}.
  The masks provided by the dataset were used to confine the information used for registration.
  Evaluation was performed on the warped manual segmentations of the liver, spleen, left and right kidneys.
  \emph{The major challenges of this dataset are the large deformation and missing correspondences between the two modalities, as well as the distribution shift between training and test datasets.}
  \item \textbf{OASIS.}
  The OASIS-1 dataset contains 414 T1-weighted 3D MR scans from young, middle-aged, non-demented and demented older adults \cite{journal/jcn/marcus2010}.
  The images were skull-stripped, bias-corrected, and registered into an affinely-aligned, common template space with FreeSurfer \cite{journal/ni/fischl2012,conference/ipmi/hoopes2021}.
  Evaluation was conducted on the warped manual segmentations of the cortex, subcortical grey matter, white matter, and CSF.
  The volumes were resampled $2\times 2\times 2\ \text{mm}$ and cropped to size of $80\times 96\times 96$.
  Besides, 287/40/87 images were randomly selected for training/validation/test.
  \emph{The major challenge of this task comes from the large inter-subject variability of the brain structures.}
\end{enumerate}

\subsection{Compared Methods}
We compared two types of baselines with our models in the experiments. 
The first type features state-of-the-art iterative methods for multi-modal groupwise registration, including accumulated pairwise estimates (APE) \cite{journal/tpami/wachinger2012}, conditional template entropy (CTE) \cite{journal/media/polfliet2018}, and $\mathcal{X}$-CoReg \cite{journal/tpami/luo2022}.
These methods do not need training, but they can be prone to image artefacts and can hardly register large image groups due to excessive computational burdens.
The second type of baselines extends the first type, aiming to learn the groupwise spatial correspondence via a neural network using the aforementioned similarity measures.
To promote a fair comparison, we used Attention U-Net \cite{conference/midl/oktay2018} for the network backbone, the same architecture as the encoder in our proposed framework.
\emph{However, note that these learning-based baselines can only be trained to register images of a fixed group size, and therefore cannot be applied to data of variable group sizes in the test stage.}

For the proposed model, we implemented two variants of the registration module and adopted two training strategies.
Particularly, the registration modules can be a convolutional network as in our previous work \cite{conference/ipmi/wang2023}, or can be replaced by the Demons force proposed in this work.
In addition, we compare the partial and complete learning strategies for our model.
The partial learning strategy trains the network with only random image pairs, while the trained model can be used to register images of complete group sizes. Thus, the four variants of our model are: 
\begin{itemize}
    \item \emph{Ours-PN}: with partial learning and convolutional network based registration modules.
    \item \emph{Ours-CN}: with complete learning and convolutional network-based registration modules.
    \item \emph{Ours-PD}: with partial learning and Demons-based registration modules.
    \item \emph{Ours-CD}: with complete learning and Demons-based registration modules.
\end{itemize}
Here, ``C'' and ``P'' signify Complete and Partial learning strategies, respectively; ``N'' and ``D'' indicate Network- or Demons-based registration modules, respectively. 

\subsection{Implementation Details}
During preprocessing, the intensity range of each image was linearly normalised to $[0,1]$. 
For cardiac and abdominal images, they were further multiplied by the ROI mask to confine the information used by the network to predict registration.
The encoder has an Attention U-Net architecture with $L=5$ levels, and the $l$-th level consists of two Conv-BN-LeakyReLU blocks with kernel size 3 and $16\times 2^{L-l}$ channels. 
The output of each convolutional block in the upsampling path is passed to a Conv layer to produce logits of the categorical structural distribution, with $8\times 2^{L-l}$ channels.
Samples of the geometric mean structural distribution are passed to the decoder to reconstruct the original images.
The decoder is composed of Conv-BN-LeakyReLU-Upsample blocks with kernel size 1 and $8\times 2^{L-l}$ channels.

To facilitate learning, different weights were used for reconstruction loss, intrinsic structural distance, and registration regularisation, corresponding to a likelihood-tempered ELBO \cite{conference/neurips/osawa2019}.
The balancing weights of the loss terms for different training datasets are presented in the Supplementary Material Section G.
The hyperparameter for spatial regularisation was set to $\lambda=10$ as in \cite{journal/media/dalca2019}, and the maximum magnitude of the velocity calculated by the Demons algorithm was set to $\alpha_0^l=10\times 2^{l-L}$ at each level.
The number of samples for the GR estimator was set to $S=3$ in all experiments.
Network optimisation was performed by stochastic gradient descent using the Adam optimiser \cite{conference/iclr/kingma2015}.
Besides, we used a learning rate of $1\times 10^{-3}$ and a batch size of 20 or 2 for the MS-CMRSeg or the other datasets, respectively.

\subsection{Evaluation Metrics}
For each test image group in an experiment, groupwise semantic evaluation metrics can be constructed by averaging the pairwise version over all possible pairs of the propagated segmentation masks. In other words, a metric on the $n$-th group $\{u_i\}_{i=1}^{N_n}$ is calculated as
$$
\operatorname{Eva}_n\triangleq\frac{1}{\binom{N_n}{2}}\sum_{\substack{1\leq i,j\leq N_n \\ i\neq j}} \operatorname{Eva}(y_i\circ \widehat{\phi}_i,y_j\circ\widehat{\phi}_j),
$$
where $\operatorname{Eva}$ is the Dice similarity coefficient (DSC) or average symmetric surface distance (ASSD), $y_i,y_j$ are distorted masks associated with their respective images $u_i,u_j$, and $\widehat{{\phi}}_i,\widehat{{\phi}}_j$ are the corresponding predicted transformations. The masks are used for model evaluation only.

For experiments on the BraTS-2021 dataset, since the ground-truth spatial correspondence is available,  the groupwise warping index (gWI) can be implemented as an additional metric, which measures the root mean squared error for unbiased deformation recovery based on the ground-truth and predicted deformation fields \cite{journal/tpami/luo2022}, \emph{i.e.}, 
\begin{equation}
\operatorname{gWI}_n\triangleq \frac{1}{N_n} \sum_{i=1}^{N_n} \sqrt{\frac{1}{\left|\widehat{\Omega}_i^f\right|} \sum_{\boldsymbol{\omega} \in \widehat{\Omega}_j^f}\left\|\bar{r}_j(\boldsymbol{\omega})\right\|_2^2},
\end{equation}
where $\widehat{\Omega}_j^f \triangleq\left\{\boldsymbol{\omega} \in \Omega \mid \phi_j^{\dagger} \circ \widehat{\phi}_j(\boldsymbol{\omega}) \in F\right\}$ with $F$ the foreground region of the initial phantom before distortion and $\phi_j^{\dagger}$ the ground-truth deformation for $u_j$, and 
\begin{equation}
\bar{r}_j(\boldsymbol{\omega}) \triangleq r_j(\boldsymbol{\omega})-\frac{1}{N_n} \sum_{j^{\prime}=1}^{N_n} r_{j^{\prime}}(\boldsymbol{\omega}), \quad r_j(\boldsymbol{\omega}) \triangleq \phi_j^{\dagger} \circ \widehat{\phi}_j(\boldsymbol{\omega})-\boldsymbol{\omega}.
\end{equation}
The gWI provides a fine-grained measurement for voxel-wise co-registration accuracies, whereas the semantic evaluation metrics attend to high-level structural concurrence of distinct tissues.
Besides, since all data have isotropic physical spacing, we measured ASSD and gWI in voxel units for convenience.

\subsection{Multi-Modal \& Intersubject Groupwise Registration}
\subsubsection{Experimental Design}
This experiment aims to showcase the performance of our model under various data conditions using the four datasets. 
These conditions include multi-modal data (all datasets except OASIS), heterogeneous data such as BraTS-2021 of brain tumours and MS-CMRSeg of myocardium infarction, and intersubject groupwise registration (OASIS). 
In addition, our objective was to demonstrate the efficiency of the partial learning strategy for our model. To this end, we examined four variants of our model, namely \emph{Ours-CN}, \emph{Ours-CD}, \emph{Ours-PN}, and \emph{Ours-PD}. 

In this experiment, the test groups of each dataset maintained their original sizes, which was equal to the training group sizes used for both the baselines and complete learning of our model. However, in the case of partial learning for our model, the training group size was consistently set to 2, achieved by randomly selecting images before the experiment, rather than in each training iteration. 
This resulted in a reduced number of training images, posing additional challenges for model generalisability. Note that our model does not have a partial learning version for the Learn2Reg dataset (as each original group is an MR-CT pair) and the OASIS dataset (as it is mono-modal).

\subsubsection{Quantitative Results of Registration Accuracy}

\begin{table*}[t]
  \scriptsize
  \centering
  \caption{Evaluation metrics of the groupwise registration results on the MS-CMRSeg, BraTS-2021, Learn2Reg, and OASIS datasets. The top and second-best results for each dataset are highlighted in bold and underline, respectively. 
  The ASSDs were measured in voxel units. The parameter counts are expressed in millions, and for our model there are test and training (in parentheses) values. 
  The $p$-values were computed using the gWIs or DSCs (for the BraTS-2021 and other datasets, respectively) between the method \emph{Ours-CD} and the others with a two-sided paired $t$-test. 
  $\abs{\det J_{\phi}\leq 0}$ represents the proportion (in \%) of voxels with negative Jacobian determinants in the predicted displacements, where the values were first calculated for the foreground region of each registered image and then averaged over all images among all test groups.
  }
  \begin{tabular}{C{1.5cm}@{\hspace{2ex}}C{1.4cm}@{\hspace{2ex}}C{1.4cm}@{\hspace{3ex}}C{1.4cm}@{\hspace{3ex}}C{1.4cm}c@{\hspace{0cm}}C{2.1cm}@{\hspace{3ex}}C{1.4cm}@{\hspace{3ex}}C{1.4cm}@{\hspace{3ex}}C{1.4cm}@{\hspace{3ex}}C{1.4cm}}
    \toprule
    \multirow{2}{*}{Method} & \multicolumn{4}{c}{MS-CMRSeg} & & \multicolumn{5}{c}{BraTS-2021}\\
    \cmidrule{2-5}\cmidrule{7-11} 
    & DSC $\uparrow$ & ASSD $\downarrow$ & \#Params & $p$-value & &gWI $\downarrow$ & DSC $\uparrow$ & ASSD $\downarrow$ & \#Params & $p$-value\\
    \midrule
    None & $.722\pm .104$ & $3.86\pm1.43$  & N/A & $<10^{-10}$ && $1.430\pm 0.644$ & 
$.610\pm .150$ & $0.93\pm 0.42$ &N/A& $<10^{-10}$ \\
    \hdashline\noalign{\vskip 0.5ex}
    APE \cite{journal/tpami/wachinger2012} & 
    $.746 \pm .101$ & $3.48\pm1.36$  & 0.154 & $<10^{-10}$ && $0.629\pm 0.141$ & $\underline{.707\pm .069}$ & $\bm{0.62\pm0.12}$ & 7.37 & $1.3\times 10^{-1}$ \\
    CTE \cite{journal/media/polfliet2018}  & 
    $.766 \pm .100$ & $3.15\pm1.32$  & 0.154 & $<10^{-10}$ && $1.223\pm 0.255$ & $.500\pm .102$ & $1.27\pm0.45$ & 7.37 & $<10^{-10}$ \\
    $\mathcal{X}$-CoReg \cite{journal/tpami/luo2022} & 
    $.757 \pm .107$ & $3.31\pm1.40$  & 0.154 & $<10^{-10}$ && $0.728\pm 0.196$ & $.698\pm .086$ & $0.65\pm0.17$ & 7.37 & $<10^{-10}$ \\
    \hdashline\noalign{\vskip 0.5ex}
    APE-Att & 
    $.799\pm .061$ & $2.78\pm0.72$  & 8.04 & $<10^{-10}$ && $0.757\pm 0.153$ & 
$.690\pm .078$ & $0.67\pm 0.14$ & 22.95 & $<10^{-10}$ \\
    CTE-Att & 
    $.820\pm .066$ & $2.45\pm0.74$  & 8.04 & $<10^{-10}$ && $0.916\pm 0.210$ & 
$.661\pm .094$ &$0.75\pm 0.20$& 22.95& $<10^{-10}$ \\
    \hdashline\noalign{\vskip 0.5ex}
    Ours-PN & $.803\pm .062$ & $2.68\pm 0.70$ & 5.06 (5.22) & $<10^{-10}$  && $\underline{0.608\pm 0.115}$ & $.670\pm .071$ & $0.69\pm 0.13$ & 14.91 (15.10) & $<10^{-10}$ \\ 
    Ours-CN & $\bm{.842\pm .051}$ & $\bm{2.17\pm 0.58}$ & 5.06 (5.22) & $<10^{-10}$&& $\bm{0.538\pm 0.108}$ & $\bm{.709\pm .063}$ & $\underline{0.62\pm 0.15}$ & 14.91 (15.10) & $<10^{-10}$ \\ 
    Ours-PD & $.799\pm .060$ & $2.74\pm 0.65$  & 1.91 (2.85) & $<10^{-10}$  && $0.720\pm 0.151$ & $.670\pm .072$ & $0.70\pm 0.14$ & 5.44 (15.06) & $<10^{-10}$ \\ 
    Ours-CD & $\underline{.836\pm .043}$ & $\underline{2.21\pm 0.47}$  & 1.91 (2.85) & N/A  && $0.663\pm 0.129$ & $.669\pm .074$ & $0.70\pm 0.15$ & 5.44 (15.06) & N/A \\ 

    \midrule
    \multirow{2}{*}{Method} & \multicolumn{4}{c}{Learn2Reg} & & \multicolumn{5}{c}{OASIS}\\
    \cmidrule{2-5}\cmidrule{7-11} 
      & DSC $\uparrow$ & ASSD $\downarrow$ & \#Params & $p$-value & & $\abs{\det J_{\phi}\leq 0}$ (\%) $\downarrow$ & DSC $\uparrow$ & ASSD $\downarrow$ & \#Params & $p$-value\\
      \midrule
        None & $.415\pm .160$ & $5.27\pm3.01$  & N/A & $2.4\times 10^{-4}$ &&N/A& $.614\pm .027$ & $0.96\pm 0.10$  & N/A & $< 10^{-10}$\\
    \hdashline\noalign{\vskip 0.5ex}
    APE \cite{journal/tpami/wachinger2012} & $.554\pm.384$ & $6.14\pm7.66$  & 10.52 & $9.7\times 10^{-2}$ &&$.5297\pm.0814$&$.777\pm .030$ & $0.59\pm 0.09$ & 8.85 & $7.9\times 10^{-8}$ \\
    CTE \cite{journal/media/polfliet2018}  & $.514\pm.373$ & $6.85\pm7.89$  & 10.52 & $5.1\times 10^{-2}$ &&$.2291\pm.0447$&$.746\pm .031$ & $0.62\pm 0.09$ & 8.85 & $< 10^{-10}$\\
    $\mathcal{X}$-CoReg \cite{journal/tpami/luo2022} & $.664\pm.361$ & $6.86\pm14.1$  & 10.52 & $3.0\times 10^{-1}$ &&$.1330\pm.0404$&$.777\pm .017$ & $0.54\pm 0.05$ & 8.85 & $< 10^{-10}$\\
    \hdashline\noalign{\vskip 0.5ex}
    APE-Att & $.687\pm .092$ & $2.09\pm0.57$  & 22.95 & $2.9\times 10^{-3}$ &&$.0479\pm.0130$&$.777\pm .018$ & $0.57\pm 0.05$ & 22.95 & $< 10^{-10}$\\
    CTE-Att & $.679\pm .069$ & $2.17\pm0.56$  & 22.95 & $1.4\times 10^{-2}$ &&$.0552\pm.0154$&$.773\pm .039$ & $0.59\pm 0.13$ & 22.95 & $4.2\times 10^{-6}$\\
    \hdashline\noalign{\vskip 0.5ex}
    Ours-CN & $\bm{.803\pm .084}$ & $\bm{1.32\pm 0.57}$ & 14.90 (15.10) & $9.3\times 10^{-2}$  &&$.1746\pm.0305$&$\bm{.803\pm .007}$ & $\bm{0.49\pm 0.02}$ & 13.16 (13.26) & $< 10^{-10}$\\ 
    Ours-CD & $\underline{.793\pm .086}$ & $\underline{1.39\pm 0.60}$ & 5.43 (15.06) & N/A  &&$\bm{.0066\pm.0029}$&$\underline{.791\pm .008}$ & $\underline{0.51\pm 0.02}$ & 3.69 (13.22) & N/A\\
    \bottomrule
  \end{tabular}
  \label{tab:four_datasets_results}
\end{table*}

\cref{tab:four_datasets_results} summarises the means and standard deviations of the evaluation metrics (DSC, ASSD, and gWI) on the four datasets before and after groupwise registration by different methods. 
The effectiveness of iterative and deep learning baselines varies across datasets. 
For example, all iterative approaches show inferior performance on MS-CMRSeg and Learn2Reg compared to deep learning baselines. 
Particularly on Learn2Reg, the evaluation reveals that APE and CTE yield marginal improvements in mean DSC, and all iterative methods even result in worse mean ASSDs after registration. 
In addition, deep learning baselines performed worse than APE and $\mathcal{X}$-CoReg on BraTS-2021, and all baselines (except for CTE) exhibit similar performance on OASIS. 
These results show the instability of both types of baselines under complex image conditions. 
Besides, their performance was achieved at the expense of either time-consuming iterative optimisation for each test group, or a substantial increase in parameter count for deep learning baselines.

In contrast, our model with complete learning attained notable improvements in registration accuracy, consistently outperforming all other methods across all datasets. For example, the two variants of our model demonstrate an over 15\% enhancement in DSC and an over 33\% reduction in ASSD on Learn2Reg, compared with the best baseline. On BraTS-2021, our model also surpasses all baselines, especially in gWI. It is important to highlight that gWI assesses the registration accuracy for all spatial locations within the images, providing a comprehensive evaluation, while DSC and ASSD are based only on tumour region. Therefore, the results showcase the superiority and versatility of our model in aligning various brain structures with normal and pathological tissues and diverse imaging patterns.

Remarkably, the proposed method also exhibits a significant reduction in parameters for both training and test. 
For instance, compared with the deep learning baselines, our models with CNN or Demons based registration modules achieve reductions of 37\% or 76\% (for test) on MS-CMRSeg, respectively. Notably, the Demons-based models, namely
\emph{Ours-PD} and \emph{Ours-CD}, require even fewer parameters than iterative methods on the 3D datasets during test. This demonstrates the parameter efficiency of Demons-based variants of the proposed model. 

In addition, the partial learning strategy allows our model to reduce the training image group size to 2, while maintaining performance equal to or surpassing that of the baselines. This reduction results in a 33\%/50\%/50\% (for the MS-CMRSeg/BraTS-2021/OASIS datasets) decrease in training image resources and approximately the same reduction ratio in memory footprint during training. The results illustrate that our model, trained with partial learning and reduced computational costs on a limited number of images, achieves a level of generalisability comparable to the baselines that leverage the entire dataset. This showcases the potential of our approach in scenarios involving image groups with absent modalities and limited resources.

Moreover, examining the proportions of voxels with negative Jacobian determinants in the displacements predicted by various methods on the OASIS dataset reveals that all methods exhibit decent smoothness. Notably, \emph{Ours-CD} achieved a significantly lower proportion than other methods, highlighting the advantage of the proposed registration modules with the Demons algorithm in predicting diffeomorphic transformations. 
Readers are referred to Section F of the Supplementary Material for the proportion values on all four datasets.

\subsubsection{Robustness on Different Initial Misalignment}

\begin{figure}[t]
  \centering
  \includegraphics[width=\linewidth]{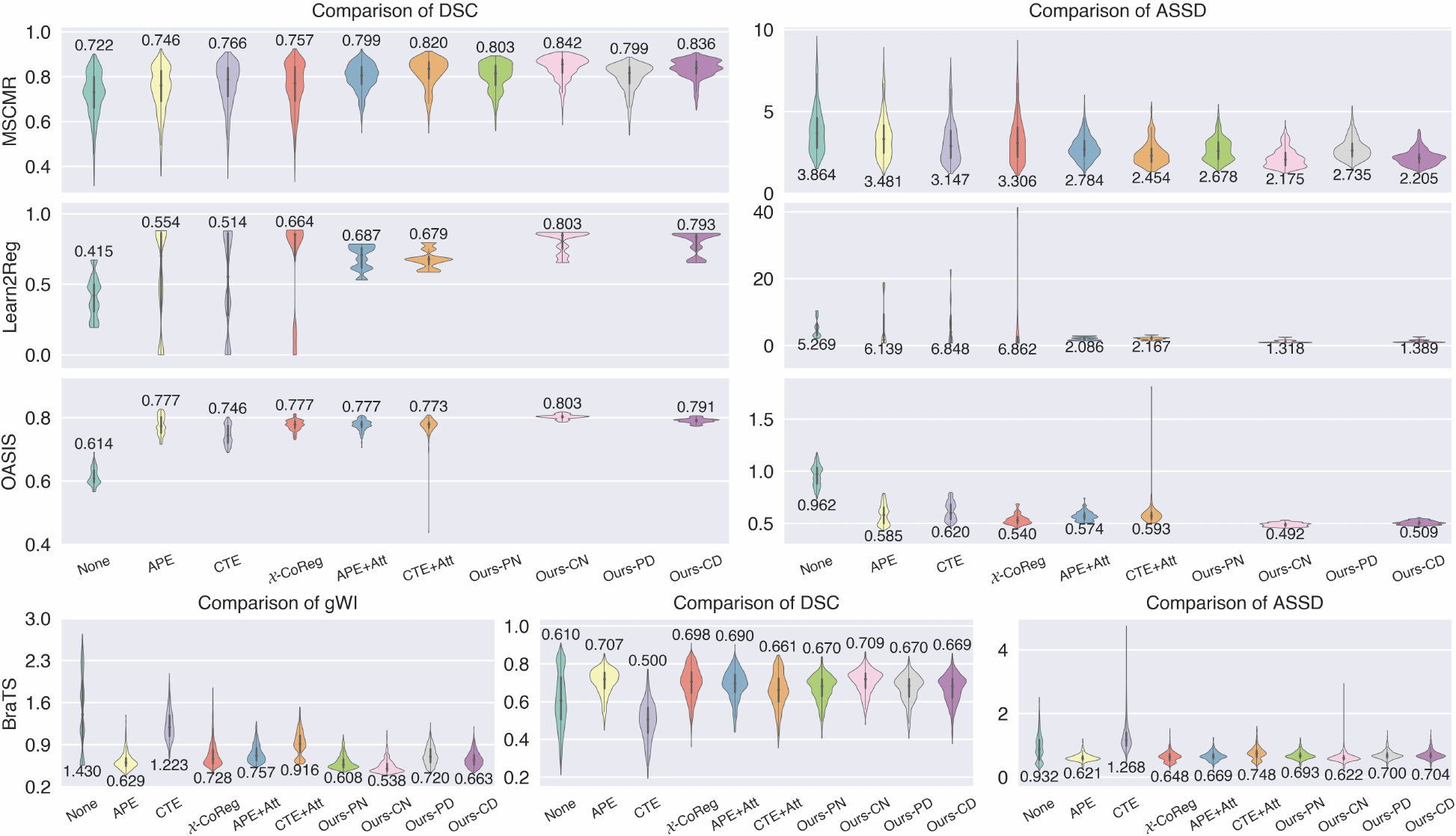}
  \caption{Quantitative evaluation metrics of the compared methods on the test groups of the four datasets. The mean values from each method are indicated.
  A zoomed version of the figure can be found in the Supplementary Material.}
  \label{fig:violin}
\end{figure}

\cref{fig:violin} presents the violin plots of the evaluation metrics for the compared methods on all test groups. 
For all datasets (particularly MSCMR, Learn2Reg and BraTS-2021), the initial DSCs/ASSDs/gWIs of the test image groups exhibit a broad spread, indicating considerable variability of misalignments. 
However, iterative methods achieved only marginal alignments for nearly all groups of the MSCMR dataset, and even led to worse distortion for certain Learn2Reg image groups, even though such methods incorporate a combination of rigid, affine, and FFD transformations to enhance registration outcomes.
This implies that the limited capture range of iterative methods makes it difficult to establish reasonable spatial correspondence across long distances. This challenge may arise from the presence of local minima in the optimisation of manually crafted similarity metrics. 
Although deep learning baselines can mitigate the wide range of distortion in misaligned images, the overall improvement by them remains modest on certain datasets such as Learn2Reg. 

In contrast, the violin plots of our models with complete learning achieve shorter tails (smaller variance in registration metrics) for all datasets, exhibiting greater stability when handling varying levels of distortion, and showing superior performance especially for image groups with significant misalignments. 

\subsubsection{Visualisation of Registration Results}

The results of an example test group from the MS-CMRSeg and Learn2Reg datasets are visualised in \cref{fig:vis_mscmr} and \cref{fig:vis_learn2reg}, respectively. 
For the example from MS-CMRSeg, the original images were severely distorted with initial DSCs below 0.4. Under such condition, the iterative methods were rather conservative to output small-magnitude deformations, illustrated by the deformation grids. 
For the group from Learn2Reg, conversely, the iterative methods estimated highly irregular deformation fields containing numerous non-diffeomorphic positions, as shown by the displacement and Jacobian determinant maps, although they achieved DSCs comparable to the proposed methods on this image group. 
This illustrates the instability of the deformations produced by iterative methods across different images. In contrast, while deep learning baselines avoid registration failures for all datasets, their output deformation fields imply that such an effect comes at the cost of a tendency to predict more conservative deformations with high rigidity. 
Readers are referred to Section F of the Supplementary Material for visualisation of the registration results on the BraTS and OASIS datasets.

The proposed model, however, demonstrates an advantage in addressing these issues, capable of generating fine-grained deformations that contain both global and localised subtle movements. To further showcase the effectiveness of the proposed hierarchical decomposition strategy for estimating the velocity fields, in \cref{fig:multilevel_disps_mscmr} we visualise an example of the multi-level deformations produced by our model on an image group from the MS-CMRSeg dataset. The estimated spatial transformations $\{\phi_m^l\}_m$ with a lower level $l$ tend to be smoother and consist of global displacements, while the higher-level deformations concentrate on small distortions in different local regions. The visualisation results for the other three datasets are provided in Section F of the Supplementary Material.

\begin{figure}[t]
  \centering
  \includegraphics[width=\linewidth]{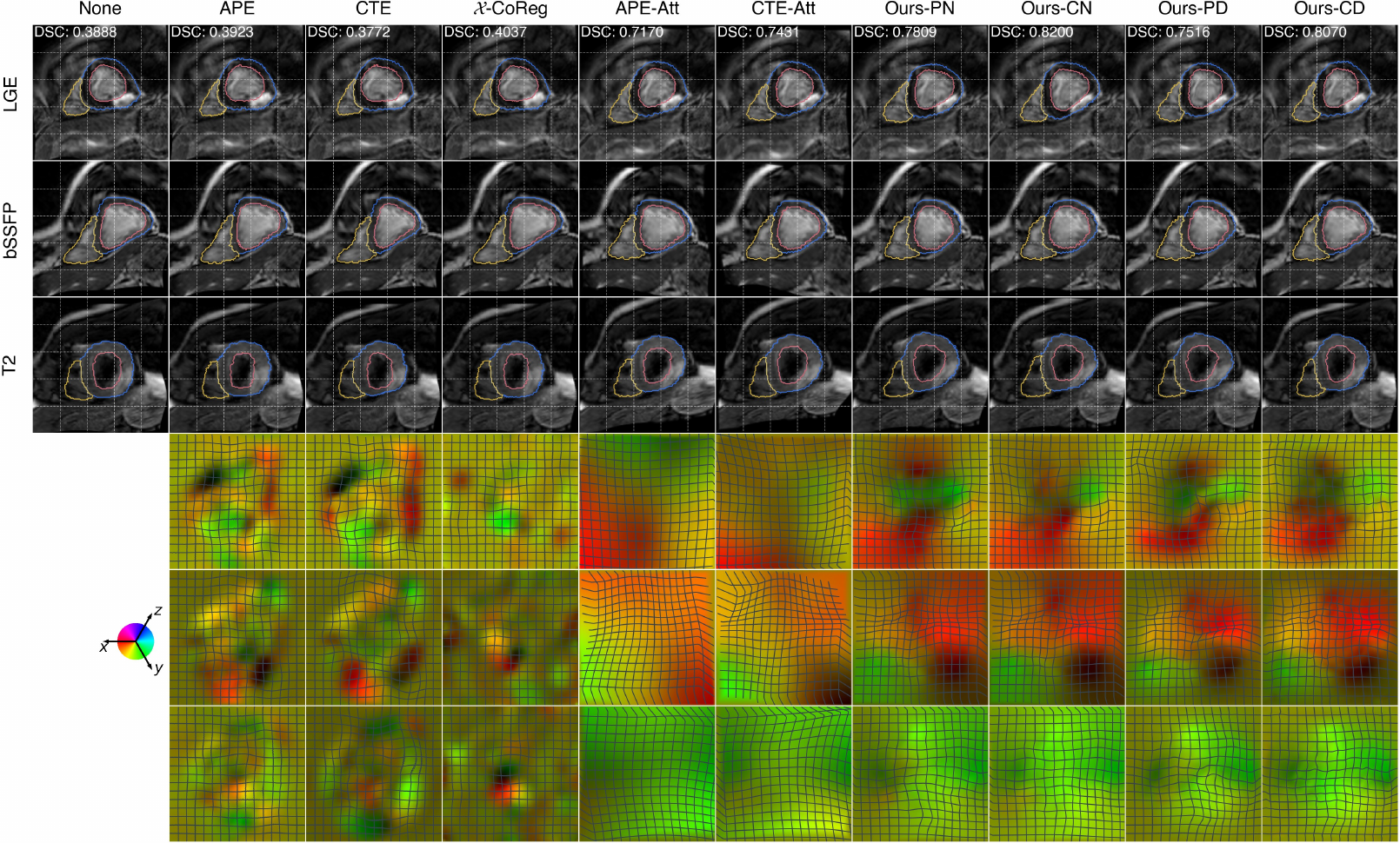}
  \caption{Results of an image group from the MS-CMRSeg dataset. The mean DSCs of all foreground classes on this group are shown for each method.
  A zoomed version of the figure can be found in the Supplementary Material.}
  \label{fig:vis_mscmr}
\end{figure}

\begin{figure}[t]
  \centering
\includegraphics[width=\linewidth]{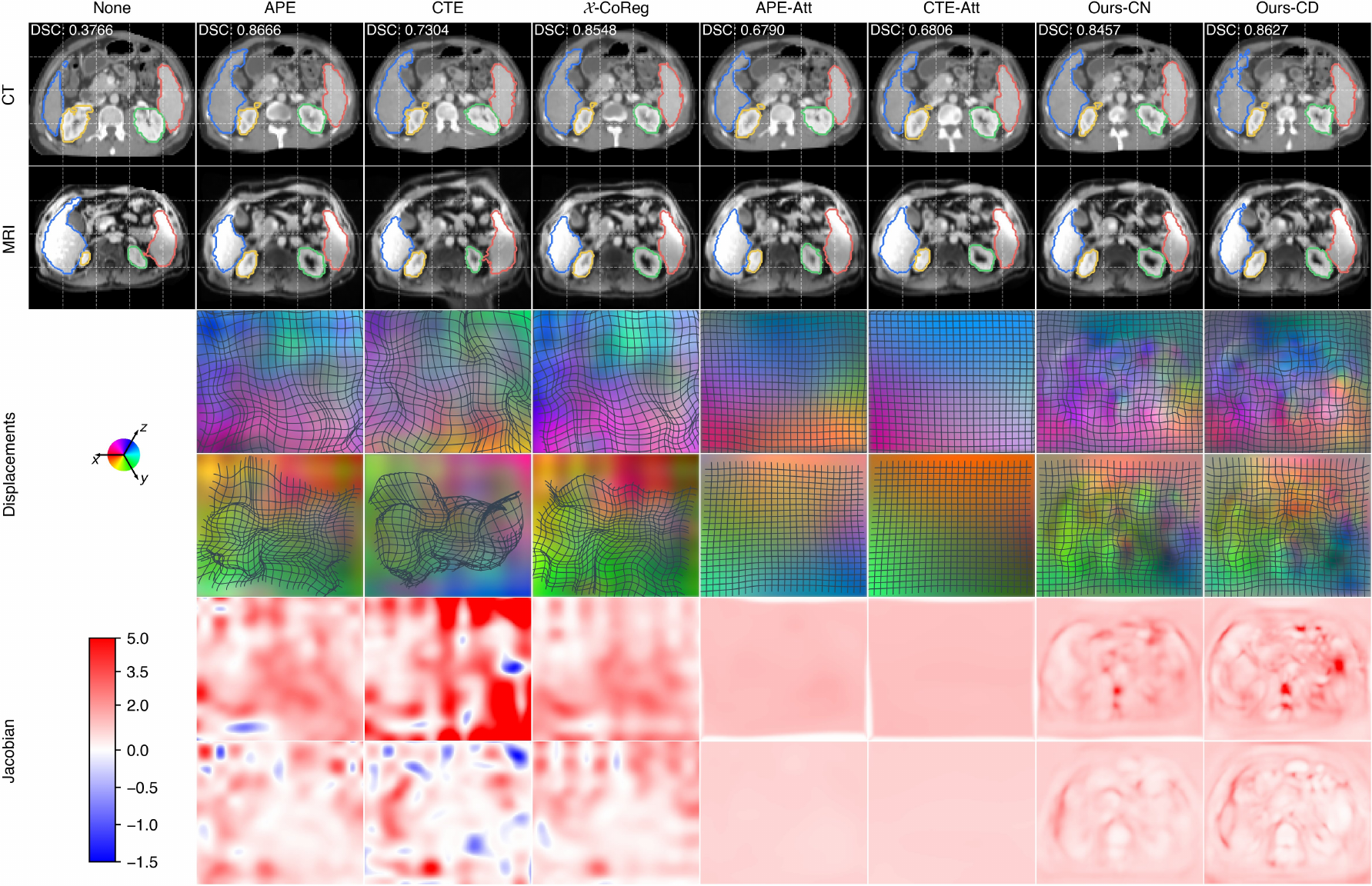}
  \caption{Results of an image group from the Learn2Reg dataset. The mean DSCs of all foreground classes on this group are shown for each method.
  A zoomed version of the figure can be found in the Supplementary Material.}
  \label{fig:vis_learn2reg}
\end{figure}

\begin{figure}[t]
  \centering
  \includegraphics[width=\linewidth]{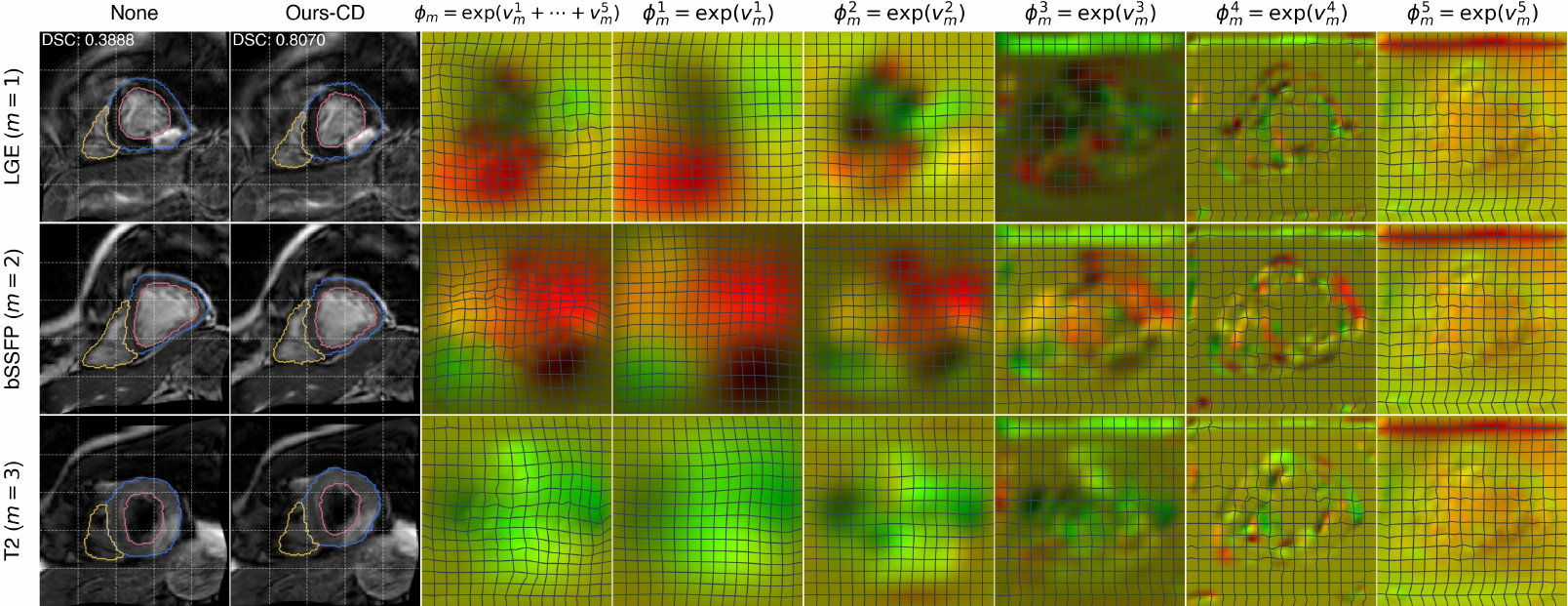}
  \caption{Multi-level deformations from our model \emph{Ours-CD} on MS-CMRSeg, where the image group to register is the same as in \cref{fig:vis_mscmr}.
  A enlarged version of the figure can be found in the Supplementary Material.}
  \label{fig:multilevel_disps_mscmr}
\end{figure}

\subsection{Scalability Test on Large-scale and Variable-size Image Groups}
\label{sec:scalability}

\begin{figure}[t]
  \centering
  \includegraphics[width=\linewidth]{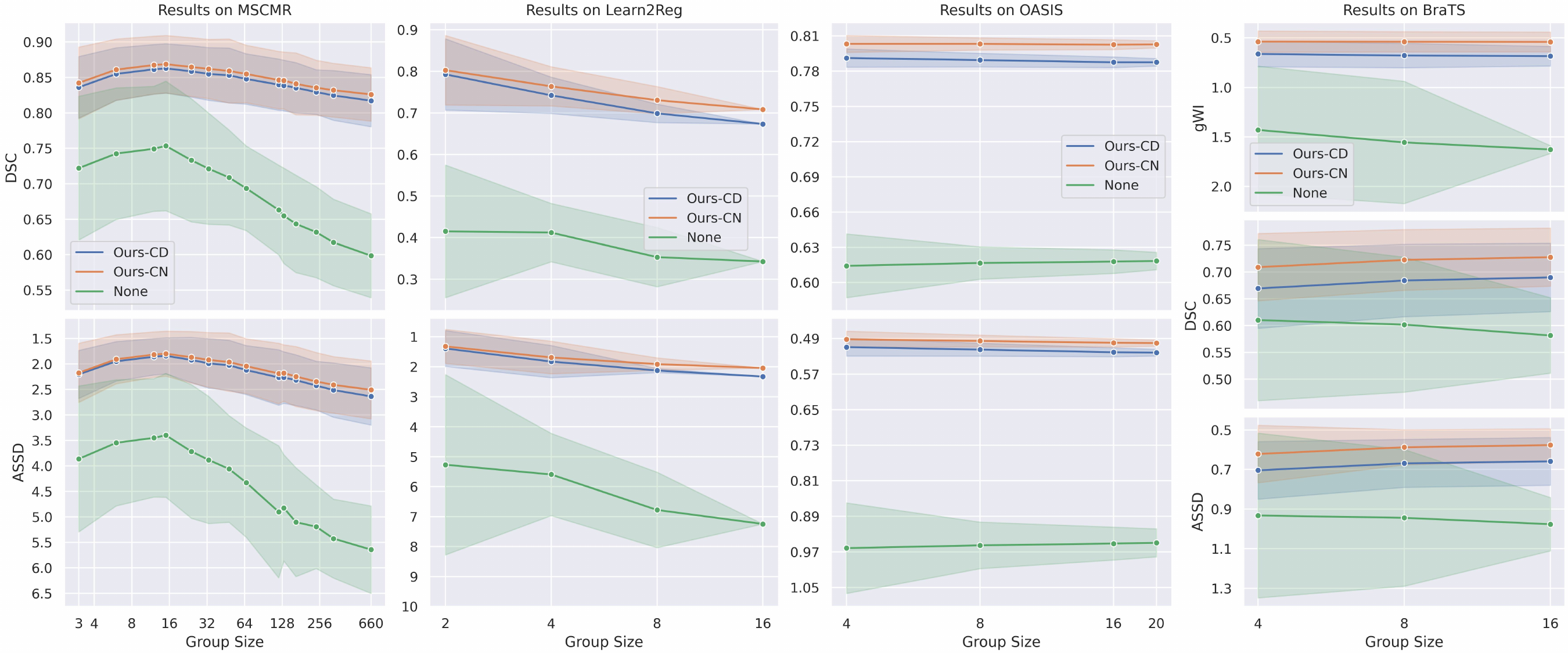}
  \caption{Evaluation metrics (mean values with one standard deviation bands) of registration results on image groups with different sizes.
  A zoomed version of the figure can be found in the Supplementary Material.}
  \label{fig:scalability}
\end{figure}

\subsubsection{Experimental Design}
This experiment aims to evaluate the performance of our model for large-scale and variable-size groupwise registration. Co-registering large image groups poses a significant challenge for iterative methods due to their high computational complexity. On the other hand, traditional deep learning approaches typically fix the input channel number as a predetermined group size, limiting them to handling only small groups of the same size during training and test due to GPU memory constraints. In contrast, while only requiring a small group size for training, our model can generalize effectively to much larger image groups of varying sizes during inference stage. The results in this section illustrate this scalability.

The original group sizes of the BraTS-2021, MS-CMRSeg, Learn2Reg and OASIS datasets were $N_{\text{train}}=4,3,2,4$, respectively. 
To test the scalability of our model, we constructed test groups of a larger size for each dataset, by merging every $R$ original groups as one new group. 
Therefore, using the original test set with $W$ groups of size $N_{\text{train}}$, we constructed a new test set with $W/R$ groups of size $N_{\text{test}}=N_{\text{train}}R$, where $W\equiv 0 \pmod{R}$. 
The variants of our model (\emph{Ours-CN} and \emph{Ours-CD}) were trained on the original training groups, and then tested on newly constructed test sets with different $N_{\text{test}}$. 

\subsubsection{Results}

\cref{fig:scalability} presents
the evaluation metrics versus $N_{\text{test}}$ on different datasets.
The initial DSC/ASSD/gWI metrics indicate that as $N_{\text{test}}$ increases, the initial misalignment becomes significantly severe, posing greater challenges for co-registration. 
However, our model sustains good performance even when group size is very large (e.g., more than 600 2D images from MS-CMRSeg and 20 3D images from OASIS). Besides, while the pre-registration metrics on BraTS-2021 worsen with increasing group size, our models achieved even better registration accuracy when co-registering larger image groups. These demonstrate robustness of the proposed framework on large-scale multi-modal groupwise registration.
In addition, the two variants of our model exhibit similar scalability, showing the effectiveness of both types of registration modules. 

Note that the maximum test group size for Learn2Reg and BraTS-2021 is 16, because Learn2Reg contains only 16 test images, and for BraTS, it is infeasible to register intersubject images due to difference in tumour structures, and there are only 16 distorted images for each subject.

\subsection{Integration with Other State-of-the-Art Registration Methods}
\subsubsection{Experimental Design}
This experiment aims to integrate state-of-the-art (SoTA) methods from other tasks (e.g., pairwise or mono-modal registration) into the proposed framework, and therefore evaluates the compatibility and versatility of our method. To this end, we selected the following recent works:
\begin{itemize}
    \item TransMorph \cite{media/transmorph}. We replaced the Attention U-Net encoder of our model with TransMorph to extract features from different image modalities. 
    \item PIViT \cite{miccai/PIViT}. We replaced the registration modules of our model by the registration networks proposed by PIViT. Particularly, as both our and the PIViT models have five levels of scales (resolutions) for spatial transformation inference, we follow the same settings in that paper, i.e., we use the LCD modules (Long-range Correlation Decoder) proposed by the PIViT paper for the three coarsest levels, and use CNNs for the other two levels.
    \item ModeT \cite{miccai/modet}. We replaced the registration modules of our model by the registration networks proposed in the ModeT paper. Particularly, as both our and the ModeT models have five levels, we follow the same settings in that paper, i.e., for each of the three coarsest levels, we first use a ModeT module (Motion Decomposition Transformer) to infer multiple transformations, and then use a CWM (Competitive Weighting Module) to generate a single transformation; for each of the other two levels, we use a ModeT module to infer a single transformation directly. 
\end{itemize}
Note that the model Ours-CN uses the networks in the registration modules to produce both mean and log variance of the velocity field distributions, while Ours-CD calculated the mean of the velocity field distributions using the Demons algorithm and uses the networks in the registration modules to produce logarithmic variance only. Therefore, with Transmorph or PIViT networks integrated, our model still has these two types. As for ModeT, since it was designed to calculate spatial transformations based on a set of basis vectors, it is not suitable to be used as a predictor of logarithmic variance. As a result, we integrate ModeT only into Ours-CN.

\subsubsection{Results}

\begin{table}[t]
\caption{The results of our models integrated with SoTA methods on the MS-CMRSeg dataset.}
\label{tab:compatability}
\centering
\scriptsize
\begin{tabular}{ccccc}
\hline
Model   & Encoder         & Reg. Module & DSC$\uparrow$  & ASSD$\downarrow$ \\ \hline
Ours-CN & Attention U-Net & Convs       & 0.842$\pm$0.051 & 2.17$\pm$0.58    \\
Ours-CD & Attention U-Net & Convs       & 0.836$\pm$0.043 & 2.21$\pm$0.47    \\ \hdashline
Ours-CN & TransMorph      & Convs       & 0.838$\pm$0.051 & 2.21$\pm$0.54    \\
Ours-CD & TransMorph      & Convs       & 0.848$\pm$0.048 & 2.09$\pm$0.52    \\
Ours-CN & Attention U-Net & PIViT       & 0.841$\pm$0.043 & 2.18$\pm$0.47    \\
Ours-CD & Attention U-Net & PIViT       & 0.838$\pm$0.045 & 2.26$\pm$0.49    \\
Ours-CN & Attention U-Net & ModeT         & 0.848$\pm$0.040 & 2.12$\pm$0.42    \\ \hline
\end{tabular}
\end{table}

The results on the MS-CMRSeg dataset are shown in \cref{tab:compatability}. Our models integrated with other SoTA methods achieved a similar level of performance compared to the results we obtained before (first two rows in the table). In particular, Ours-CD with TransMorph encoders and Ours-CN with ModeT registration modules even surpass the previous best performance. Therefore, it is demonstrated that our method has good versatility, compatible with different SoTA methods. Besides, this indicates that we can effectively apply methods that dominate various registration tasks to multi-modal groupwise registration, just through integration with our framework.

\subsection{Model Interpretability}
We further demonstrate the interpretability of the proposed framework by
1) visualising the structural representation maps learnt by our model, and
2) validating the symmetry structures respected by our model.

\subsubsection{Structural Representation Maps}
\begin{figure}[t]
  \centering
  \includegraphics[width=\linewidth]{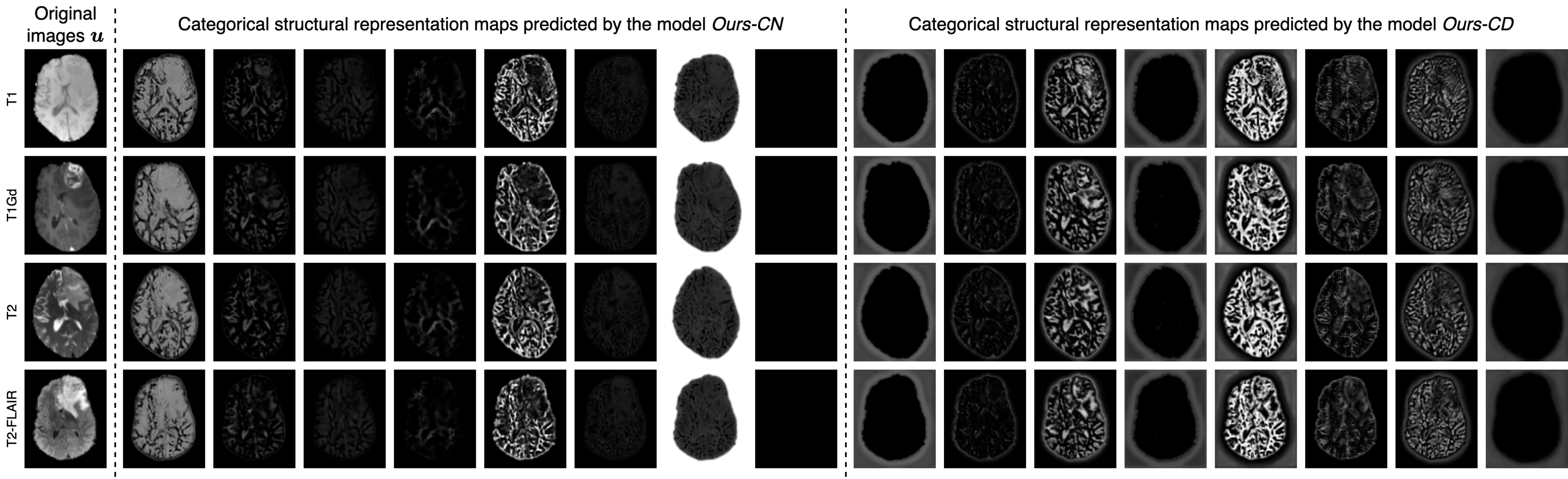}
  \caption{Categorical structural representations from the proposed models. 
  One can see that complementary brain structures are revealed from these representations.
  Particularly, the representations from the model \emph{Ours-CD} look more fine-grained than those from \emph{Ours-CN}.
  A zoomed version of the figure can be found in the Supplementary Material.}
  \label{fig:features_brats}
\end{figure}

We visualise the structural representation maps as the categorical latent distribution learnt by the proposed model.
For instance, \cref{fig:features_brats} presents the finest-scale structural representation maps estimated by the proposed models of an image group from the BraTS-2021 dataset.
One can observe that meaningful semantic features are extracted from the input images.
Particularly, the model \emph{Ours-CN} tends to learn high-level information, e.g. structures and boundaries, while \emph{Ours-CD} prefers to extract more detailed features such as skeletons and textures.
This difference is reasonable because the registration is computed directly from these representations for the Demons algorithm, which needs more fine-grained information.

These structural representations also demonstrate the clear advantage of our proposed models over conventional learning-based counterparts in terms of model interpretability and registration accuracy. 
This suggests that the trade-off between model interpretability and performance is not inevitable, while one may believe that more complex models are more accurate.
In fact, since medical images are scarce and highly correlated \cite{journal/tpami/luo2022}, the observations are well structured and only occupy a low-dimensional landscape (or manifold) within the entire data space \cite{journal/if/arrieta2020}.
Therefore, by respecting the underlying generating process and the inherit symmetry structures, the proposed models achieved better interpretability and higher accuracy, while using fewer parameters.

\subsubsection{Counterfactual Validation of Latent Symmetries}
\reva{}{
A deeper level of interpretability can be assessed by examining if the model respects the underlying symmetries of the data-generating process, namely the disentanglement of anatomy and geometry.
We validate this by performing counterfactual reconstructions
\cite{book/cambridge/pearl2009,conference/iclr/besserve2018}, where we actively intervene on the learnt latent variables and observe the outcome.
This allows us to ask ``what-if'' questions and verify that the model behaves in a predictable and semantically meaningful way.}{AE}

\reva{}{
We first investigate the disentangled anatomy representation $\bm{z}$.
For a given test image group, we compute the latent anatomy $\bm{z}$ and then perform an ontological intervention by setting a specific channel $\bm{z}_k$ to zero (i.e., a $\operatorname{do}(\bm{z}_k=\bm{0})$ operation).
This corresponds to asking the model: ``What would the registered images look like if this specific anatomical feature $k$ did not exist''.
The results of this experiment on the OASIS dataset are shown in \cref{fig:vis_oasis_counter_recon}.
The figure demonstrates that zeroing out a single channel cleanly removes a specific anatomical structure (e.g., the white matter) from the reconstruction.
The difference map between the original and counterfactual reconstructions--the direct effect--clearly isolates this structure.
To quantify this observation, we measure the Normalised Cross-Correlation (NCC) between the map of the latent channel $\bm{z}_k$ that was set to zero and the corresponding direct effect map.
For the example in \cref{fig:vis_oasis_counter_recon}, the four non-zero channels $\bm{z}_1$, $\bm{z}_2$. $\bm{z}_4$ and $\bm{z}_6$ yielded high NCC values of 0.99, 0.98, 0.97 and 0.62, respectively.
This high correlation confirms that the model has learnt a highly disentangled representation, where individual latent channels have a direct and exclusive correspondence to specific anatomical structures.
}{}

\reva{}{
This level of granular control and semantic correspondence is a key advantage over conventional end-to-end registration networks. 
In such ``black-box'' models, the latent features are typically entangled, and manipulating a single feature is unlikely to produce a coherent or predictable semantic change. 
Our framework, by explicitly modelling anatomy and geometry as disentangled latent variables, provides interpretable latent that can be used to understand and validate the model's internal logic, a crucial step towards trustworthy medical AI. 
We also validate the geometric equivariance in \cref{fig:vis_oasis_deform_recon}, showing that the model correctly disentangles spatial transformations.
}{}

\begin{figure*}
  \centering
  \begin{subfigure}{0.49\textwidth}
    \centering
    \includegraphics[width=\linewidth]{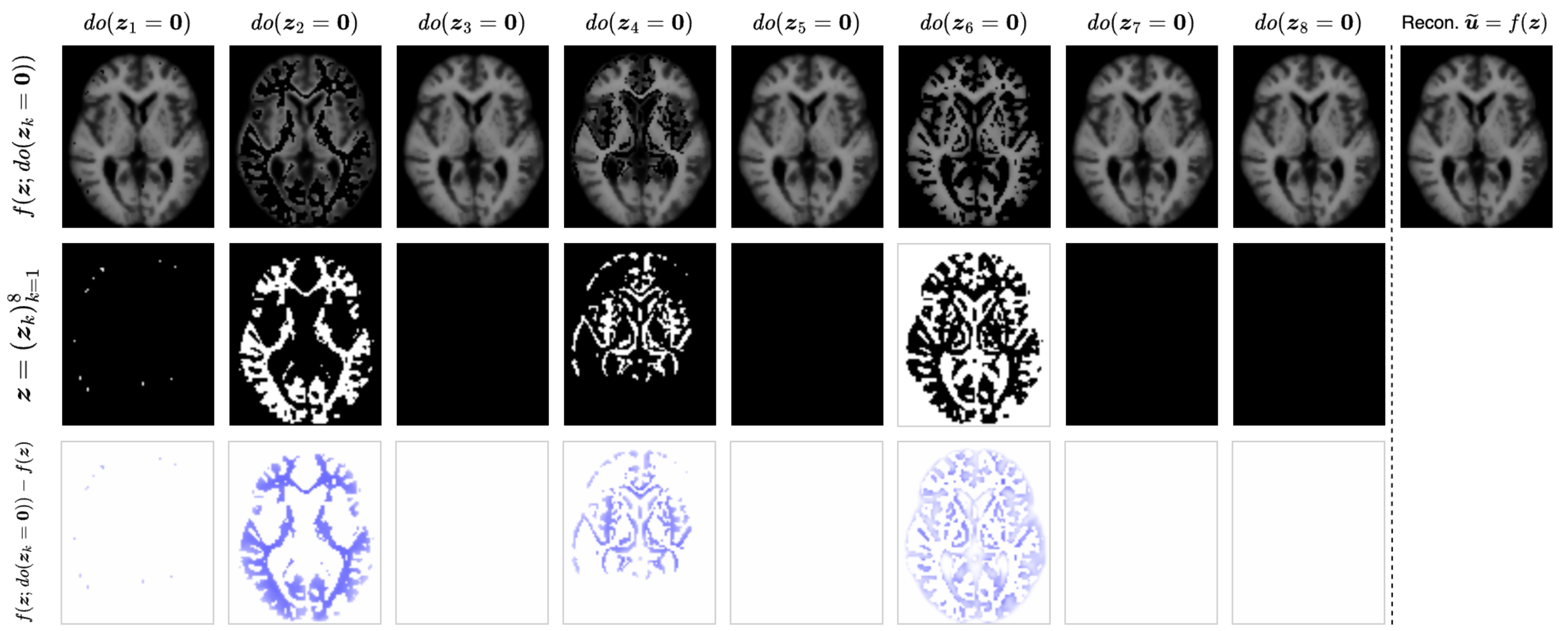}
    \caption{Counterfactual reconstruction by ontological transformations.}
    \label{fig:vis_oasis_counter_recon}
  \end{subfigure}
  \hfill
  \begin{subfigure}{0.49\textwidth}
    \centering
    \includegraphics[width=\linewidth]{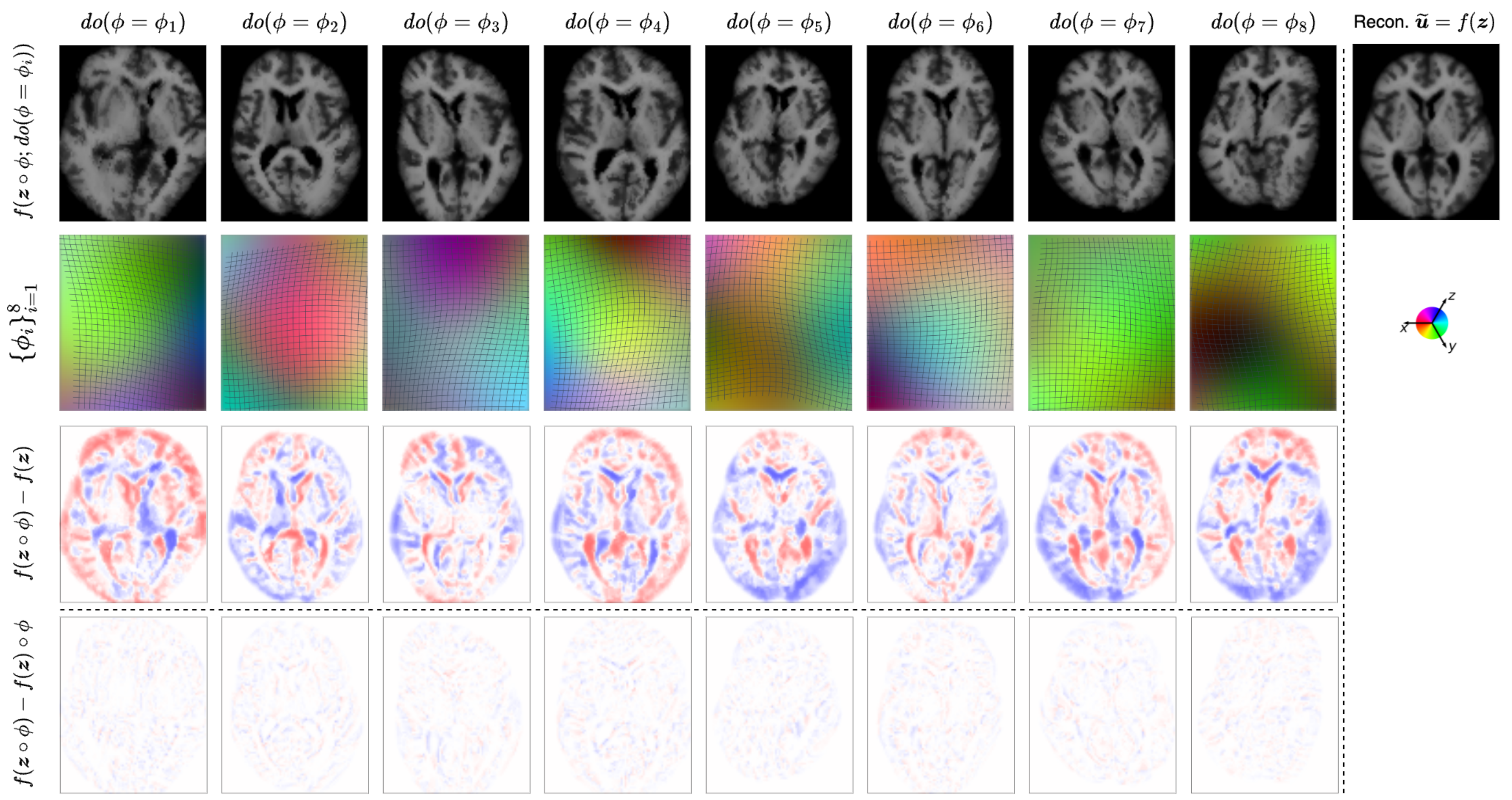}
    \caption{Counterfactual reconstruction by diffeomorphic transformations.}
    \label{fig:vis_oasis_deform_recon}
  \end{subfigure}
  \caption{Counterfactual reconstruction on an image from the OASIS dataset using the underlying symmetry transformations. 
    (a) For ontological transformations acting in the anatomy domain, one can see that the difference calculated by $f(\bm{z};\mathit{do}(\bm{z}_k=\bm{0}))-f(\bm{z})$ indeed corresponds to the $k$-th latent structure $\bm{z}_k$, which means that the learnt decoder respects the ontological symmetry.
    (b) For diffeomorphic transformations acting in the spatial domain, one can observe that the equivariance difference $f(\bm{z}\circ\phi_i)-f(\bm{z})\circ\phi_i$, where $f(\bm{z}\circ\phi_i)=f(\bm{z}\circ\phi;\mathit{do}(\phi=\phi_i))$, are almost zero except for interpolation errors, which indicates that the learnt decoder is indeed transformation-equivariant.}
\end{figure*}

\section{Conclusion and Discussion}\label{sec:discussion}
In this work, we have developed a Bayesian framework for unsupervised multi-modal groupwise image registration.
In marked contrast to similarity-based registration approaches, the proposed method builds on a principled generative modelling of the imaging process.
Moreover, a specially designed network architecture realises the explicit disentanglement of anatomy and geometry from the observed images, so that registration is learnt from their underlying structural representations in a unified closed-loop self-reconstruction process.
The experiments on four different datasets of cardiac, brain, and abdominal medical images have demonstrated the advantage of the proposed modelling for image registration.

There have also been works in the literature trying to learn disentanglement of appearance and geometry \cite{conference/eccv/shu2018,conference/neurips/skafte2019,conference/miccai/bone2020,journal/tpami/xing2020},
or inference of semantic or geometric variations \cite{conference/cvpr/ying2023,conference/iclr/mao2024}.
However, they usually involve modelling of an intensity template, and do not consider a more general multi-modality setup.
The generated template trained on one dataset can be suboptimal for groupwise registration on a new dataset.
This highlights the advantage of disentangling each separate image groups into their corresponding anatomical representations in our framework as a subroutine to groupwise registration, rather than optimising for intensity variations. 

The intrinsic distance proposed in \cref{sec:intrinsic_distance} also has further implications. 
Developed based on the variational inference framework, the intrinsic distance is primarily intended to estimate the difference between structural representations of the common anatomy and the input images.
Nevertheless, there could be alternative methods to calculate the average representation and the intrinsic distance.
For instance, conventional label fusion methods for multi-atlas segmentation \cite{journal/tpami/wang2012,journal/tmi/agrawal2020,journal/media/audelan2022} could be used to estimate the average probability maps of the common anatomical representations, while the intrinsic metrics between probabilistic shapes could be improved by respecting the underlying manifold structure \cite{journal/ni/fletcher2009,conference/aistats/arvanitidis2021}.

Furthermore, recent years have witnessed a growing academic interest in the field of interpretable or explainable artificial intelligence (XAI) \cite{journal/ai/gunning2019,journal/if/arrieta2020,journal/tnnls/tjoa2020,journal/tetci/zhang2021}.
Our proposed model exhibits the local and active interpretability characterised by \cite{journal/tetci/zhang2021} in that the structural representations encoded by the network are deliberately devised to have visual semantics in a human-understandable fashion.
The proposed registration module and the network architecture also demonstrate certain levels of algorithmic transparency \cite{journal/if/arrieta2020}.
Indeed, the Demons algorithm for calculating the velocity fields is mathematically interpretable, as it depends linearly on the gradient and difference of structural representations.
On the other hand, the decoder architecture is spatially decomposable, as the reconstruction of image intensities is determined locally by the anatomy at each independent location (owing to the use of convolutions with kernel size 1 isotropically), thus respecting the spatial equivariance in the imaging process \emph{w.r.t.} diffeomorphic transformations.
One of the potential trade-offs or limitations of the proposed architecture is that the usage of convolutions with kernel size 1 in the decoders may provoke a trade-off between the reconstruction fidelity and the equivariance constraint.
\revb{}{However, by sacrificing a small degree of reconstruction fidelity, we gain a significant improvement in the identifiability and accuracy of the geometric transformations. This is an example of how imposing a strong, domain-appropriate inductive bias can guide a learning system to a more robust and accurate solution for its primary task.}{R2.2}
Future studies could be addressed to represent the generative process of multi-modal images from the underlying common anatomy in a more compact and unified manner using techniques from, for example, normalising flows \cite{journal/tpami/kobyzev2020} that construct an invertible mapping between the observation and the latent variable.


\begin{appendices}

\section{Minimization of The Intrinsic Distance}
We define the KL divergence $D_S^*$ from the prior to the posterior distribution of the common anatomy structural representation as
\begin{equation}
  \begin{aligned}
    D_S^* &\triangleq\kldiv{q^*(\bm{z})}{p^+(\bm{z})} 
  \end{aligned}
\end{equation}
where $q_j^{\diamond}(\bm{z})$ is the single-view variational distribution for the $j$-th input image, and 
\begin{equation}
  \begin{aligned}
    q^*(\bm{z}) &\triangleq \frac{\left[\prod_{j=1}^N q_j^{\diamond}(\bm{z})\right]^{\nicefrac{1}{N}}}{\int\left[\prod_{j=1}^N q_j^{\diamond}(\bm{z})\right]^{\nicefrac{1}{N}}\operatorname{d}\bm{z}}, \\
    p^+(\bm{z}) &\triangleq \frac{1}{N}\sum_{j=1}^N q_j^{\diamond}(\bm{z})
  \end{aligned}
\end{equation}
are the geometric and arithmetic mean distributions of $q_j^{\diamond}$'s, respectively.
For categorical distributions $q_j^{\diamond}(\bm{z})=\operatorname{Cat}(\bm{z};\bm{\pi}_{j,1}^{\diamond},\dots,\bm{\pi}_{j,K}^{\diamond})$ with $j\in[N]$, the geometric and arithmetic mean distribution are both categorical, \emph{i.e.} $q^*(\bm{z})=\operatorname{Cat}\left(\bm{z};\bm{\pi}_1^*,\dots,\bm{\pi}_K^*\right)$ and $p^+(\bm{z})=\operatorname{Cat}\left(\bm{z};\bm{\pi}_1^+,\dots,\bm{\pi}_K^+\right)$, with 
\begin{equation}
  \begin{aligned}
    \bm{\pi}_k^* &= \frac{\left[\prod_{j=1}^N\bm{\pi}_{j,k}^{\diamond}\right]^{\nicefrac{1}{N}}}{\sum_{k=1}^K\left[\prod_{j=1}^N\bm{\pi}_{j,k}^{\diamond}\right]^{\nicefrac{1}{N}}}, \\
    \bm{\pi}_k^+ &= \frac{1}{N}\sum_{j=1}^N \bm{\pi}_{j,k}^{\diamond},
  \end{aligned}
\end{equation}
where all operations are applied element-wise.
Therefore, $D_S^*$ is minimized when $\bm{\pi}_k^*=\bm{\pi}_k^+$ for $k=1,\dots,K$, and each $\bm{\pi}_{j,k}^{\diamond}$ is characterized by
\begin{equation}
  \left\{
  \begin{aligned}
    \frac{1}{N}\sum_{j=1}^N \bm{\pi}_{j,k}^{\diamond} &=\frac{\left[\prod_{j=1}^N\bm{\pi}_{j,k}^{\diamond}\right]^{\nicefrac{1}{N}}}{\sum_{k=1}^K\left[\prod_{j=1}^N\bm{\pi}_{j,k}^{\diamond}\right]^{\nicefrac{1}{N}}} && k=1,\dots,K \\
    \sum_{k=1}^K\bm{\pi}_{j,k}^{\diamond} &= \bm{1},&& j=1,\dots,N
  \end{aligned}
  \right.
\end{equation}
which is an under-determined system with $N+K$ equations and $NK$ unknown variables.
Particularly, in general we cannot conclude that $\bm{\pi}_{1,k}^{\diamond}=\dots=\bm{\pi}_{N,k}^{\diamond}$ element-wise, $\forall\,k=1,\dots,K$.

On the other hand, for Gaussian distributions $q_j^{\diamond}(\bm{z})=\mathcal{N}(\bm{z};\bm{\mu}_j^{\diamond},\bm{\Sigma}_j^{\diamond})$ with $j\in[N]$, the geometric mean distribution is $q^*(\bm{z})=\mathcal{N}(\bm{z};\bm{\mu}^*,\bm{\Sigma}^*)$ with 
\begin{equation}
  \bm{\Sigma}^* = N\left[\sum_{j=1}^N\bm{\Sigma}_j^{\diamond\,-1}\right]^{-1},\quad 
  \bm{\mu}^* = \frac{\bm{\Sigma}^*}{N}\sum_{j=1}^N \bm{\Sigma}_j^{\diamond\,-1}\bm{\mu}_j^{\diamond},
\end{equation}
while the arithmetic mean distribution $p^+(\bm{z})$ is a Gaussian mixture model.
Therefore, when the distributions $q^*(\bm{z})$ and $p^+(\bm{z})$ are equal, the Gaussian mixture reduces to a single Gaussian, \emph{i.e.}, all the mixture components become identical, with the same mean and covariance matrix.
However, the KL divergence involving Gaussian mixture models is computationally intractable, which cannot be written as an explicit formula of its parameters. 

In practice, we use an upper bound of ${D}_S^*$ given by the Jensen's inequality as the objective function, \emph{a.k.a.} the \emph{intrinsic distance}:
\begin{equation}
  \widetilde{D}_S^*\triangleq \frac{1}{N}\sum_{j=1}^N\kldiv{q^*(\bm{z})}{q_j^{\diamond}(\bm{z})}\geq D_S^*.
\end{equation}
Note that the geometric mean distribution $q^*$ minimizes the functional $\widetilde{D}_S[q]\triangleq\frac{1}{N}\sum_{j=1}^N\kldiv{q(\bm{z})}{q_j^{\diamond}(\bm{z})}$ \emph{w.r.t.} $q$, because
\begin{equation}
  \begin{aligned}
    q^* &=\argmin_{q}\widetilde{D}_S[q] \\
    &=\argmin_{q}\left\{\kldiv{q(\bm{z})}{q^*(\bm{z})}-\log \mathcal{Z}\right\},
  \end{aligned}
\end{equation}
where $\mathcal{Z}\triangleq\int\left[\prod_{j=1}^Nq_j^{\diamond}(\bm{z})\right]^{\nicefrac{1}{N}}\operatorname{d}\bm{z}$ is a normalizing constant.
Then, we minimize the intrinsic distance $\widetilde{D}_S^*$ \emph{w.r.t.} $q_j^{\diamond}$'s.
By definition, for any distribution $q(\bm{z})$ and $p(\bm{z})$, we have $\kldiv{q(\bm{z})}{p(\bm{z})}\geq 0$, with $\kldiv{q(\bm{z})}{p(\bm{z})}=0$ if and only if $q(\bm{z})=p(\bm{z})$ (almost everywhere for continuous distributions).
Therefore, $\widetilde{D}_S^*$ is minimized and equals zero when $q^*(\bm{z})=q_j^{\diamond}(\bm{z})$ \emph{a.e.}, $\forall\,j$.
Since $\widetilde{D}_S^*\geq D_S^*\geq 0$, this also implies that $D_S^*$ is minimized and equals zero when $q^*(\bm{z})=q_j^{\diamond}(\bm{z})$ \emph{a.e.}, $\forall\,j$.

\section{The Gumbel-Rao Estimator}
For many machine learning problems, we need to optimize the expectation of a continuous differentiable function $f:\mathds{R}^K\rightarrow\mathds{R}$ over a discrete latent variable $\bm{z}\in\{0,1\}^K$, \emph{i.e.}
\begin{equation}
  \min_{\bm{\psi}}\mathbb{E}_{\bm{z}\sim q(\bm{z};\bm{\psi})}[f(\bm{z})],
\end{equation}
where $q(\bm{z};\bm{\psi})$ can be the distribution of $\bm{z}$ estimated by a neural network with parameters $\bm{\psi}$.
Unfortunately, using Monte-Carlo approximation, computing the gradient directly by 
\begin{equation*}
  \nabla_{\bm{\psi}}\triangleq\frac{\partial f(\bm{z})}{\partial\bm{z}}\frac{\operatorname{d}\bm{z}}{\operatorname{d}\bm{\psi}} 
\end{equation*}
would encounter the problem of non-differentiability with $\nicefrac{\operatorname{d}\bm{z}}{\operatorname{d}\bm{\psi}}$.
On the other hand, the REINFORCE estimator \cite{journal/ml/williams1992}
\begin{equation*}
  \nabla_{\text{REINFORCE}}\triangleq f(\bm{z})\frac{\partial\log q(\bm{z};\bm{\psi})}{\partial\bm{\psi}}
\end{equation*}
can suffer from high variance.
Therefore, to facilitate training via the reparameterization trick \cite{conference/iclr/kingma2014}, a continuous relaxation of the discrete variable using the Gumbel-Softmax (GS) estimator is proposed \cite{conference/iclr/jang2017,conference/iclr/maddison2017}.

Specifically, let $\bm{z}$ be a categorical variable with class probabilities $\bm{\pi}(\bm{\psi})=[\pi_1,\dots,\pi_K]\in[0,1]^K$.
The Gumbel-Max trick samples $\bm{z}$ via
\begin{equation}
  \bm{z}=\text{one-hot}\left(\argmax_{k} [g_k+\log\pi_k]\right),
\end{equation}
where $\bm{g}\triangleq (g_k)_{k=1}^K\stackrel{\text{i.i.d.}}{\sim} \text{Gumbel}(0,1)$, \emph{i.e.}, $g_k=-\log(-\log(u_k))$ with $u_k\stackrel{\text{i.i.d.}}{\sim}\text{Uniform}(0,1)$.
The GS distribution uses the softmax as a continuous and differentiable relaxation to argmax, and generate a random vector on a simplex $\bm{y}_{\tau}=[y_1,\dots,y_K]^{\intercal}\triangleq\operatorname{softmax}_{\tau}(\bm{g}+\log\bm{\pi})\in\Delta^{K-1}$, such that
\begin{equation}
  y_k = \frac{\exp\left[(g_k+\log\pi_k)/\tau\right]}{\sum_{k=1}^{K}\exp\left[(g_k+\log\pi_k)/\tau\right]},\quad k=1,\dots,K.
\end{equation}
To overcome the challenge of gradient computation with discrete stochasticity, the Straight-Through Gumbel-Softmax (ST-GS) estimator \cite{conference/iclr/jang2017,conference/iclr/paulus2021} uses the continuous relaxation, giving rise to a biased Monte-Carlo gradient estimator of the form 
\begin{equation}
  \nabla_{\text{ST-GS}}^{\text{MC}}\triangleq\frac{\partial f(\bm{z})}{\partial \bm{z}}\frac{\operatorname{d}\operatorname{softmax}_{\tau}(\bm{g}+\log\bm{\pi}(\bm{\psi}))}{\operatorname{d}\bm{\psi}},
\end{equation}
where the forward pass in $f(\cdot)$ is computed using the non-relaxed discrete samples.
To further reduce the variance in gradient estimator by Rao-Blackwellization, the Gumbel-Rao (GR) estimator \cite{conference/iclr/paulus2021} was proposed, which takes the form
\begin{equation}
  \nabla_{\text{GR}}^{\text{MC}}\triangleq \mathbb{E}\left[\nabla_{\text{ST-GS}}^{\text{MC}}\mid\bm{z}\right]
  \approx\frac{\partial f(\bm{z})}{\partial \bm{z}}\left[\frac{1}{S}\sum_{s=1}^S \frac{\operatorname{d}\operatorname{softmax}_{\tau}(\bm{G}^s(\bm{\psi}))}{\operatorname{d}\bm{\psi}}\right],
\end{equation}
where $\bm{G}^s\stackrel{\text{i.i.d.}}{\sim}\bm{g}+\log\bm{\pi}\mid\bm{z}$ for $s=1,\dots,S$.
Thus, we have the following proposition: let $\widebar{\nabla}_{\bm{\psi}}\triangleq\nicefrac{\operatorname{d}\mathbb{E}[f(\bm{z})]}{\operatorname{d}\bm{z}}$ be the true gradient we are trying to estimate, then
\begin{equation}
  \mathbb{E}\norm*{\nabla_{\text{GR}}^{\text{MC}}-\widebar{\nabla}_{\bm{\psi}}}^2\leq \mathbb{E}\norm*{\nabla_{\text{ST-GS}}^{\text{MC}}-\widebar{\nabla}_{\bm{\psi}}}^2.
\end{equation}

\section{The Diffeomorphic Demons Algorithm}
Without loss of generality, we consider pairwise registration between a fixed image $\bm{F}:\mathds{R}^d\supset\Omega_F\rightarrow\mathds{R}^D$ and a moving image $\bm{M}:\mathds{R}^d\supset\Omega_M\rightarrow\mathds{R}^D$, where $D$ is the feature dimensionality. 
Given the current transformation $\phi_s:\Omega_F\rightarrow\Omega_M$, the correspondence update (velocity) field $\bm{u}:\Omega_F\rightarrow\mathds{R}^d$ is derived by the following optimization problem
\begin{equation}
  \argmax_{\bm{u}}\frac{1}{2\abs{\Omega_F}}\left\{\norm{\bm{F}-\bm{M}\circ \phi_s\circ (\operatorname{id}+\bm{u})}_F^2 + \sigma_{\bm{\varphi}}^2\norm{\bm{u}}_F^2\right\}
\end{equation}
For a given location $\bm{\omega}\in\Omega_F$, let $\bm{\varphi^{\omega}}_s(\bm{u})\triangleq \bm{F}(\bm{{\omega}})-\bm{M}\circ \phi_s\circ(\operatorname{id}+\bm{u})(\bm{{\omega}})$, and assume the following linearization is available:
\begin{equation}
  \begin{aligned}
    \bm{\varphi^{\omega}}_s(\bm{u})&\approx \bm{\varphi^{\omega}}_s(\bm{0})+\bm{J}_{\bm{\varphi}}(\bm{\omega})\cdot\bm{u}(\bm{\omega}) \\
    &= \bm{F}(\bm{{\omega}})-\bm{M}\circ \phi_s(\bm{{\omega}})+\bm{J}_{\bm{\varphi}}(\bm{\omega})\cdot\bm{u}(\bm{\omega}),
  \end{aligned}
\end{equation}
where $\bm{J}_{\bm{\varphi}}(\bm{\omega})\in\mathds{R}^{D\times d}$ is the linearization matrix.
Substituting this linearization into the objective function gives 
\begin{equation}
  \begin{aligned}
    E_s(\bm{u}) &= \frac{1}{\abs{\Omega_F}}\sum_{\bm{\omega}\in\Omega_F} E_s^{\bm{\omega}}(\bm{u}) \\ 
    &\triangleq \frac{1}{2\abs{\Omega_F}}\sum_{\bm{\omega}\in\Omega_F}\left\{\norm{\bm{\varphi^{\omega}}_s(\bm{u})}_2^2+\sigma_{\bm{\varphi}}^2\norm{\bm{u}(\bm{\omega})}_2^2\right\} \\
    &\approx \frac{1}{2\abs{\Omega_F}}\sum_{\bm{\omega}\in\Omega_F}\norm{
    \begin{bmatrix}
      \bm{\varphi^{\omega}}_s(\bm{0}) \\
      \bm{0}
    \end{bmatrix} +
    \begin{bmatrix}
      \bm{J}_{\bm{\varphi}}(\bm{\omega}) \\
      \sigma_{\bm{\varphi}}\bm{I}_d
    \end{bmatrix} \cdot \bm{u}(\bm{\omega})
    }_2^2.
  \end{aligned}
\end{equation}  
Therefore, we only need to solve the following numerical equation at each location $\bm{\omega}$, 
\begin{equation}
  \begin{bmatrix}
    \bm{J}_{\bm{\varphi}}^{\intercal}(\bm{\omega})\:\:
    \sigma_{\bm{\varphi}}\bm{I}_d
  \end{bmatrix}
  \cdot
  \begin{bmatrix}
    \bm{J}_{\bm{\varphi}}(\bm{\omega}) \\
    \sigma_{\bm{\varphi}}\bm{I}_d
  \end{bmatrix}
  \cdot
  \bm{u}(\bm{\omega})
  =- 
  \begin{bmatrix}
    \bm{J}_{\bm{\varphi}}^{\intercal}(\bm{\omega})\:\:
    \sigma_{\bm{\varphi}}\bm{I}_d
  \end{bmatrix}
  \cdot
  \begin{bmatrix}
    \bm{\varphi^{\omega}}_s(\bm{0}) \\
    \bm{0}
  \end{bmatrix}
\end{equation}
which yields the solution
\begin{equation}
  \bm{u}^*(\bm{\omega}) = - \left[\bm{J}_{\bm{\varphi}}^{\intercal}(\bm{\omega})\bm{J}_{\bm{\varphi}}(\bm{\omega})+\sigma_{\bm{\varphi}}^2\bm{I}_d\right]^{-1}\bm{J}_{\bm{\varphi}}^{\intercal}(\bm{\omega})\bm{\varphi^{\omega}}_s(\bm{0}).
\end{equation}
Thus, by plugging $\bm{J}_{\bm{\varphi}}(\bm{\omega})=-\nabla^{\intercal}(\bm{M}\circ \phi_s)({\bm{\omega}})$ we can obtain a first-order Gauss-Newton step for the update rule.
Moreover, by plugging $\bm{J}_{\bm{\varphi}}(\bm{\omega})=-\frac{1}{2}\left[\nabla^{\intercal}(\bm{M}\circ \phi_s)+\nabla^{\intercal}F\right]({\bm{\omega}})$ we can derive a second-order approximation of the Demons force \cite{conference/miccai/vercauteren2008,journal/ni/vercauteren2009}.
Therefore, the diffeomorphic Demons algorithm will iterate over the following steps \cite{journal/ni/vercauteren2009}
\begin{itemize}
  \item Given the current transformation $s$, compute the correspondence update field $\bm{u}^*$ for minimizing $E_s(\bm{u})$;
  \item For a fluid-like regularization, let $\bm{u}\leftarrow K_{\text{fluid}}\star \bm{u}$, where $K_{\text{fluid}}$ will typically be a Gaussian kernel;
  \item Let $\phi_s\leftarrow \phi_s\circ \exp(\bm{u})$,
\end{itemize}
where the last step of vector field exponential can be implemented by the Scaling and Squaring method \cite{conference/miccai/arsigny2006}.

\section{Network Architecture}
\begin{figure*}[t]
    \centering
\includegraphics[width=\textwidth]{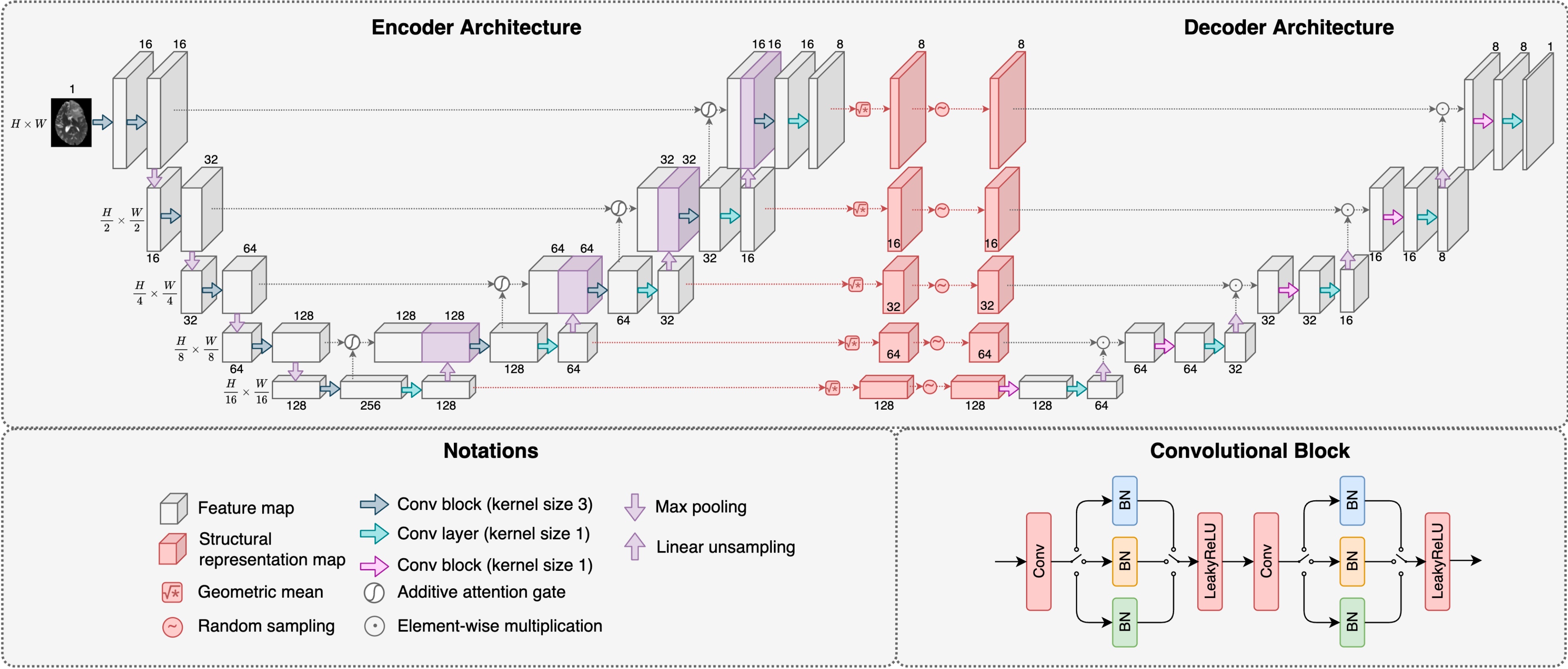}
    \caption{The architectures of the encoder and decoder of each modality. Red modules, operations, and maps are shared among all modalities. The batch normalisation (BN) layers in each convolutional block are modality-specific, i.e., for each modality, only the corresponding BN branch is used. The channel numbers are indicated around each map, and the spatial sizes of the maps at the same level are equal, indicated on the left side of the encoder.}
\label{fig:encoder_decoder_architecture}
\end{figure*}

\cref{fig:encoder_decoder_architecture} presents the network architecture of the encoder and decoder.
The encoder builds on an Attention U-Net structure \cite{conference/midl/oktay2018} to extract structural information, and the decoder is fed with the multi-level common structural representations to reconstruct the original images.
Domain-specific batch normalisation (BN) layers \cite{conference/cvpr/chang2019} are used to account for multi-modal appearance variations.
Moreover, the kernel size of the convolutions in the decoder is set to 1 to impose the spatial equivariance constraint.
Note that for simplicity we omit the registration module, which is illustrated in the main article, and we assume that the input of encoder is the warped image group during the inference steps \#2.

\section{Disentangled Representation Learning}
Disentangled representation learning aims to identify and reveal the underlying explanatory factors of variation in the observed data \cite{journal/tpami/bengio2013,journal/media/liu2022}.
Although various attempts have been made to characterise disentangled representations from the observed data through statistical independence \cite{conference/iclr/higgins2016,journal/jmlr/achille2018,conference/neurips/ridgeway2018,conference/neurips/chen2018,conference/iclr/kumar2018,conference/icml/mathieu2019,conference/aistats/esmaeili2019}, as \citet{conference/icml/locatello2019} and \citet{conference/aistats/khemakhem2020} have pointed out, without placing restrictions on the data and models or conditioning on additionally observed variables, this type of unsupervised disentanglement learning is hopelessly unidentifiable.
For image data, content-style (or shape-appearance) disentanglement has then become popular in image synthesis \cite{conference/eccv/huang2018}, segmentation \cite{journal/media/chartsias2019,journal/tmi/chartsias2021}, registration \cite{conference/ipmi/qin2019}, harmonisation \cite{conference/ipmi/zuo2021,conference/miccai/zuo2022}, and reconstruction \cite{conference/ipmi/ouyang2021}, to name a few.
To build data-driven inductive bias, they usually require paired multi-modal images with the same latent anatomy and even manual labels to guide disentanglement of anatomy and modality.

\subsection{\small Disentangled Representations \& Symmetry Groups}
To instead promote a more principled characterisation of disentangled representations, \citet{journal/arxiv/higgins2018} propose to define a vector representation as disentangled \emph{w.r.t.} a particular decomposition of a symmetry group into subgroups, if it decomposes into multiple subspaces, each of which is transformed independently by the action of a unique symmetry subgroup, while the actions of all other subgroups leave the subspace unchanged.
Formally speaking in the language of group theory, they propose the following definition
\begin{definition*}[Disentangled representation]
  Suppose that the symmetry transformations of the observational space $\mathcal{U}$ decomposes as a direct product $\mathcal{G}={G}_1\times\dots\times{G}_n$.
  Then the space of latent representation $\mathcal{Z}$ is disentangled \emph{w.r.t.} the decomposition if 
  \begin{enumerate}[a)]
    \item There is an action $\smallsquare:\mathcal{G}\times \mathcal{Z}\rightarrow \mathcal{Z}$;
    \item The map $h:\mathcal{U}\rightarrow \mathcal{Z}$ is \emph{equivariant} between the actions on $\mathcal{U}$ and $\mathcal{Z}$, i.e.,
    \begin{equation*}
      \mathfrak{g}\smallsquare h(u) = h(\mathfrak{g}\cdot u),\quad\forall\,\mathfrak{g}\in\mathcal{G},\ u\in \mathcal{U};
    \end{equation*}
    \item There is a decomposition $\mathcal{Z}=Z_1\times\dots\times Z_n$ such that each $Z_i$ is fixed by the action of all ${G}_j$, $\forall\,j\neq i$ and affected only by ${G}_i$.
  \end{enumerate}
\end{definition*}
\noindent Though the problem of learning disentangled representations is separate from what defines disentangled representation, this formal definition is indeed inspiring for us to design proper and useful inductive biases for disentangled representation learning.

In our situations, the transformation group on the observed image space $\mathcal{U}$ can be decomposed as $G=\mathcal{G}\times\mathcal{T}$, where $\mathcal{G}$ is the group of diffeomorphic transformations acting on the spatial domain while $\mathcal{T}$ is the group of ontological transformations in the anatomy domain, \emph{i.e.,}
\begin{equation*}
  \mathfrak{g}\cdot U\triangleq f(T(\bm{Z}))\circ\phi^{-1} = f(T(\bm{Z})\circ\phi^{-1}),\quad\forall\,\mathfrak{g}\in G,
\end{equation*}
where $\bm{Z}$ is the anatomy of $U$ and $f:\bm{Z}\mapsto U$ is an equivariant imaging function \emph{w.r.t.} $\forall\,\phi\in\mathcal{G}$.
The ontological transformations characterise the anatomical variation, \emph{i.e.}, whether a certain structure should change its existence, which cannot be explained by $\mathcal{G}$.
More precisely, if $\bm{Z}=(z_1,\dots,z_K)\in\{0,1\}^K$, then $T(\bm{Z})$ will flip some entries of $\bm{Z}$ from 0 to 1 or vice versa, i.e., $\mathcal{T}$ is a product group of $K$ cyclic groups of order 2, i.e., $\mathcal{T}\cong\Pi_{k=1}^K C_2$, where $C_2=\{e,g_c\}$ with $g_c^2=e$.
The identity element $e$ in the $k$-th $C_2$ means $T$ would keep the $k$-th anatomical structure unchanged, while the element $g_c$ means to make the structure appear or disappear, \emph{i.e.}, to change its existence.

Therefore, when we decompose the latent space into $\widehat{\mathcal{G}}\times\mathcal{Z}_0$ where $\widehat{\mathcal{G}}\subset\mathcal{G}$ is the diffeomorphic transformations registering $U\in\mathcal{U}$ to the common space, and $\mathcal{Z}_0$ is the common-space structural representations, then the map $h:U=f\circ\widehat{\phi}(\bm{Z}_0)\mapsto [\widehat{\phi},\bm{Z}_0]$ is disentangled \emph{w.r.t.} the group decomposition $G=\mathcal{G}\times\mathcal{T}$.
To see this, for $\mathfrak{g}=[\phi_{\mathfrak{g}},T_{\mathfrak{g}}]\in G$, we have 
\begin{equation*}
  \begin{aligned}
    h(\mathfrak{g}\cdot U) &= h\left( f(T_{\mathfrak{g}}(\bm{Z})\circ\phi_{\mathfrak{g}}^{-1}) \right) \\
    &= h\left( f(\phi_{\mathfrak{g}}\circ T_{\mathfrak{g}}\circ \widehat{\phi}(\bm{Z}_0))\right) \\ 
    &= h\left( f(\phi_{\mathfrak{g}}\circ \widehat{\phi}\circ T_{\mathfrak{g}}(\bm{Z}_0))\right) \\ 
    &= [\phi_{\mathfrak{g}}\circ\widehat{\phi},T_{\mathfrak{g}}(\bm{Z}_0)],
  \end{aligned}
\end{equation*}
where the third equation comes from the fact that elements of $\mathcal{G}$ commute with elements of $\mathcal{T}$ for product groups \cite{book/pearson/artin2011}.

Thus, if we define an action $\smallsquare$ on the latent subspaces $\widehat{\mathcal{G}}\times{\mathcal{Z}_0}$ by $\mathfrak{g}\smallsquare[\widehat{\phi},{\bm{Z}_0}]\triangleq[\phi_{\mathfrak{g}}\circ\widehat{\phi},T_{\mathfrak{g}}({\bm{Z}_0})]$, which is actually a group action by noting that $\forall\,\mathfrak{g}_1,\mathfrak{g}_2\in G$,
\begin{equation*}
  (\mathfrak{g}_1\mathfrak{g}_2)\smallsquare[\widehat{\phi},\bm{Z}_0] = [\phi_1\circ\phi_2\circ\widehat{\phi},T_1\circ T_2(\bm{Z}_0)] = \mathfrak{g}_1\smallsquare(\mathfrak{g}_2\smallsquare [\widehat{\phi},\bm{Z}_0]),
\end{equation*}
then the map $h:\mathcal{U}\rightarrow\widehat{\mathcal{G}}\times\mathcal{Z}_0$ is \emph{equivariant} between the actions on $\mathcal{U}$ and the latent subspaces $\widehat{\mathcal{G}}\times\mathcal{Z}_0$, \emph{i.e.}
\begin{equation}
  h(\mathfrak{g}\cdot U) = \mathfrak{g}\smallsquare h(U), \quad \forall\, \mathfrak{g}\in G.
\end{equation}
Besides, $\widehat{\mathcal{G}}$ is fixed by the actions of $\mathcal{T}$, and affected only by $\mathcal{G}$, and vice versa.
Thus, the proposed decomposition of the latent subspaces fulfills all three conditions for disentangled representations, justifying its usage for the task of groupwise registration.
The following definition summarises the above construction of disentangled representations for groupwise registration.

\begin{definition*}[Disentangled representation for groupwise registration]
  Let $\mathcal{G}$ be the group of diffeomorphisms acting in the spatial domain, and $\mathcal{T}$ the group of ontological transformations in the anatomy domain as defined above.
  That is, the symmetry transformations of the observed image group $U\in\mathcal{U}$ decomposes as $G=\mathcal{G}\times\mathcal{T}$.
  In addition, denote by $\widehat{\mathcal{G}}$ the space of diffeomorphisms registering the observed image group, and $\mathcal{Z}_0$ as the space of common anatomy representations.
  Then, the space of latent representations $\mathcal{Z}=\widehat{\mathcal{G}}\times\mathcal{Z}_0$ is disentangled \emph{w.r.t.} the decomposition $G=\mathcal{G}\times\mathcal{T}$, known as the disentangled representations for groupwise registration.
\end{definition*}

\subsection{Identifiability of Disentangled Representations}
Furthermore, we will show that the encoder $h$ is able to identify the true latent factors up to equivariances in the underlying mechanisms, following the work of \citet{conference/iclr/ahuja2022}.
Formally, suppose that $g:\mathcal{Z}\rightarrow\mathcal{U}$ is the true data generating process that is bijective, and $m:\mathcal{Z}\rightarrow\mathcal{Z}$ is a known state transition mechanism of the latent space.
Then, we have the following theorem:
\begin{theorem*}[Identifiability of latent representations]
  If the data generating process follows $u_{t+1}=g\circ m\circ g^{-1}(u_t)$, then the encoders $\widetilde{g}^{-1}$ that solve the observational identity $\widetilde{g}\circ m\circ \widetilde{g}^{-1}=g\circ m\circ g^{-1}$ identify the true encoder $g^{-1}$ up to equivariances of $m$, \emph{i.e.} $\widetilde{g}^{-1}\sim_{\mathcal{E}}g^{-1}$ where $\mathcal{E}=\{a\mid a\text{ is bijective}, a\circ m=m\circ a\}$.
\end{theorem*}
\begin{proof}
  The statement $\widetilde{g}^{-1}\sim_{\mathcal{E}}{g}^{-1}$ is equivalent to $\widetilde{g}\in\mathcal{D_E}\triangleq\{\widetilde{g}\mid\exists\,a\in\mathcal{E},\  \widetilde{g}=g\circ a^{-1}\}$.
  Therefore, by defining $\mathcal{D}_{\operatorname{id}}\triangleq\{\widetilde{g}\mid\widetilde{g}\circ m\circ \widetilde{g}^{-1}=g\circ m\circ g^{-1}\}$, it suffices to prove $\mathcal{D_E}=\mathcal{D}_{\operatorname{id}}$.

  \noindent \textbf{a)} To prove $\mathcal{D}_{\operatorname{id}}\subset\mathcal{D_E}$, consider a $\widetilde{g}\in\mathcal{D}_{\operatorname{id}}$.
  For any $u\in\mathcal{U}$, we have $g\circ m\circ g^{-1}(u) = \widetilde{g}\circ m\circ \widetilde{g}^{-1}(u)$.
  Thus,
  \begin{equation*}
    \begin{aligned}
      \widetilde{g}^{-1}\circ\left( g\circ m\circ g^{-1}(u)\right) &=  \widetilde{g}^{-1}\circ\left( \widetilde{g}\circ m\circ \widetilde{g}^{-1}(u)\right) \\ 
      &= m\circ\widetilde{g}^{-1}(u).
    \end{aligned}
  \end{equation*}
  Since $g$ is invertible, we can substitute $u$ in the above equation with $u=g(z)$, and obtain for any $z\in\mathcal{Z}$, 
  \begin{equation*}
    \left(\widetilde{g}^{-1}\circ g\right)\circ m\circ \left( g^{-1}\circ g(z)\right) = m\circ \left( \widetilde{g}^{-1}\circ g(z)\right),
  \end{equation*}
  or equivalently
  \begin{equation*}
    \left(\widetilde{g}^{-1}\circ g\right)\circ m(z) = m\circ \left( \widetilde{g}^{-1}\circ g\right)(z).
  \end{equation*}
  Therefore, by denoting $a=\widetilde{g}^{-1}\circ g$, we have $a\in\mathcal{E}$ and $\widetilde{g}=g\circ a^{-1}$.
  Thus, $\widetilde{g}\in\mathcal{D_E}$, and $\mathcal{D}_{\operatorname{id}}\subset\mathcal{D_E}$.

  \noindent \textbf{b)} To prove $\mathcal{D_E}\subset\mathcal{D}_{\operatorname{id}}$, consider a $\widetilde{g}\in\mathcal{D_E}$.
  By definition, one can express $\widetilde{g}=g\circ a^{-1}$ for some $a\in\mathcal{E}$.
  Then, for any $u\in\mathcal{U}$, we can derive
  \begin{equation*}
    \begin{aligned}
      \widetilde{g}\circ m\circ\widetilde{g}^{-1}(u) &= (g\circ a^{-1})\circ m\circ (g\circ a^{-1})^{-1}(u) \\
      &= (g\circ a^{-1})\circ m\circ (a\circ g^{-1})(u) \\ 
      &= g\circ a^{-1}\circ a\circ m\circ g^{-1}(u) \\
      &= g\circ m\circ g^{-1}(u).
    \end{aligned}
  \end{equation*}
  Therefore, $\widetilde{g}\in\mathcal{D}_{\operatorname{id}}$, and $\mathcal{D_E}\subset\mathcal{D}_{\operatorname{id}}$.
\end{proof}
\noindent In our situations, the data generating process is the map $g:[\phi,\bm{Z}_0]\mapsto f(\bm{Z}_0)\circ\phi^{-1}$ where $f$ is the true imaging process, and the latent transition mechanism $m$ over $\mathcal{Z}=\widehat{\mathcal{G}}\times\mathcal{Z}_0$ is from the group $G=\mathcal{G}\times\mathcal{T}$.
From the theorem, if $g$ is bijective, we conclude that the true encoder $h=g^{-1}$ is identifiable up to all equivariances of $m$ which is actually the group $G$ itself, since the subgroups $\mathcal{G}$ and $\mathcal{T}$ are both abelian.
In other words, when $\widetilde{h}^{-1}\circ m\circ\widetilde{h}=h^{-1}\circ m\circ h$, the learnt encoder $\widetilde{h}$ satisfies $\widetilde{h}=\mathfrak{g}\circ h$ for some $\mathfrak{g}=[\phi_{\mathfrak{g}},T_{\mathfrak{g}}]\in G$.
Particularly, for $U=g(\phi,\bm{Z}_0)=f(\bm{Z}_0)\circ\phi^{-1}$, we have
\begin{equation*}
  \begin{aligned}
    [\widetilde{\phi},\widetilde{\bm{Z}}_0] = \widetilde{h}(U) &= \mathfrak{g}(h(U)) \\
    &=\mathfrak{g}\diamond [\phi,\bm{Z}_0] 
    = [\phi_{\mathfrak{g}}\circ\phi,T_{\mathfrak{g}}(\bm{Z}_0)]
  \end{aligned}
\end{equation*}
Substituting the above equation back into the generating process and assuming $\widetilde{f}$ is a learned imaging process such that $U= \widetilde{f}(\widetilde{\bm{Z}}_0)\circ\widetilde{\phi}^{-1}$, we have
\begin{equation*}
  \begin{aligned}
    f(\bm{Z}_0)\circ\phi^{-1} = U &= \widetilde{f}(\widetilde{\bm{Z}}_0)\circ\widetilde{\phi}^{-1} \\
    &=\widetilde{f}(T_{\mathfrak{g}}(\bm{Z}_0))\circ\phi^{-1}\circ\phi_{\mathfrak{g}}^{-1}.
  \end{aligned}
\end{equation*}
Note that $f$, $\widetilde{f}$, and $T_{\mathfrak{g}}$ are all equivariant to $\forall\,\phi\in\mathcal{G}$.
Hence, $f=\widetilde{f}\circ T_{\mathfrak{g}}$ and $\phi_{\mathfrak{g}}=\operatorname{id}$, which can be summarized as the following corollary:
\begin{corollary*}[Identifiability for registration]
  If the learned imaging process $\widetilde{f}$ is equivariant to the diffeomorphic transformations $\phi\in\mathcal{G}$, then the true diffeomorphic registration from the observed image to the common space is identifiable via the observational identity $\widetilde{h}^{-1}\circ m\circ\widetilde{h}=h^{-1}\circ m\circ h$ for any $m\in G$, where $\widetilde{h}$ is the learned encoder, $\widetilde{h}^{-1}$ is the learned decoder, and $h^{-1}=g$ is the true data generating process.
\end{corollary*}
Note that condition $\widetilde{h}^{-1}\circ m\circ\widetilde{h}=h^{-1}\circ m\circ h$ for any $m\in G$ implies that $\widetilde{h}^{-1}\circ\widetilde{h}=\operatorname{id}$, which indicates that better reconstruction ability may improve the identifiability or accuracy of the registration.
This may provoke a tradeoff between the reconstruction fidelity and the equivariance constraint on the decoder.
We expect future work to be addressed on this interesting perspective.

\section{Additional Results and Visualisation}
\subsection{Multi-Modal \& Intersubject Groupwise Registration}
\cref{fig:violin_sup} presents the violin plots of the evaluation metrics for the compared methods on all test groups of the four datasets.

\crefrange{fig:vis_mscmr_sup}{fig:vis_oasis_sup} visualise the registration results on a test image group from each dataset using different registration methods.
The predicted backward transformation fields are also presented as colorized grid maps, whose Jacobian determinant maps are also visualized. 
For the proportions of displacement voxels with negative Jacobian, \cref{tab:ratio_neg_jac_results} presents the results for all the four datasets, in addition to the values on the OASIS dataset summarised in the manuscript.

In addition to the visualization of multi-level deformations from our model \emph{Ours-CD} on the MSCMR dataset, we visualize results from the Learn2Reg, BraTS-2021 and OASIS datasets in  \cref{fig:multilevel_disps_l2r}, \cref{fig:multilevel_disps_brats} and \cref{fig:multilevel_disps_oasis}, respectively. We observe a phenomenon similar to MSCMR: the displacements at different levels exhibit different degrees of smoothness, with larger-level $\{\phi_m^l\}_m$ placing greater emphasis on local fine-grained distortion. The deformations $\{\phi_m^1,\phi_m^2\}_m$ for the OASIS image group and the deformation $\{\phi_m^1\}_m$ for the BraTS-2021 group tend to be close to zero, probably because misalignments in these groups primarily occur among local tissues rather than affecting the entire brain structure, as indicated by the original images. In conclusion, these figures of multi-level deformations showcase the effectiveness of our hierarchical decomposition strategy under various imaging conditions.

\subsection{Scalability Test on Large-scale and Variable-size Image Groups}

\cref{fig:scalability_sup} presents the evaluation metrics versus $N_{\text{test}}$ on different datasets.

\cref{fig:mscmr_scalability} demonstrates the result of groupwise registration on a very large image group (120 images from the MS-CMRSeg dataset in total). \cref{fig:l2r_scalability}, \cref{fig:brats_scalability}, \cref{fig:oasis_scalability} further present examples of large-scale groupwise registration results for the Learn2Reg (group size of 16), BraTS (group size of 16) and OASIS (group size of 20) datasets, respectively.
One can observe that the proposed model $\emph{Ours-CD}$ could accurately register all these images, while conventional iterative or learning-based methods can hardly deal with such large group sizes, highlighting the applicability and efficiency of our methods.

\cref{fig:oasis_age_group} displays the mean image of each registered age group in the OASIS dataset using different groupwise registration methods.
One can observe that the initial mean images of the unregistered images are fuzzy around the cortical surface. 
However, after registration using the iterative methods (APE/CTE/$\mathcal{X}$-CoReg), the mean images become sharper, and the proposed models (\emph{Ours-CN}/\emph{Ours-CD}) could produce even sharper mean images, especially for relatively large image groups in the regions marked by the red circles.

\subsection{Model Interpretability}

\cref{fig:features_brats_sup} presents an example of the learnt structural representation maps for the BraTS data extracted from the models \emph{Ours-CN} and \emph{Ours-CD}, respectively.
Notably, one can find that complementary regions of the brain are reflected on these representation maps, which are similar across the modalities.
Besides, note that these structural representations corresponded spatially to the original images, and thereby are independent of the registration variables.
Likewise, \cref{fig:features_mscmr} presents an example of the learnt structural representation maps for the MS-CMRSeg data from the models \emph{Ours-CN} and \emph{Ours-CD}, respectively.
One can observe that the representations extracted from \emph{Ours-CN} may be more semantic and high-level, while those from \emph{Ours-CD} could be more detailed and local.

\subsection{Multi-Modal Counterfactual Reconstructions}

\revb{}{We provide additional visualisations to show that our method can capture common anatomy. \cref{fig:mscmr_counterfactual} presents the counterfactual reconstruction results of an MS-CMRSeg image group by Ours-CD. In particular, the fourth row shows different channels of the latent structure $\boldsymbol{z}$, and the three images in the last column are the reconstruction results from the decoder based on $\boldsymbol{z}$. By setting a channel of $\boldsymbol{z}$ to zeros, \emph{i.e.}, imposing an ontological transformation by $do(\boldsymbol{z}_k=0)$, we can obtain the corresponding counterfactual reconstructions from the decoder, which are shown in the first three rows of the figure. 
In addition, the differences between the reconstructions without and with ontological transformations (\emph{a.k.a.} the direct effect) are shown in the last three rows.}{R2.4} 

\revb{}{In our model design, $\boldsymbol{z}$ is supposed to capture the common anatomy in the latent space, and the decoder aims to reconstruct from it the registered images.  One can first observe from the fourth row that $\boldsymbol{z}$ captures the structures of this image group mainly in two channels, where one channel captures the background and the internal region of the left and right ventricles, and the other captures the boundary of the ventricle, as well as the myocardium. This indicates that the model learnt to highlight focus areas that align with relevant anatomical structures.}{} 

\revb{}{Moreover, counterfactual reconstructions based on ontological transformations illustrate how each channel of $\boldsymbol{z}$ affects the reconstruction results: When a channel is set to zeros, the affected regions of the three reconstructed modalities are nearly the same as the foreground region of that channel (note that the effect of $do(\boldsymbol{z}_4=\boldsymbol{0})$ is small for the third modality, just because the ventricles naturally appear dark in this modality). In other words, a channel of $\boldsymbol{z}$ exclusively corresponds to a specific anatomy (e.g., internal area of the ventricles) shared by all three image modalities, which demonstrates that the proposed method can capture the common anatomy.}{} 

\subsection{Comparisons of Time Costs of Baseline and Proposed Methods}

\revb{}{Information about the time costs of different methods for the MS-CMRSeg dataset is presented in \cref{tab:mscmr_time}.}{R2.5}

\subsubsection{Time Costs for Achieving the Best Performance}

\revb{}{Iterative methods APE, CTE and $\mathcal{X}$-CoReg do not need training, so their time values correspond to the convergence on the test set, which typically took hours on a GPU, much slower than the test time costs (seconds to several minutes) of learning-based methods such as the proposed one on the same GPU.}{}

\revb{}{Besides, the training times of different variants of our model are comparable to the learning-based baselines. These deep learning methods took 1-3 days. Our four variants are all faster than the baseline method APE-Att; the time costs of Ours-CD and Ours-CN are close to the best baseline CTE-Att.}{} 

\subsubsection{Time Costs for Achieving the Performance of the Best Baseline}

\revb{}{The performance of both Ours-CN and Ours-CD surpass the best baseline CTE-Att. One can observe that it took 1 day for CTE-Att to achieve its best performance, while ours took only 5 hours to achieve the same performance. This indicates that our method requires significantly less training time to reach the same level of performance as the best baseline method.}{}

\subsubsection{Time Costs for Achieving 95\% of the Best Performance}

\revb{}{Our methods (expect for Ours-PN) took just 2-3 hours to reach 95\% of their own best performance. For example, compared to the time to achieve 100\%, Ours-CN (Ours-CD) only took 7\% (9\%) of the training time to achieve 95\% performance. This indicates that the convergence speed of the proposed method is relatively fast.}{}

\revb{}{In conclusion, the comparisons of time costs demonstrate that
\begin{itemize}
    \item Our method achieved optimal performance with a training time comparable to that of baseline methods.
    \item When attaining the same performance (optimal value by the baseline methods), our method is much faster. 
    \item Our method is highly efficient, achieving near-optimal results in a short period of time.
\end{itemize}}{}

\subsection{Comparisons of Time Costs of Proposed Method Integrated With State-of-the-Art Registration Backbones}

\subsubsection{Training Time}

In \cref{tab:training_time} we report the training time required to reach (i) the best Dice score (DSC) of each method, (ii) 95\% of the best DSC, and (iii) the best DSC achieved by the strongest baseline method (CTE-Att).

The results indicate that our model maintains competitive training efficiency when integrated with state-of-the-art registration backbones. In particular, integration with TransMorph leads to significantly shorter training time for reaching either 95\% of the best DSC or the best baseline DSC, highlighting its superior convergence speed. The total training time is also comparable to our original Att-UNet version, while achieving better registration performance. Similar trends are observed when using ModeT as the backbone.

The only substantial increase in training time occurs with PIViT. We would like to note that this overhead is not inherent to our model. During experiments with PIViT, we observed training instability, which was mitigated by reducing the learning rate from $10^{-3}$ to $10^{-4}$. While this change stabilized convergence, it also led to a longer training process. Notably, such instability was not encountered with TransMorph or ModeT, suggesting that the issue is more likely attributable to PIViT’s internal training dynamics rather than any limitation in our integration mechanism.

\subsubsection{Inference Time}

In \cref{tab:inference_time}, we compare the inference time of the original model and its integrated versions. It can be observed that our model achieves very fast inference speed, requiring only about 0.04 seconds per image group from the MS-CMRSeg dataset. Remarkably, integrating our model with state-of-the-art registration backbones does not introduce any inference time overhead. This demonstrates the efficiency and practicality of our framework for real-time medical image processing. Furthermore, the consistent inference speed across different integration settings highlights the modularity of our method, making it readily adaptable to various registration pipelines without compromising performance.

Overall, these results demonstrate that our framework can be integrated with various state-of-the-art registration backbones without introducing substantial training or inference time overhead. This confirms that our method is both flexible and computationally efficient in practice.

\subsection{Ablation Study of Reconstruction Fidelity}

We performed ablation experiments using Ours-CD on the MS-CMRSeg dataset to quantify the trade-off of reconstruction fidelity and registration performance. The results are shown in \cref{tab:recon_ablation}.

First, we replaced the decoder kernel size from 1 to 3. The results showed that kernel size 1 achieved slightly higher average DSC and ASSD compared to kernel size 3, confirming that the minor loss in reconstruction fidelity does not compromise registration performance. Considering kernel size 3 may introduce theoretical issues as discussed earlier, choosing kernel size 1 is thus well justified to ensure theoretical soundness without compromising performance.

Second, we conducted experiments where we kept the decoder kernel size at 1, but progressively reduced the weight of the reconstruction loss.
We observed that as the reconstruction loss weight decreased, the registration performance also declined.
This confirms that while photorealistic reconstruction is not the goal, the reconstruction loss provides a critical learning signal that supports accurate registration.

  In conclusion, the trade-off is real, but it is one we make deliberately and advantageously. By sacrificing a small degree of reconstruction fidelity, we gain a significant improvement in the identifiability and accuracy of the geometric transformations. 
  This is an example of how imposing a strong, domain-appropriate inductive bias can guide a learning system to a more robust and accurate solution for its primary task.

  \subsection{Ablation Study of Hierarchical Level $L$}

We conducted a parameter study for $L$ while keeping all other hyperparameters fixed. Specifically, we trained Ours-CD with $L = 3, 4, 5, 6$ on the MS-CMRSeg dataset, and report the corresponding average Dice scores (DSC) on the test set:

\begin{itemize}
\item $L = 3$: 0.782
\item $L = 4$: 0.802
\item $L = 5$: \textbf{0.836}  
\item $L = 6$: 0.801
\end{itemize}

These results demonstrate that our default choice of $L = 5$ yields the best performance, with both smaller and larger $L$ values leading to a noticeable drop.

This trend can be attributed to several interrelated factors:

\begin{enumerate}
\item \textit{Architectural Alignment with Hierarchical Inference.} In our current implementation, the number of levels $L$ is tightly coupled with the depth of the network (based on an Attention U-Net). A smaller $L$ results in a shallower network with fewer downsampling stages, reducing its representational capacity and limiting its ability to model complex deformations—leading to underfitting. Conversely, increasing $L$ adds deeper layers and coarser levels, which do not always translate to better performance, because larger $L$ values introduce additional layers and parameters, increasing the risk of overfitting and making optimization more challenging. 

\item \textit{Overly Coarse Deformations at Large $L$.} At $L = 6$, the coarsest feature map operates at a spatial resolution of only $5 \times 5$ (for a $160 \times 160$ input). The velocity field at this scale is tasked with modeling extremely global deformations, which are not always necessary and may introduce instability or optimization noise. In practice, such coarse scales often fail to yield meaningful gradient signals.

\item \textit{Diminishing Returns Beyond Optimal Hierarchical Depth.} The hierarchical decomposition is most effective when each scale contributes interpretable sub-fields corresponding to different spatial frequencies. When $L$ becomes too large, the added coarse levels may no longer correspond to meaningful anatomical structures, reducing the efficacy of the decomposition and potentially disrupting optimization.
\end{enumerate}

We would also like to emphasize that these limitations are not inherent to the proposed framework itself, but rather a by-product of our current architectural implementation. In principle, the model design allows for more flexible architectures that decouple $L$ from network depth. For example:

\begin{itemize}
\item One could avoid strict $2 \times$ downsampling at each level, maintaining higher spatial resolutions at deeper levels;
\item One could share network parameters across levels to avoid increasing the parameter count with $L$.
\end{itemize}

We believe that these modifications could allow for a further scaling of $L$ without the drawbacks observed here, and we plan to explore architectural variants that better adapt to varying $L$ in future work.

\section{Balancing Weights of the Loss Terms}

The balancing weights of the loss terms for different training datasets are presented in \cref{tab:balance_weights}.
We also added a local normalised cross-correlation (LNCC) term \cite{journal/media/avants2008} to improve reconstruction on brain images, with its weight marked by values in parentheses.

\begin{table}[h]
  \scriptsize
  \centering
  \caption{This table presents the respective weights of the reconstruction loss, intrinsic structural distance and registration regularization for different experimental setups.}
  \begin{tabular}{C{1.05cm}|C{1.6cm}C{1.25cm}C{1.2cm}C{1.65cm}}
    \toprule
    Method & BraTS & MS-CMR & Learn2Reg & OASIS  \\
    \midrule
    Ours-PN & 120 (60), 75, 15 & 120, 160, 10 & N/A & N/A  \\ 
    Ours-CN & 120 (60), 75, 20 & 120, 160, 10 & 120, 80, 4 & 120 (20), 40, 4  \\ 
    Ours-PD & 120 (60), 120, 8 & 120, 160, 10 & N/A & N/A  \\ 
    Ours-CD & 120 (60), 60, 8 & 120, 160, 10 & 120, 80, 1 & 120 (20), 20, 0.1  \\ 
    \bottomrule
  \end{tabular}
  \label{tab:balance_weights}
\end{table}

\begin{figure*}[t]
  \centering
  \includegraphics[width=\linewidth]{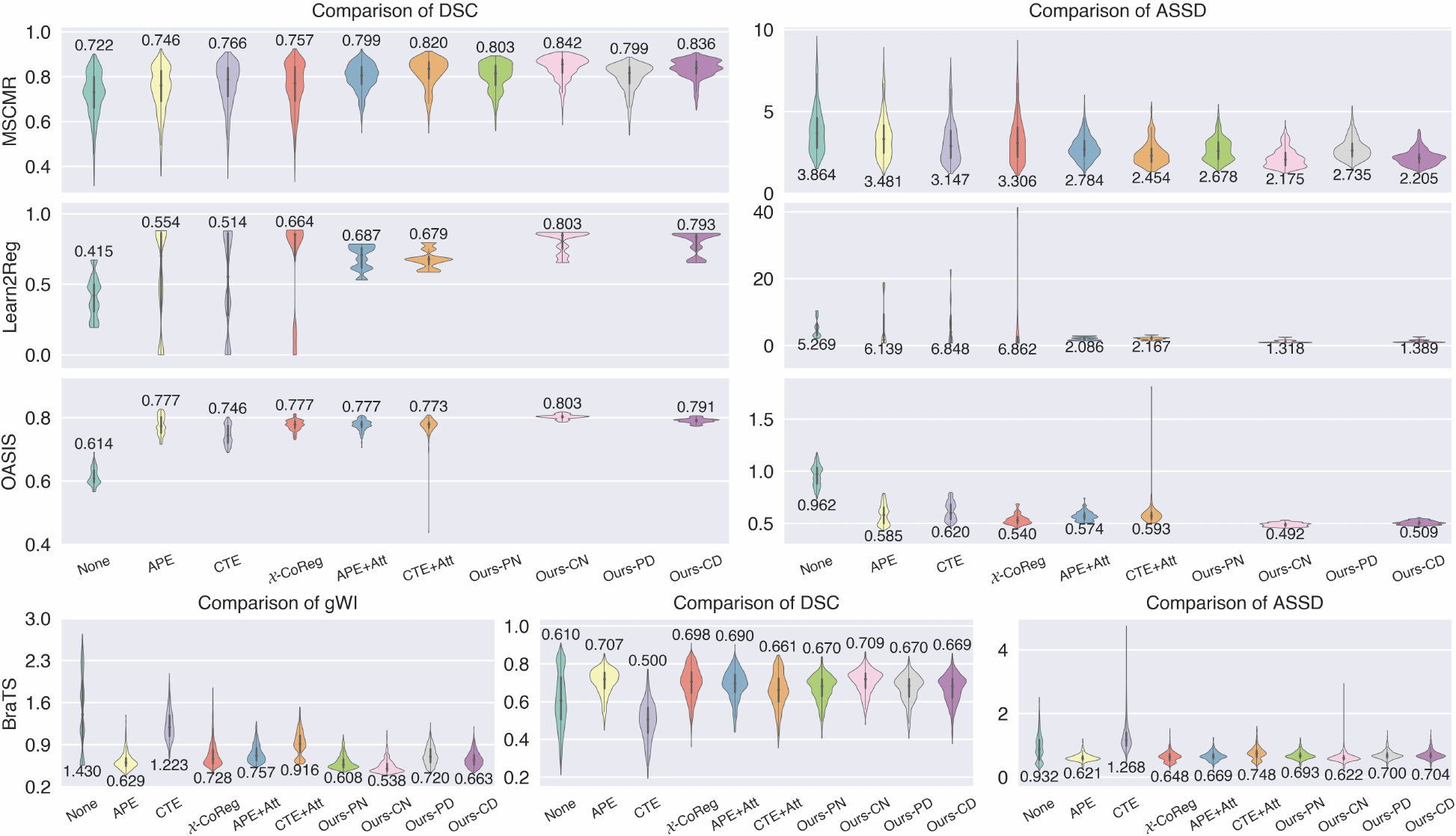}
  \caption{Quantitative evaluation metrics of the compared methods on the test groups of the four datasets. The mean values from each method are indicated.}
  \label{fig:violin_sup}
\end{figure*}

\begin{figure*}[t]
  \centering
  \includegraphics[width=\textwidth]{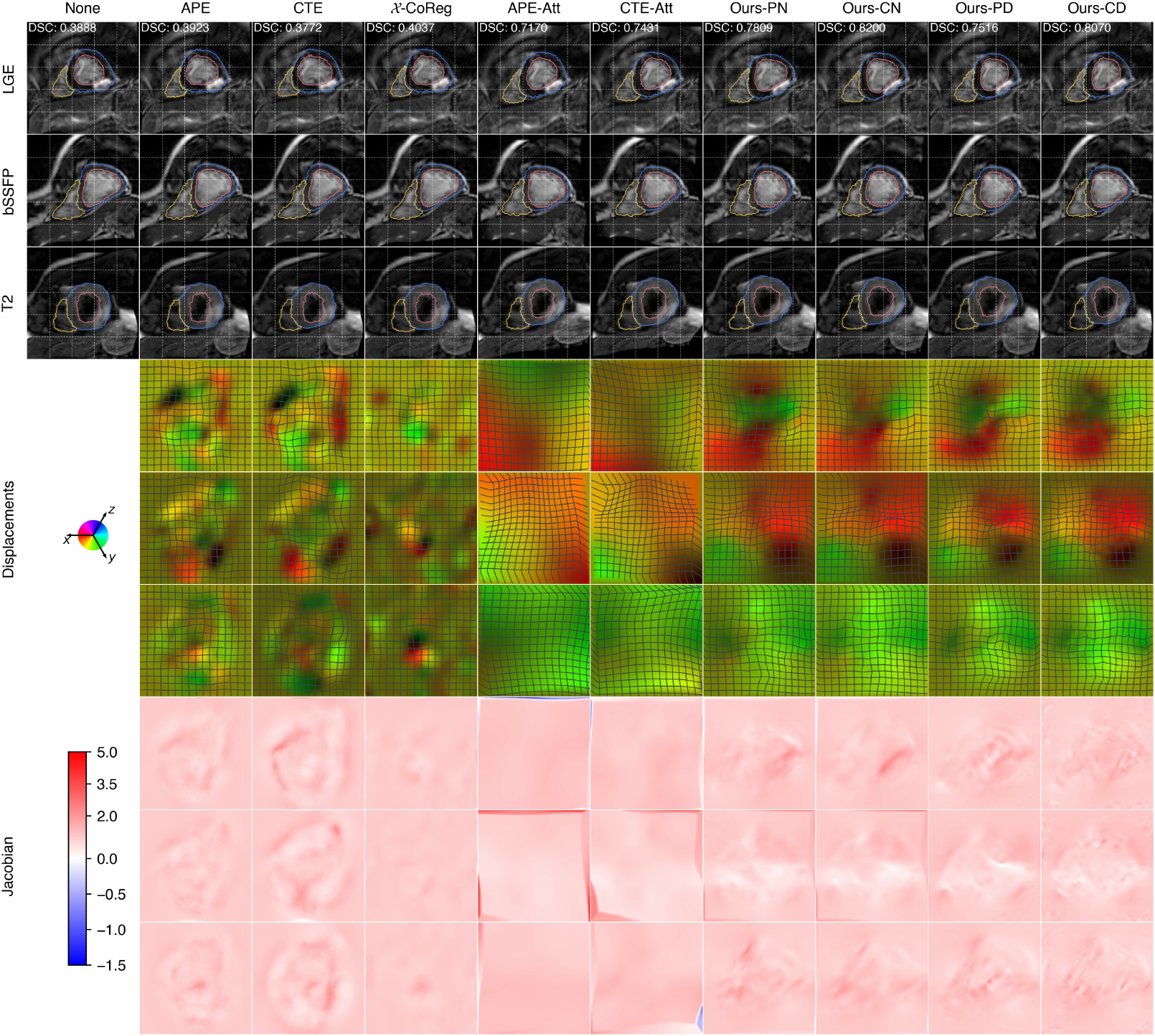}
  \caption{Results of an image group from the MS-CMRSeg dataset. The mean DSCs of all foreground classes in this group are shown for each method.}
  \label{fig:vis_mscmr_sup}
\end{figure*}

\begin{figure*}[t]
  \centering
  \includegraphics[width=\textwidth]{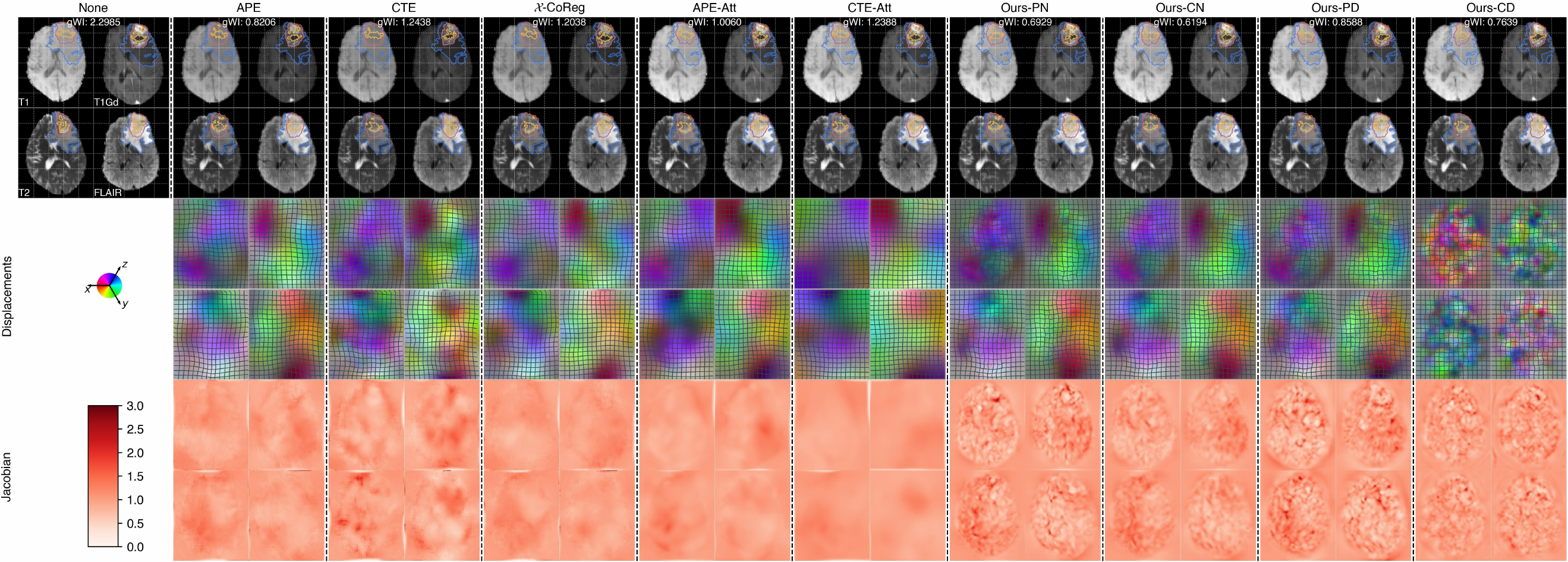}
  \caption{Results of an image group from the BraTS-2021 dataset. The mean gWIs on this group are shown for each method.}
  \label{fig:vis_brats_sup}
\end{figure*}

\begin{figure*}[t]
  \centering
\includegraphics[width=\textwidth]{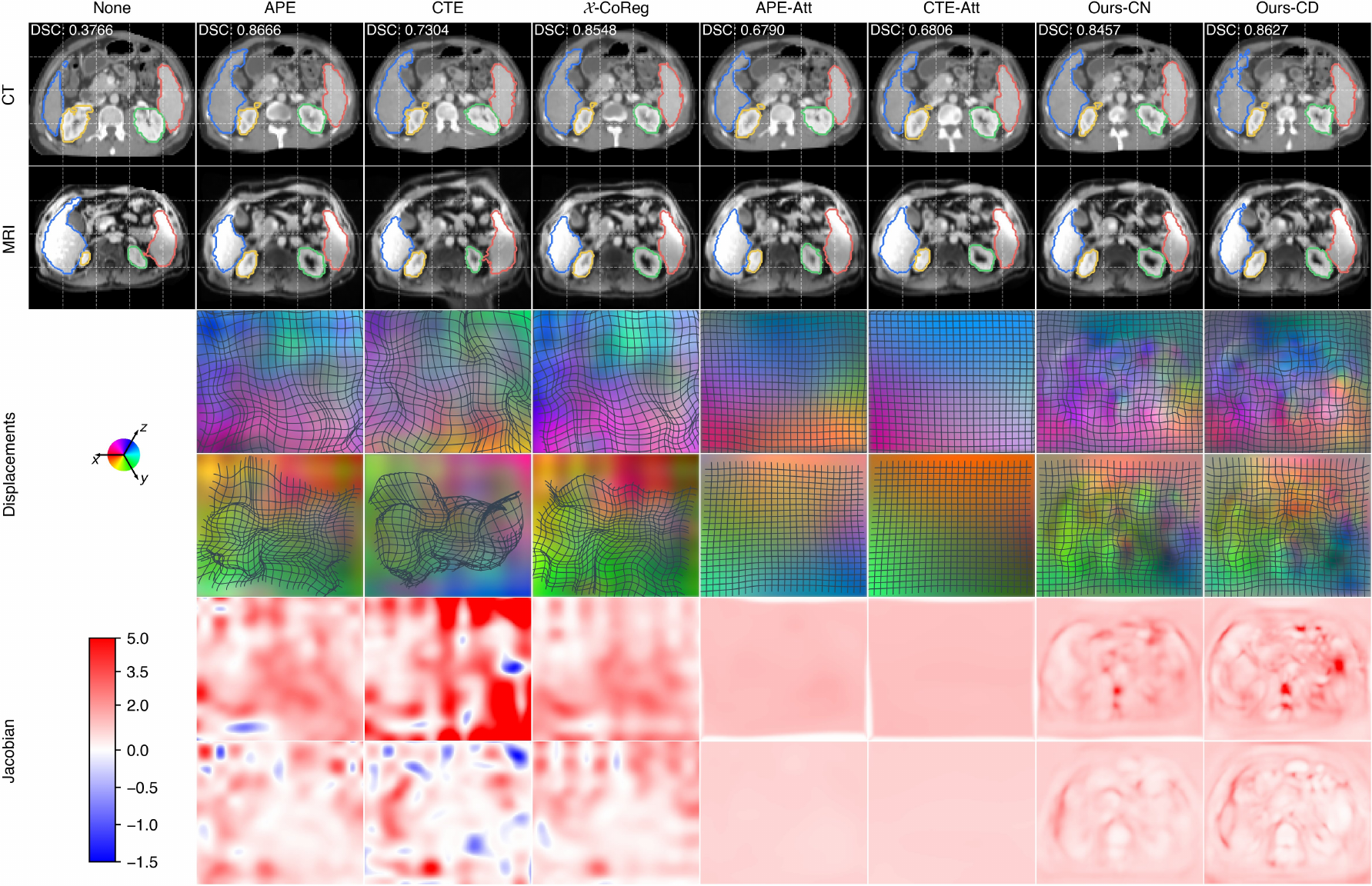}
  \caption{Results of an image group from the Learn2Reg dataset. The mean DSCs of all foreground classes in this group are shown for each method.}
  \label{fig:vis_learn2reg_sup}
\end{figure*}

\begin{figure*}[t]
  \centering
  \includegraphics[width=\textwidth]{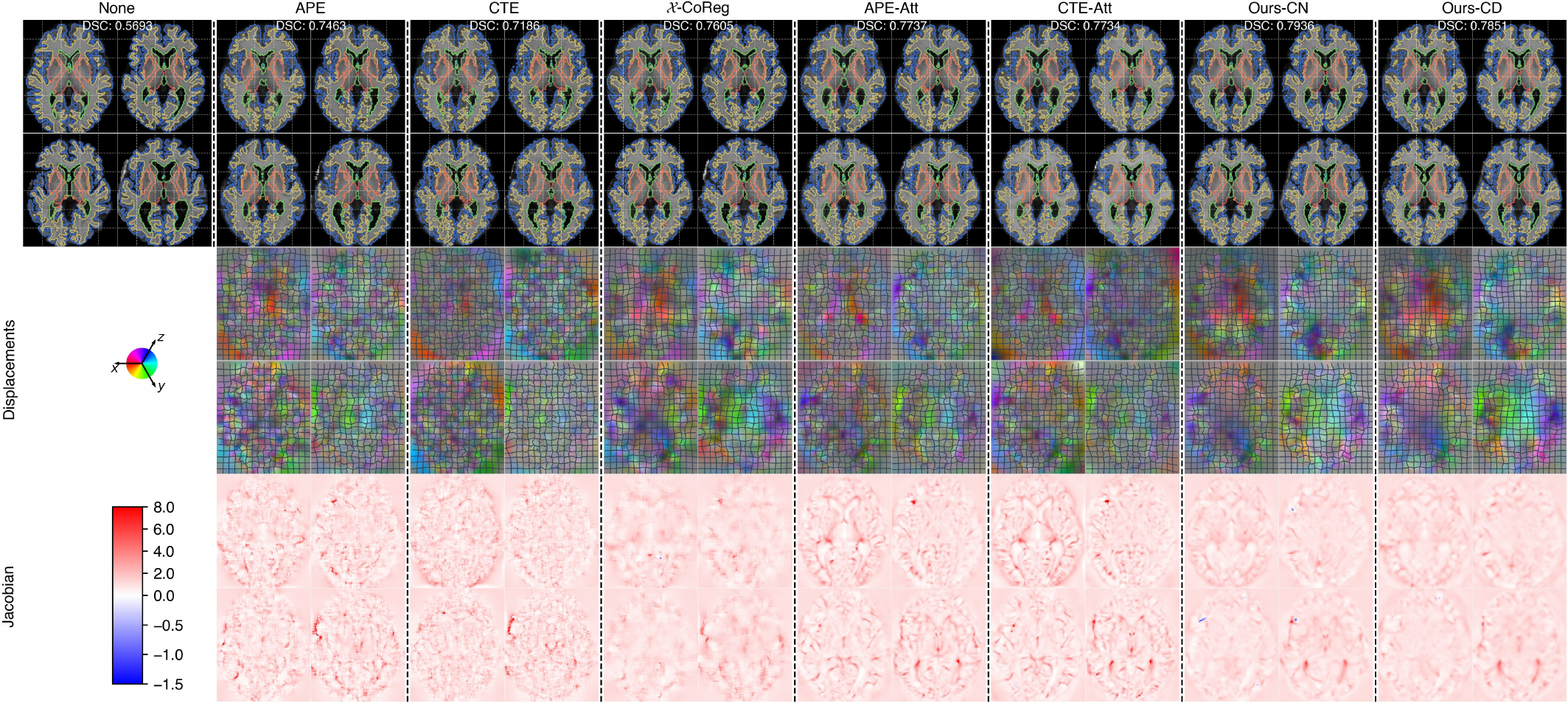}
  \caption{Results of an image group from the OASIS dataset. The mean DSCs of all foreground classes in this group are shown for each method.}
  \label{fig:vis_oasis_sup}
\end{figure*}

\begin{table*}[t]
  \centering
  \caption{The proportions (in \%) of voxels with negative Jacobian determinants in the displacements predicted by different methods on the MS-CMRSeg, Learn2Reg and BraTS-2021 datasets. The values were first calculated for the foreground region of each registered image and then averaged over all images among all test groups.}
  \begin{tabular}{C{2.5cm}|C{3cm}C{3cm}C{3cm}C{3cm}}
    \toprule
    Method & MS-CMRSeg & Learn2Reg &  BraTS-2021 & OASIS  \\
    \midrule
    APE \cite{journal/tpami/wachinger2012} & 
    $0.0002\pm 0.0010$ & $6.945\pm 3.973$  & $0.0002\pm0.0009$ & $0.5297\pm0.0814$\\
    CTE \cite{journal/media/polfliet2018}  & 
    $0.0143\pm 0.0392$ & $21.78\pm 5.445$  & $0.0022\pm 0.0065$ &$0.2291\pm0.0447$  \\
    $\mathcal{X}$-CoReg \cite{journal/tpami/luo2022} & 
    $0.0020\pm 0.0001$ & $10.91\pm 19.93$  & $0.0001\pm 0.0007$ &$0.1330\pm 0.0404$\\
    \hdashline\noalign{\vskip 0.5ex}
    APE-Att & 
    $0.0000\pm0.0000$ & $0.0000\pm0.0000$  & $0.0000\pm0.0000$ & $0.0479\pm0.0130$\\
    CTE-Att & 
    $0.0000\pm0.0000$ & $0.0000\pm0.0000$  & $0.0000\pm0.0000$ & $0.0552\pm 0.0154$\\
    \hdashline\noalign{\vskip 0.5ex}
    Ours-PN & $0.0000\pm0.0000$ & N/A & $0.0000\pm 0.0003$ & N/A\\ 
    Ours-CN & $0.0000\pm0.0000$ & $0.0000\pm 0.0001$ & $0.0000\pm0.0000$ & $0.1746\pm 0.0305$ \\ 
    Ours-PD & $0.0000\pm0.0000$ & N/A  & $0.0000\pm 0.0000$ & N/A \\ 
    Ours-CD & $0.0000\pm0.0000$ & $0.0025\pm 0.0041$  & $0.0000\pm 0.0000$ & $0.0066\pm 0.0029$\\ 
    \bottomrule
  \end{tabular}
  \label{tab:ratio_neg_jac_results}
\end{table*}

\begin{figure*}[t]
  \centering
\includegraphics[width=\linewidth]{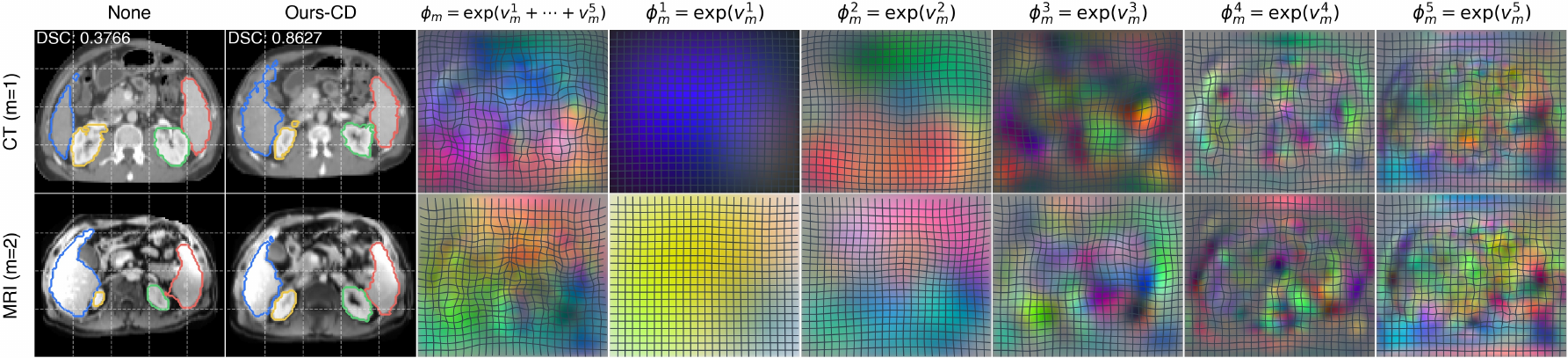}
  \caption{Multi-level deformations from our model \emph{Ours-CD} on an image group from Learn2Reg.}
\label{fig:multilevel_disps_l2r}
\end{figure*}

\begin{figure*}[t]
  \centering
\includegraphics[width=\linewidth]{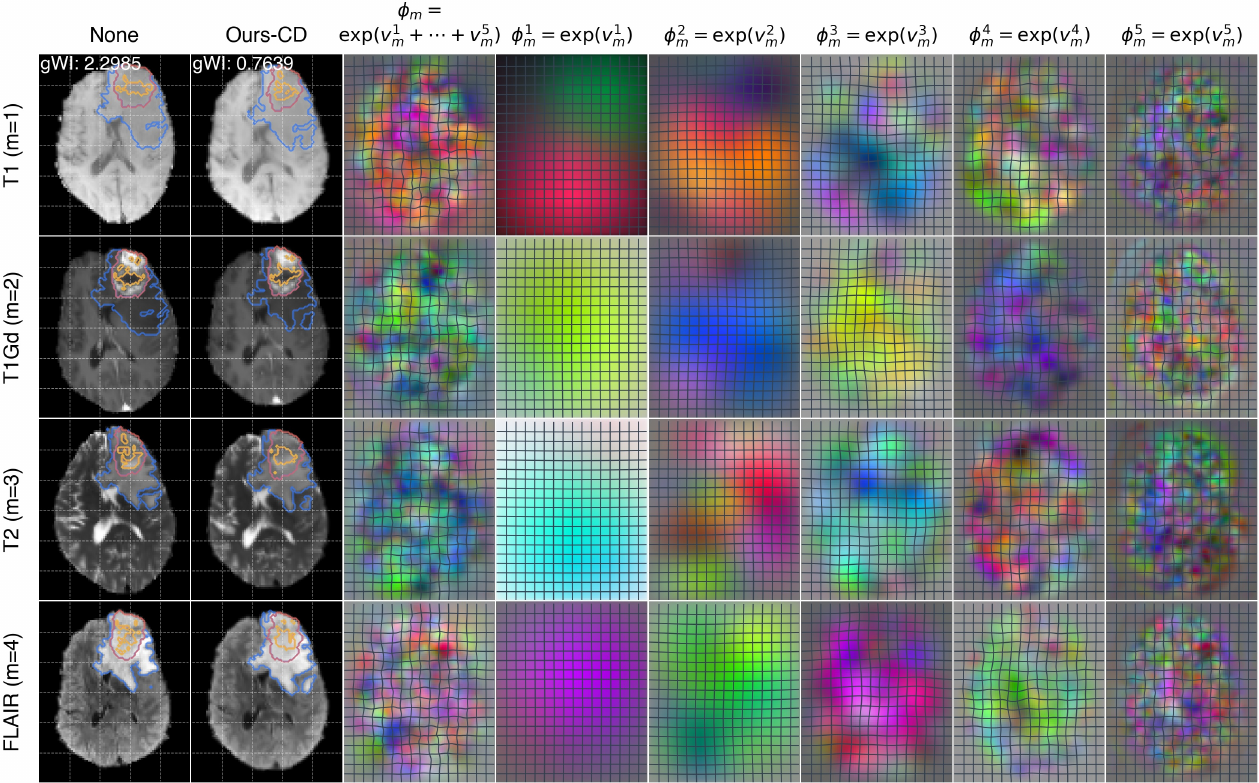}
  \caption{Multi-level deformations from our model \emph{Ours-CD} on an image group from BraTS-2021.}
\label{fig:multilevel_disps_brats}
\end{figure*}

\begin{figure*}[t]
  \centering
\includegraphics[width=\linewidth]{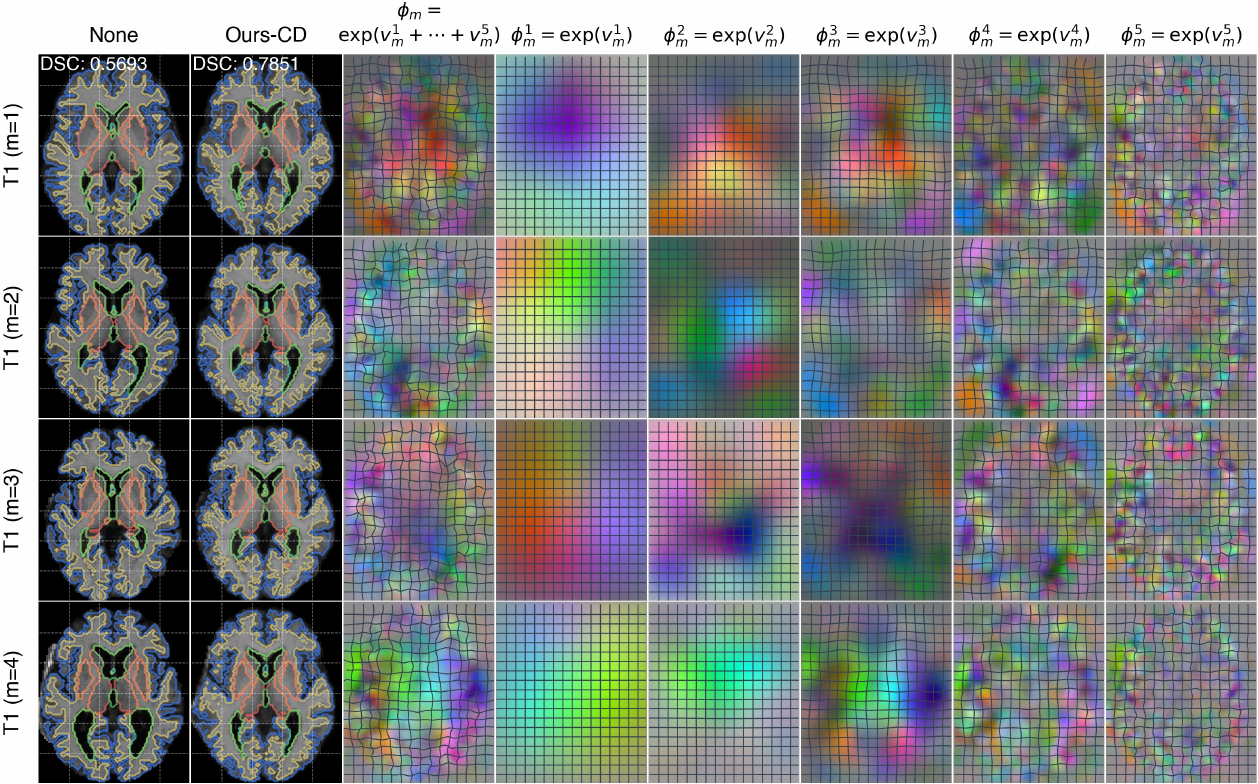}
  \caption{Multi-level deformations from our model \emph{Ours-CD} on an image group from OASIS.}
\label{fig:multilevel_disps_oasis}
\end{figure*}

\begin{figure*}[t]
  \centering
  \includegraphics[width=\linewidth]{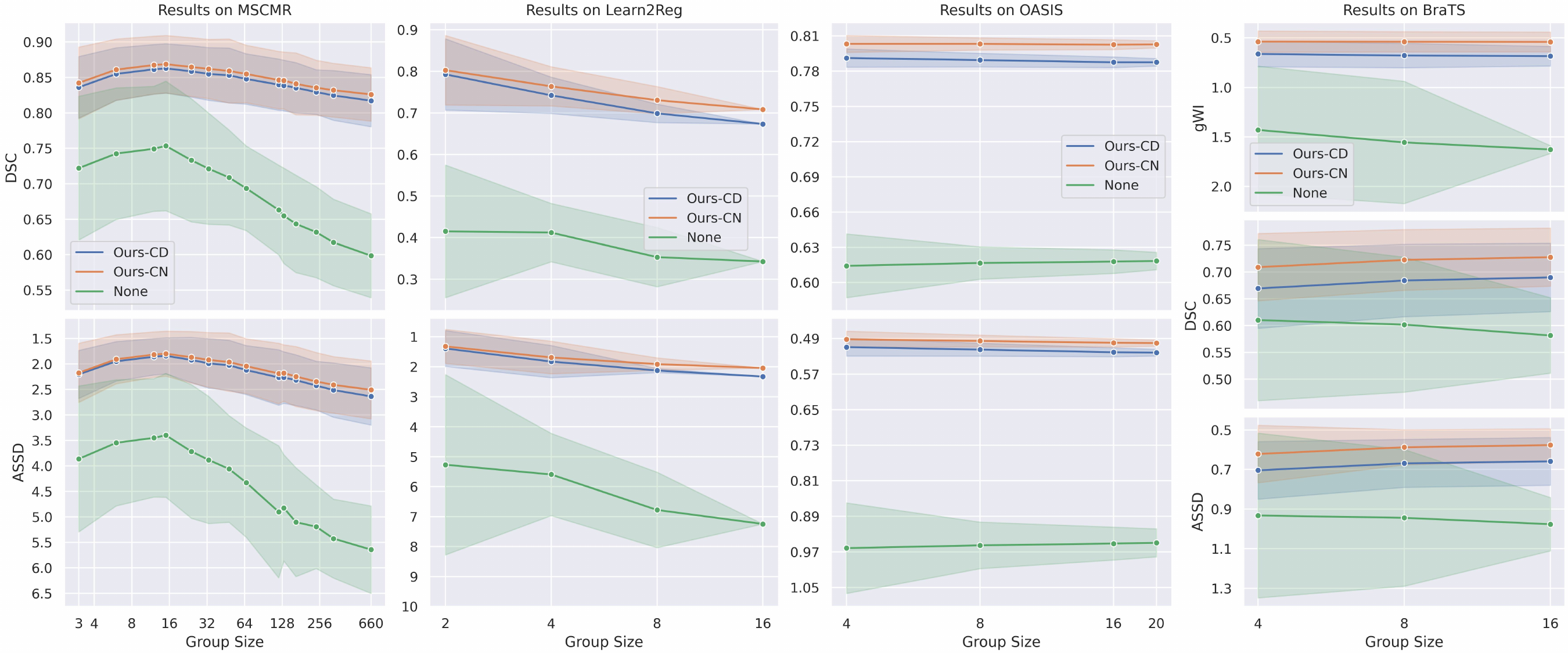}
  \caption{Evaluation metrics (mean values with one standard deviation bands) of registration results on image groups with different sizes.}
  \label{fig:scalability_sup}
\end{figure*}

\begin{figure*}
    \centering
    \includegraphics[width=\textwidth]{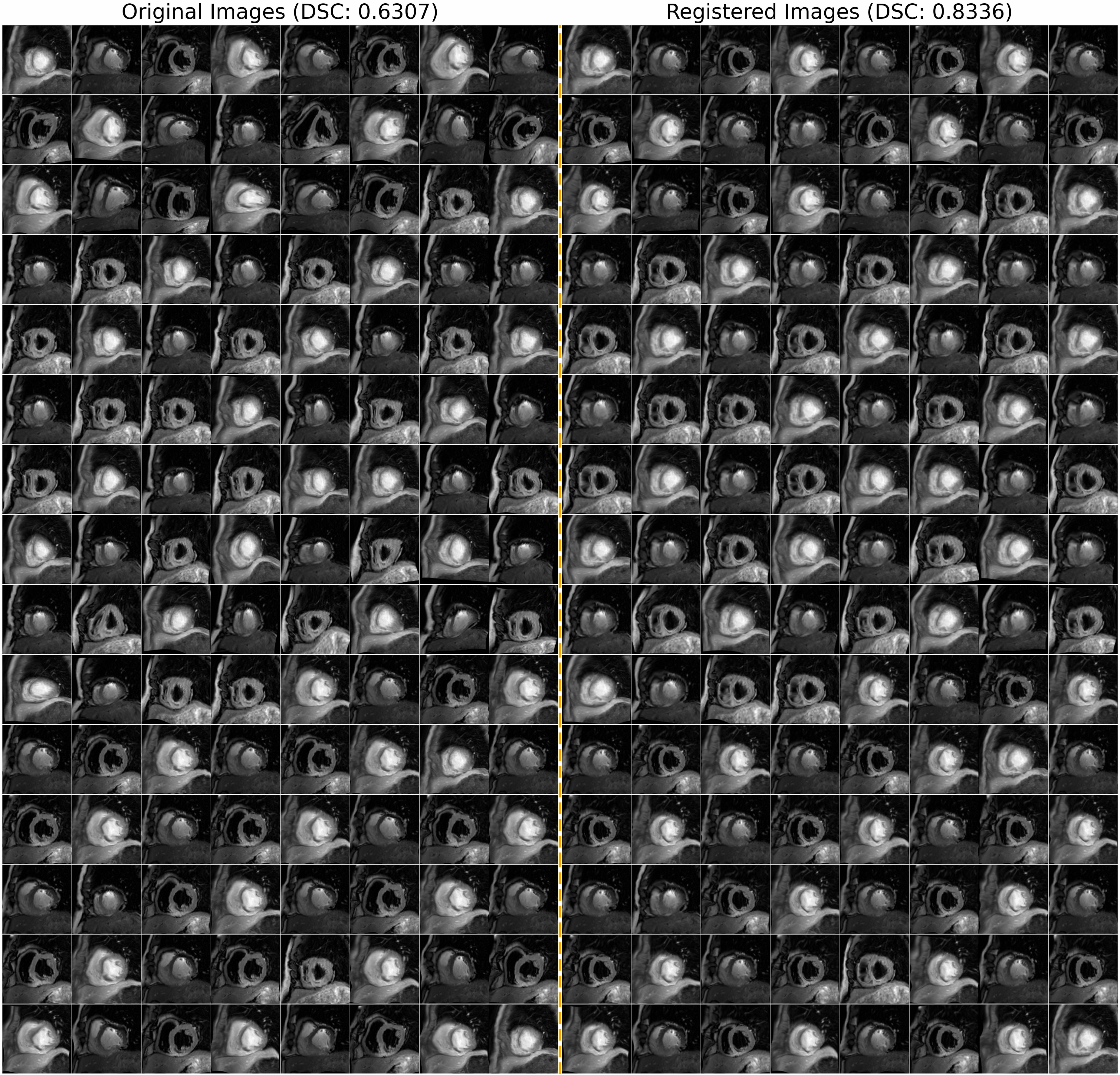}
    \caption{An example of co-registration by our model (Ours-CD) for a group of 120 images from the MS-CMRSeg dataset.}
    \label{fig:mscmr_scalability}
\end{figure*}

\begin{figure*}
    \centering
    \includegraphics[width=\textwidth]{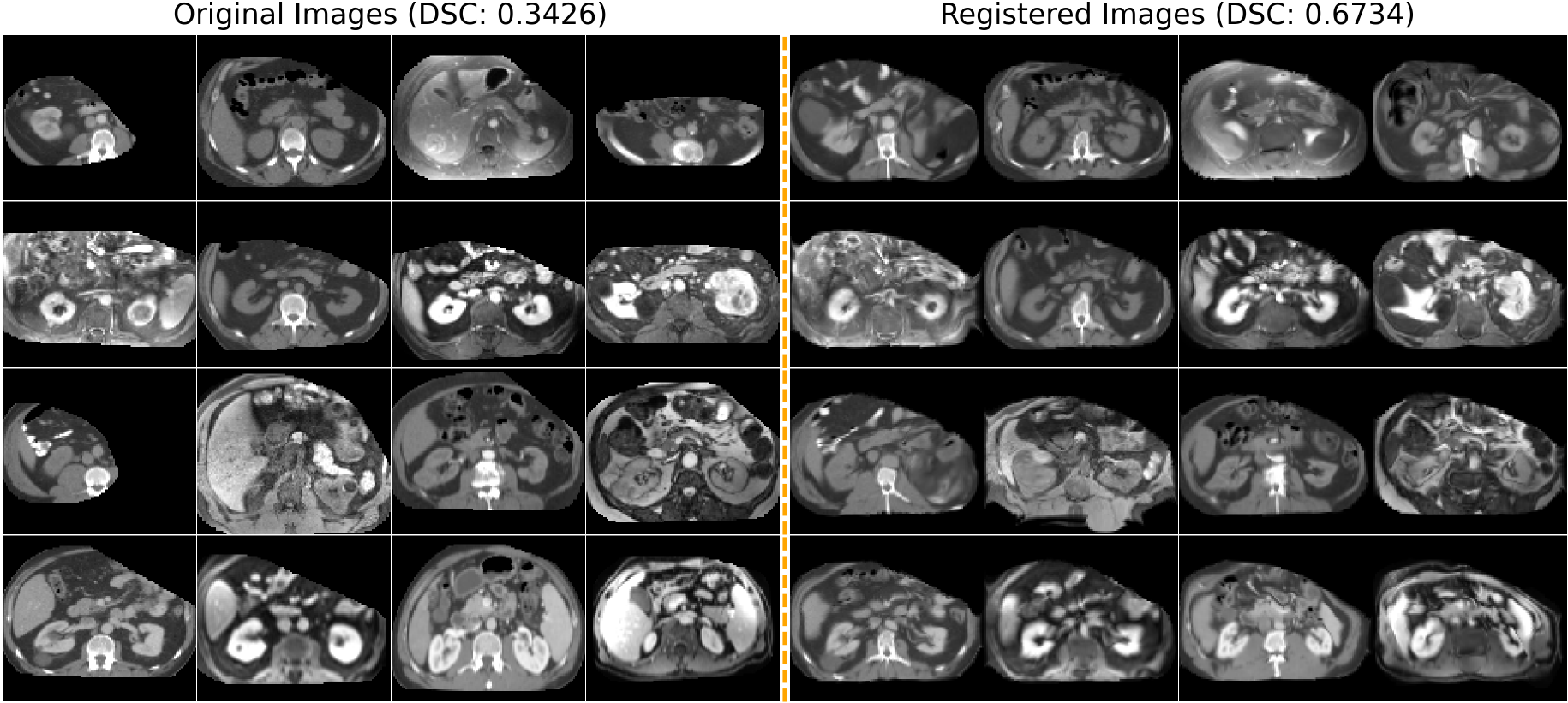}
    \caption{An example of co-registration by our model (Ours-CD) for the group of all 16 3D test images of the Learn2Reg dataset.}
    \label{fig:l2r_scalability}
\end{figure*}

\begin{figure*}
    \centering
    \includegraphics[width=\textwidth]{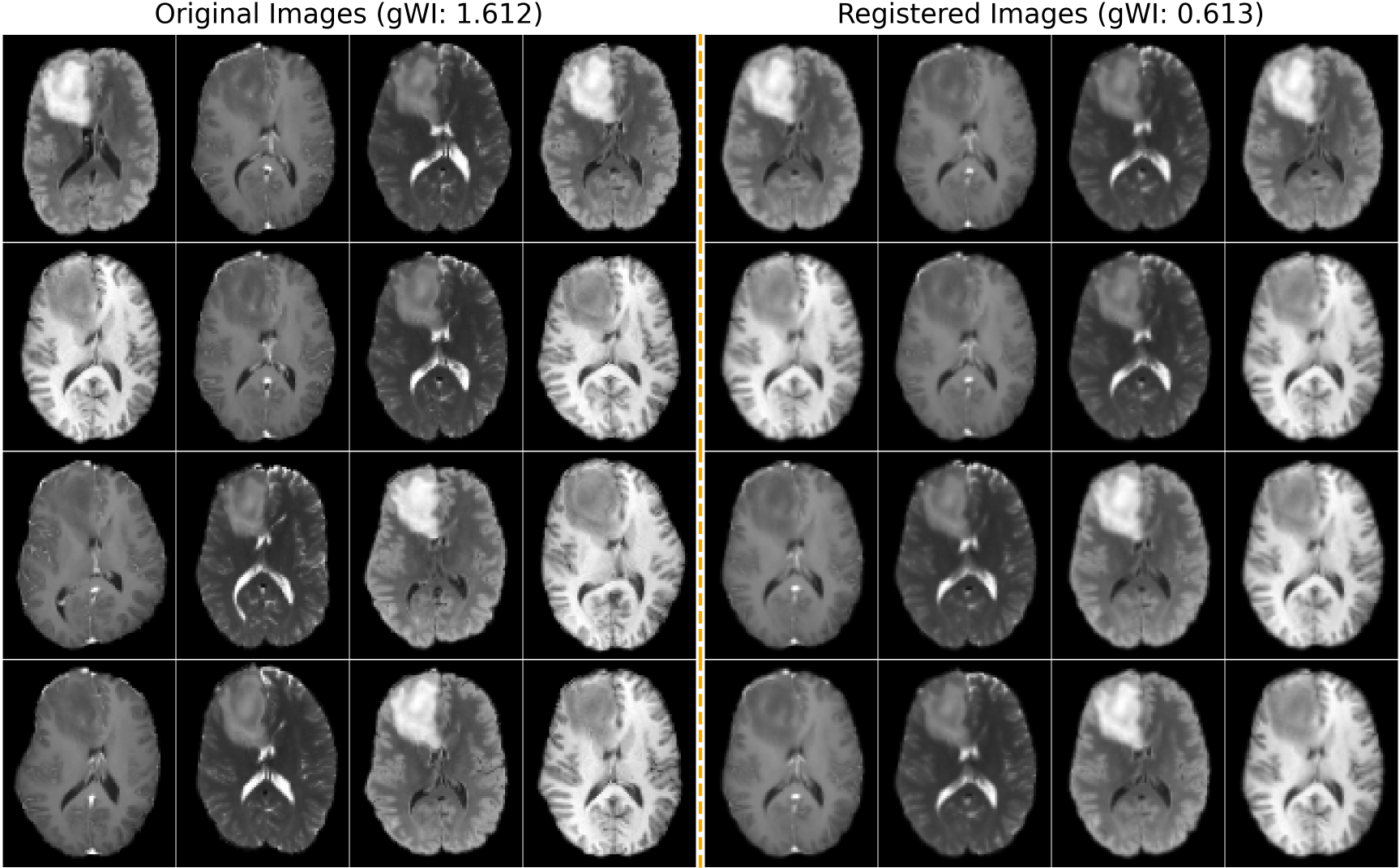}
    \caption{An example of co-registration by our model (Ours-CD) for a group of 16 3D images from the BraTS dataset.}
    \label{fig:brats_scalability}
\end{figure*}

\begin{figure*}
    \centering
    \includegraphics[width=\textwidth]{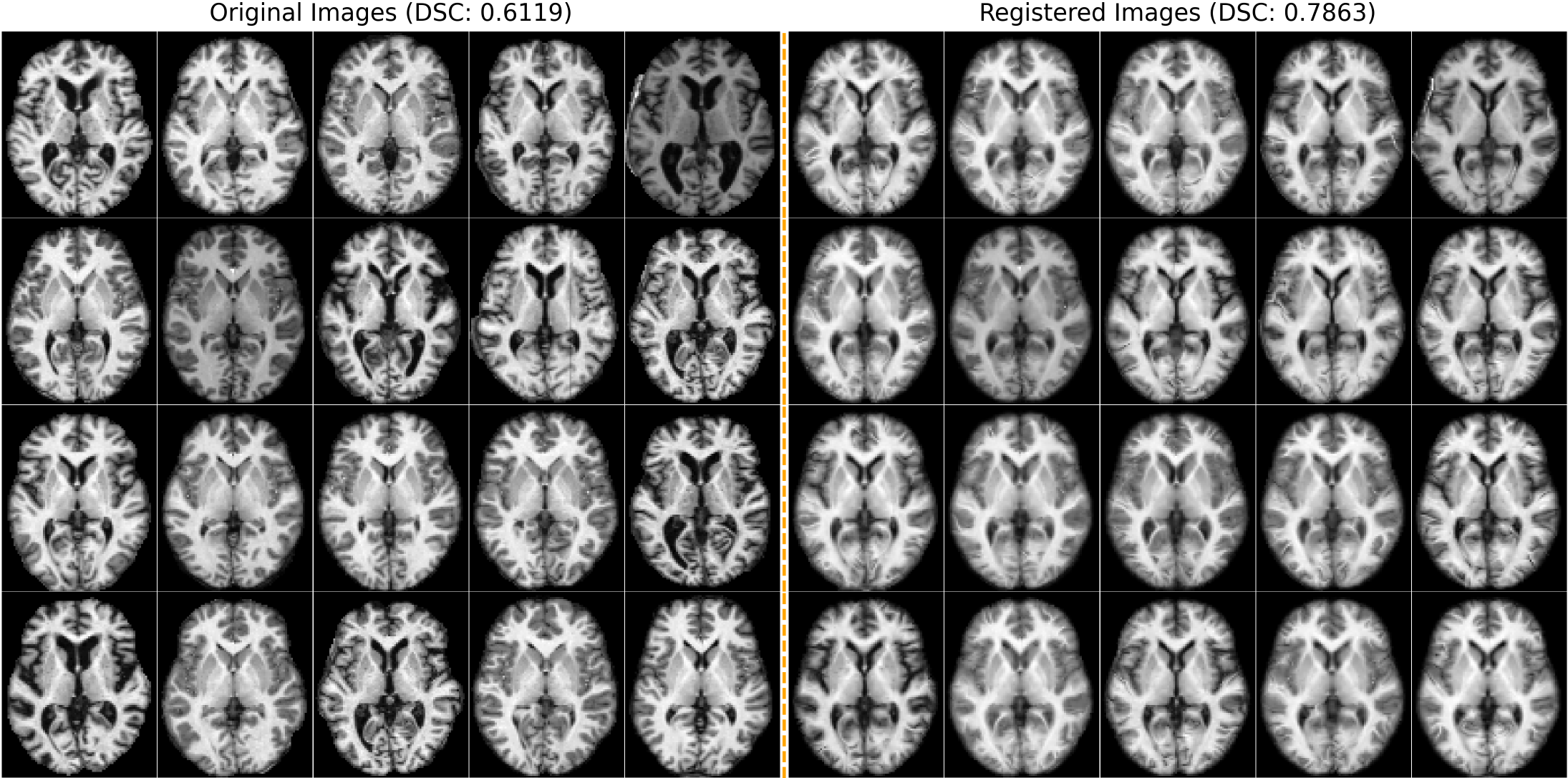}
    \caption{An example of co-registration by our model (Ours-CD) for a group of 20 3D images from the OASIS dataset.}
    \label{fig:oasis_scalability}
\end{figure*}

\begin{figure*}
    \centering
    \includegraphics[width=0.8\textwidth]{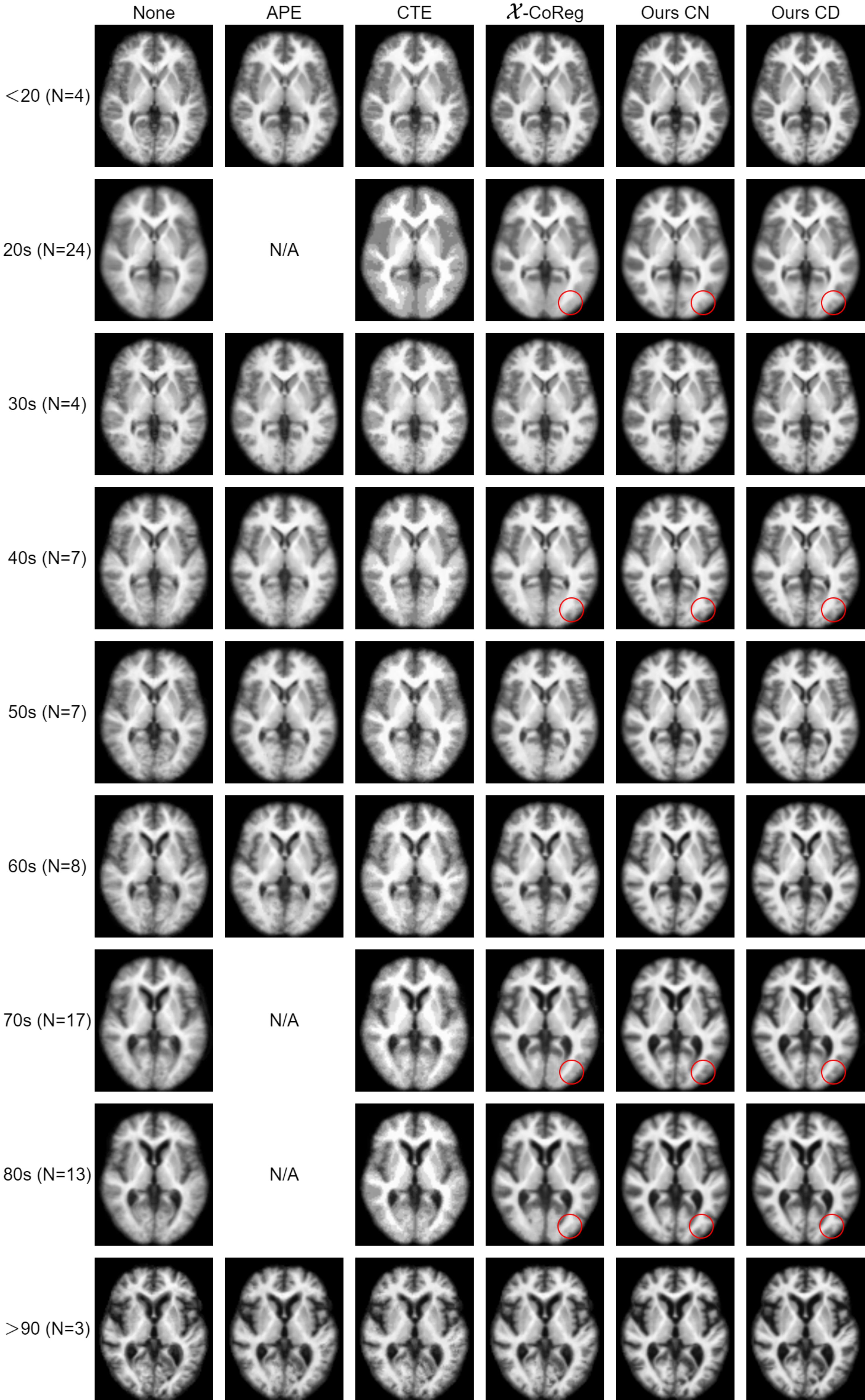}
    \caption{The mean registered images of each age group obtained from different groupwise registration methods.
    One can observe that the proposed models could preserve more anatomical details.
    Note that APE was not available on large image groups due to its quadratic computational complexity relative to the group sizes.}
    \label{fig:oasis_age_group}
\end{figure*}

\begin{figure*}[t]
  \centering
  \begin{subfigure}{\textwidth}
    \centering
    \includegraphics[width=\textwidth]{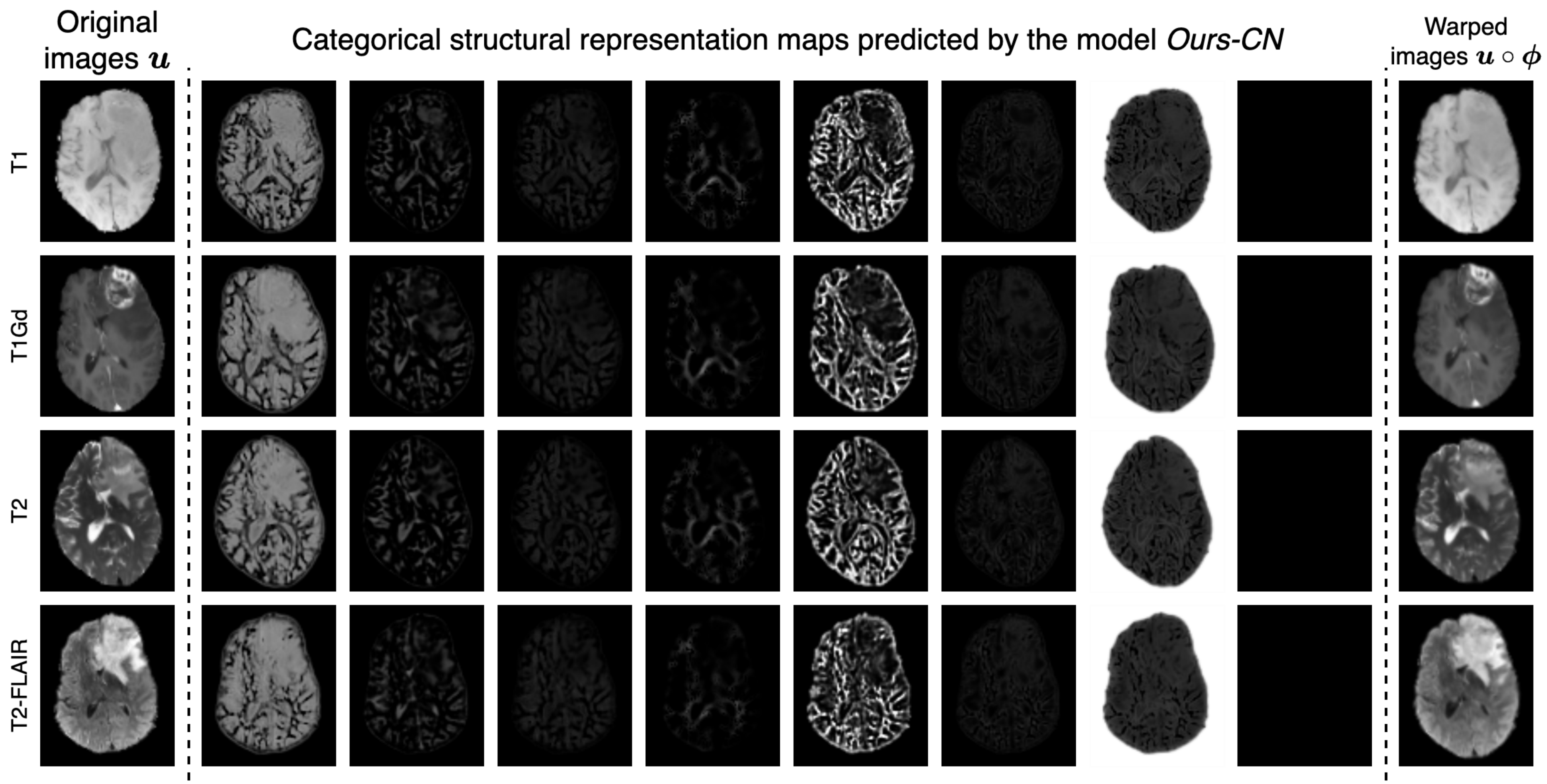}
  \end{subfigure}
  \\
  \vspace{0.5cm}
  \begin{subfigure}{\textwidth}
    \centering
    \includegraphics[width=\textwidth]{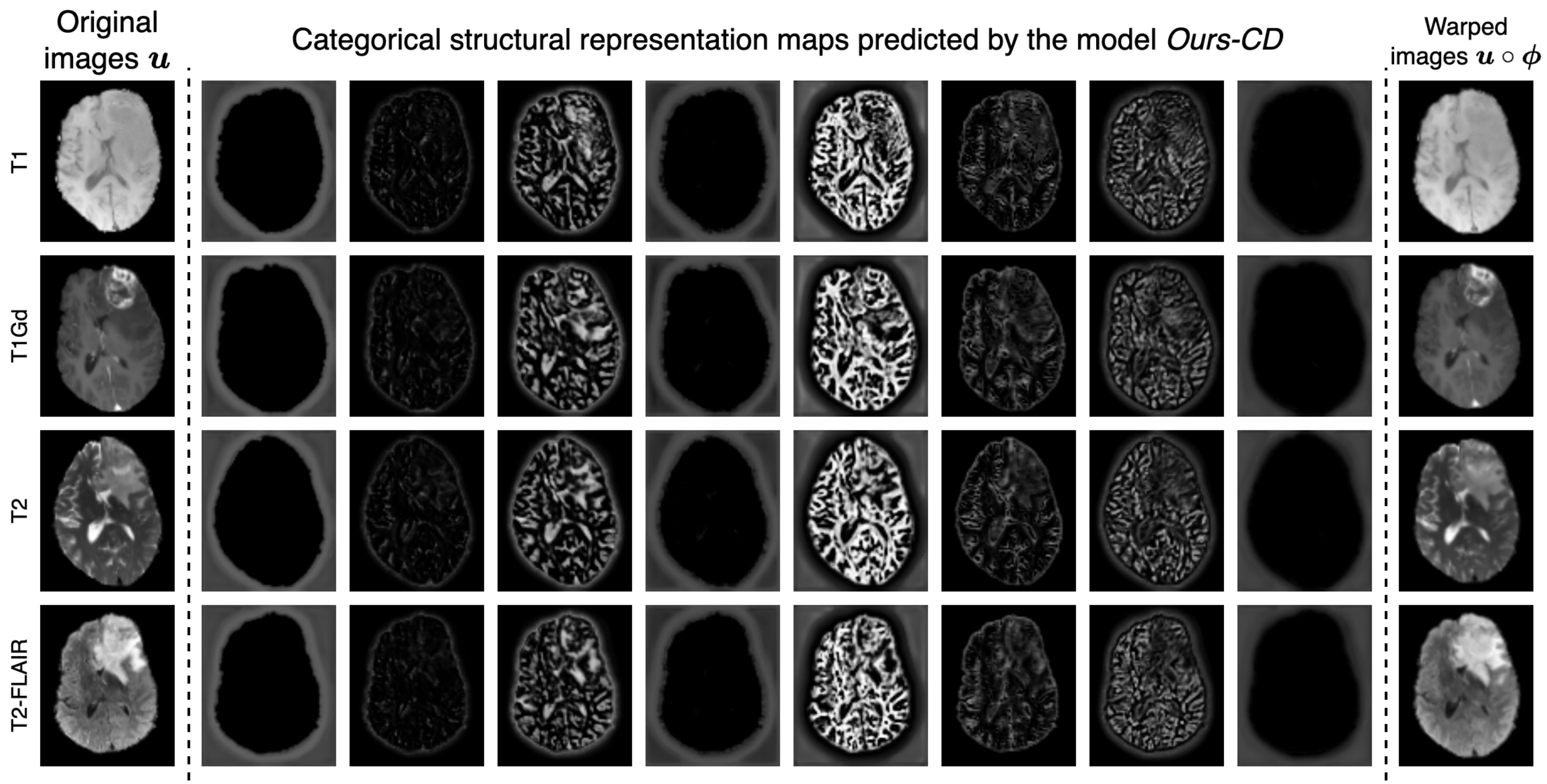}
  \end{subfigure}
  \caption{Categorical structural representations extracted by the proposed models on an image group from the BraTS dataset.}
  \label{fig:features_brats_sup}
\end{figure*}

\begin{figure*}[t]
\begin{subfigure}{\textwidth}
  \centering
  \includegraphics[width=\textwidth]{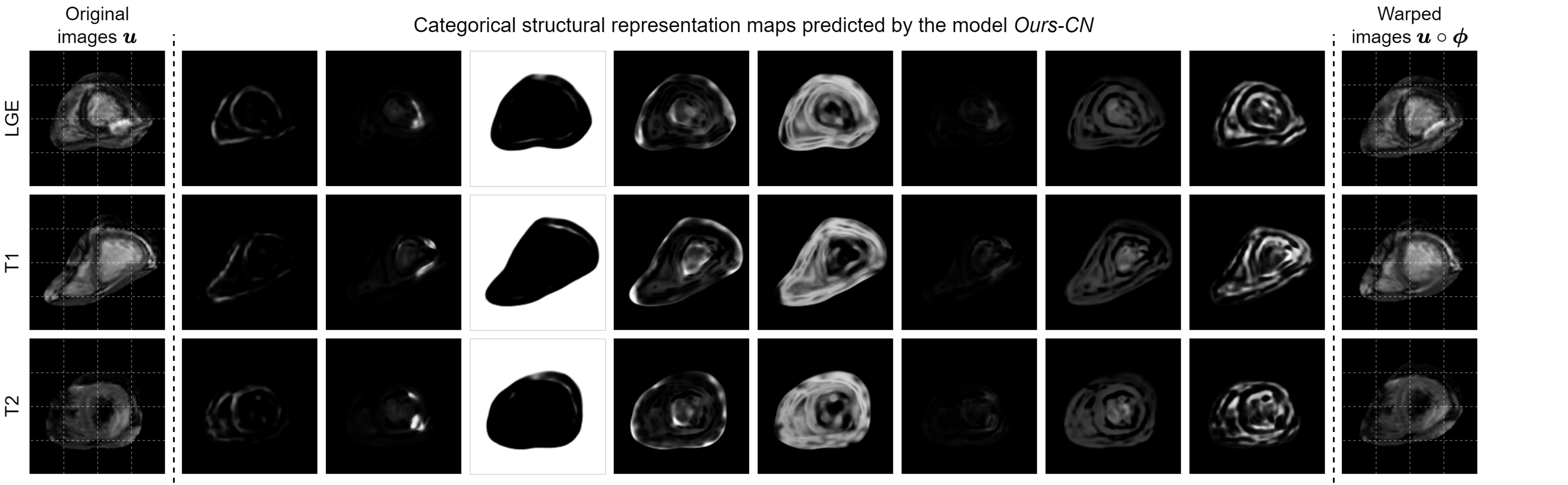}
\end{subfigure}
\\
\vspace{0.5cm}
\begin{subfigure}{\textwidth}
  \centering
  \includegraphics[width=\textwidth]{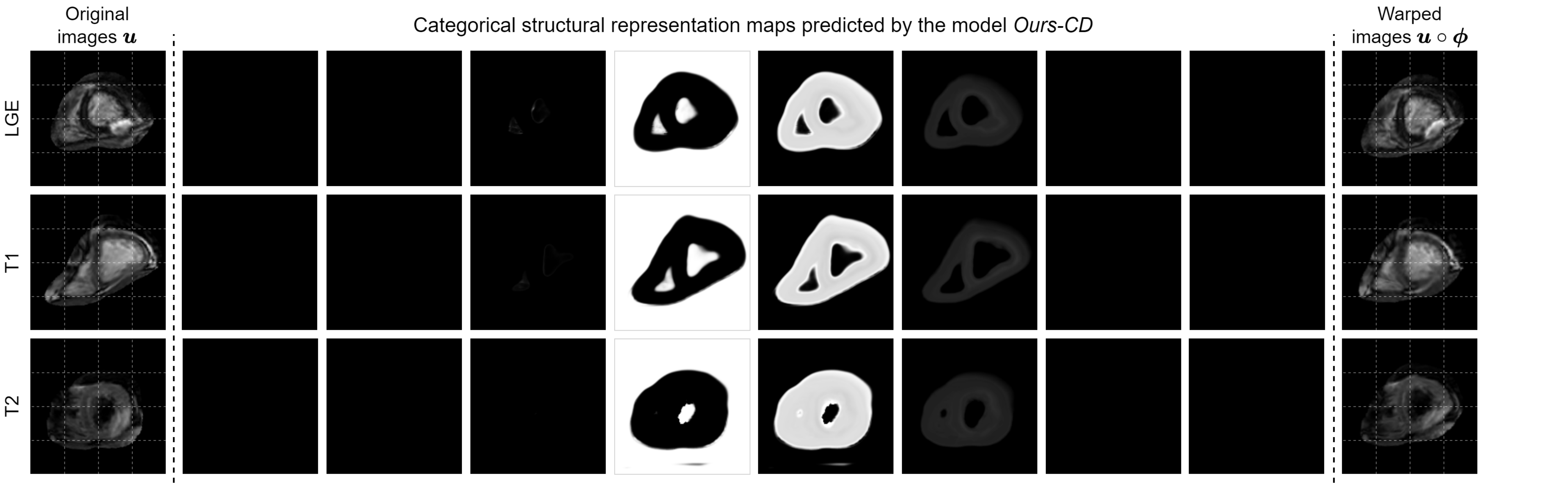}
\end{subfigure}
\caption{Categorical structural representations extracted by the proposed models on an image group from the MS-CMRSeg dataset.}
\label{fig:features_mscmr}
\end{figure*}

\begin{figure*}
\begin{subfigure}{\textwidth}
  \centering
  \includegraphics[width=\textwidth]{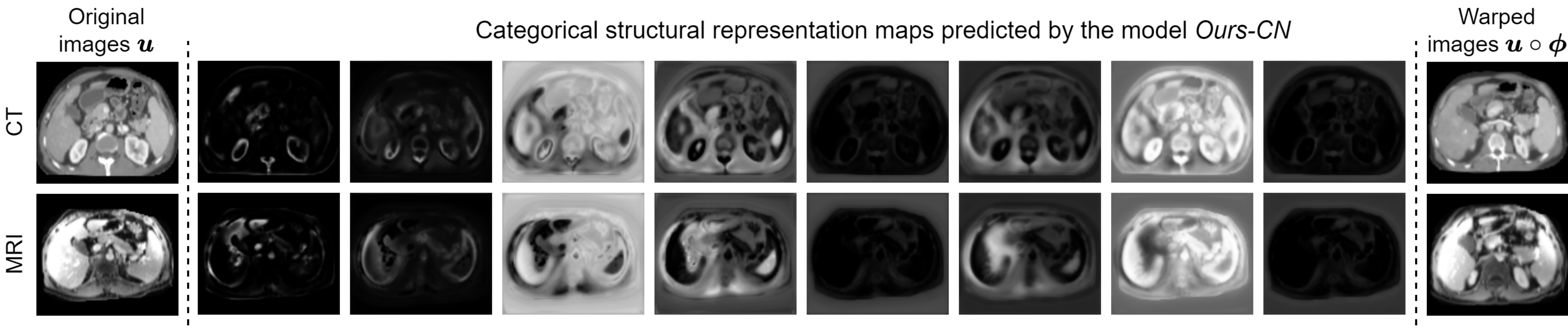}
\end{subfigure}
\\
\vspace{0.5cm}
\begin{subfigure}{\textwidth}
  \centering
  \includegraphics[width=\textwidth]{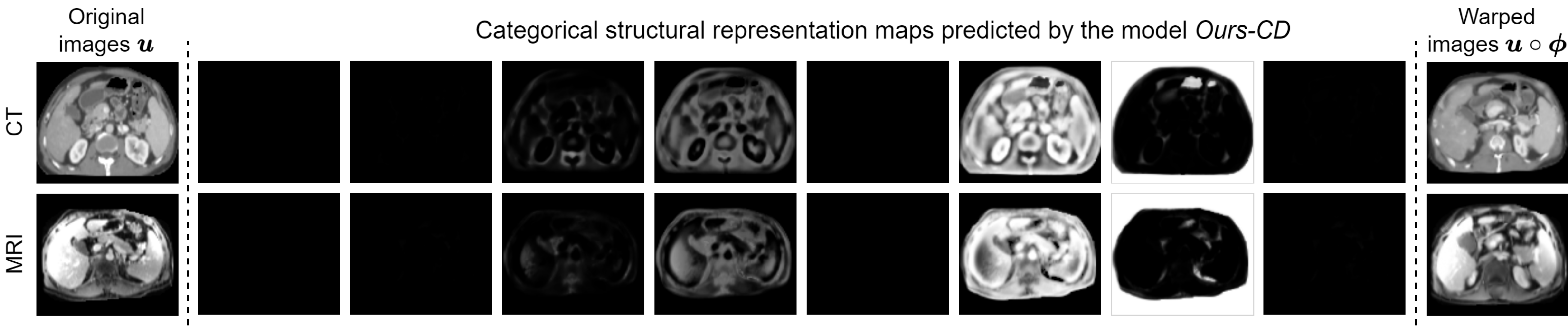}
\end{subfigure}
\caption{Categorical structural representations extracted by the proposed models on an image group from the Learn2Reg Abdominal MR-CT dataset.}
\end{figure*}

\begin{figure*}
\begin{subfigure}{\textwidth}
  \centering
  \includegraphics[width=\textwidth]{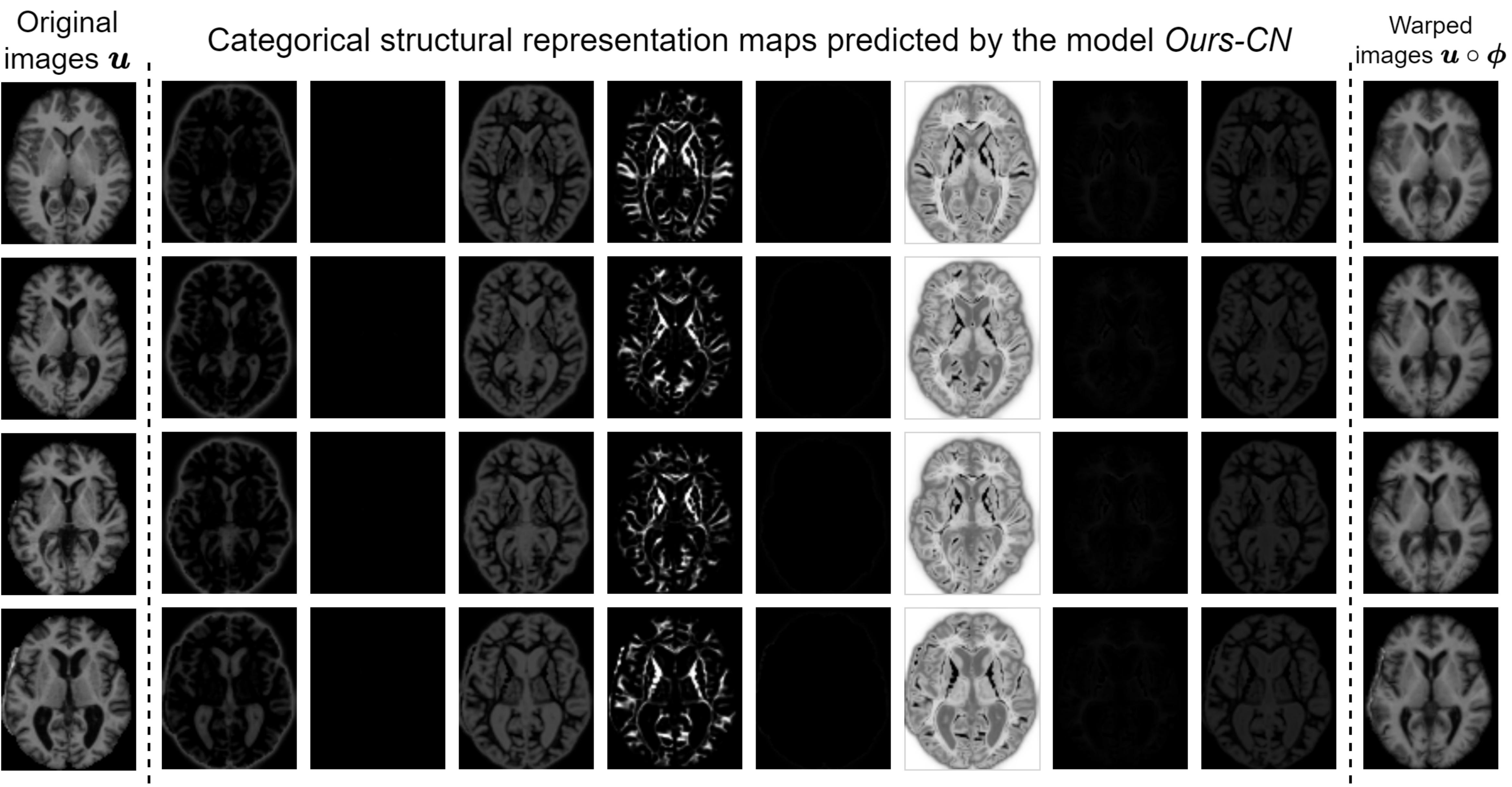}
\end{subfigure}
\\
\vspace{0.5cm}
\begin{subfigure}{\textwidth}
  \centering
  \includegraphics[width=\textwidth]{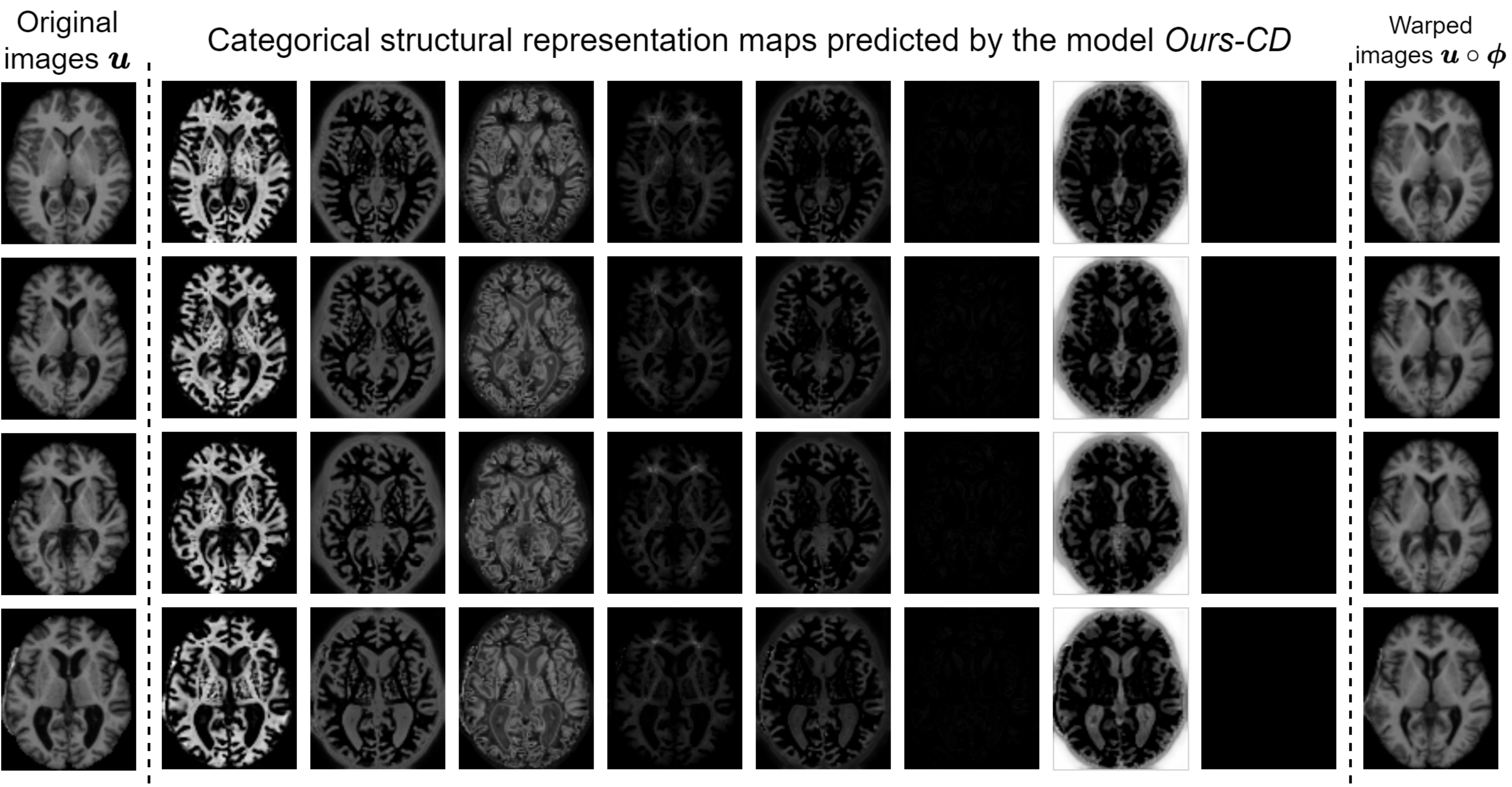}
\end{subfigure}
\caption{Categorical structural representations extracted by the proposed models on an image group from the OASIS dataset.}
\end{figure*}

 \begin{figure*}[t]
  \centering
  \includegraphics[width=\textwidth]{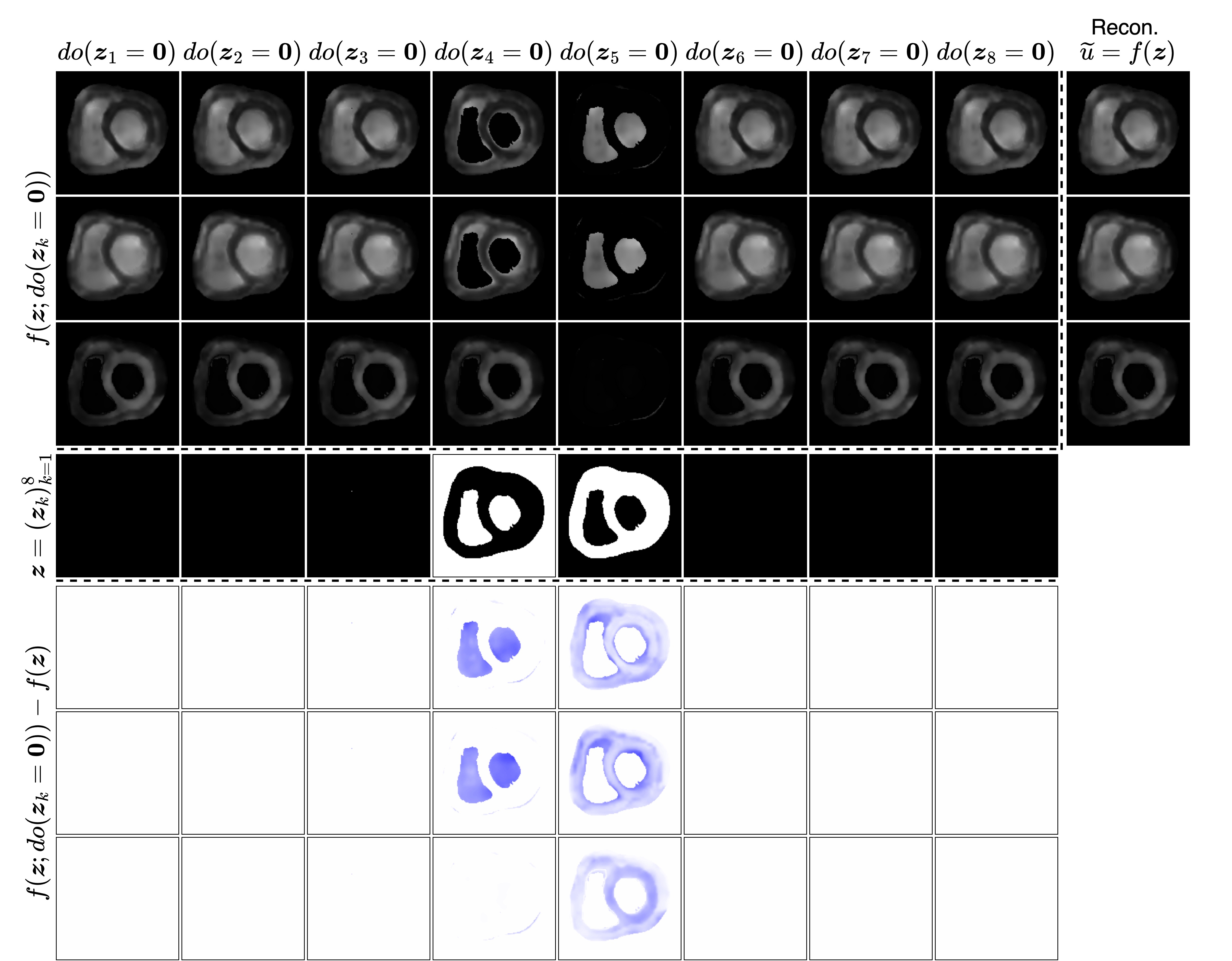}
  \caption{Counterfactual reconstruction of an image group from the MS-CMRSeg dataset by ontological transformations.
  }
  \label{fig:mscmr_counterfactual}
\end{figure*}

\begin{table*}[t]
\caption{The time costs of different methods for the MS-CMRSeg dataset. The times and numbers of epochs correspond to reaching the best DSC by each method itself, 95\% of the best DSC, or the best DSC by the best baseline (CTE-Att). For the three iterative methods APE, CTE and $\mathcal{X}$-CoReg, the values correspond to completing all iterations on the test images; for other learning-based methods, the values correspond to network training. d: days, h: hours, m: minutes.}
\label{tab:mscmr_time}
\centering
\begin{tabular}{ccccccccc}
\hline
\multirow{2}{*}{Method} & \multicolumn{2}{c}{Best DSC} &  & \multicolumn{2}{c}{95\% DSC} &  & \multicolumn{2}{c}{Best Baseline DSC} \\ \cline{2-3} \cline{5-6} \cline{8-9} 
                        & Epoch        & Time          &  & Epoch        & Time          &  & Epoch             & Time              \\ \hline
APE                     & N/A          & 13h 29m       &  & N/A          & N/A           &  & N/A               & N/A               \\
CTE                     & N/A          & 11h 45m       &  & N/A          & N/A           &  & N/A               & N/A               \\
$\mathcal{X}$-CoReg     & N/A          & 7h 47m        &  & N/A          & N/A           &  & N/A               & N/A               \\
APE-Att                 & 391          & 3d 1h         &  & 2            & 22m           &  & N/A               & N/A               \\
CTE-Att                 & 151          & 1d 1h         &  & 1            & 10m           &  & 151               & 1d 1h             \\ \hdashline
Ours-PN                 & 202          & 1d 22h        &  & 26           & 6h 5m         &  & N/A               & N/A               \\
Ours-CN                 & 135          & 1d 18h        &  & 9            & 2h 52m        &  & 17                & 5h 34m            \\
Ours-PD                 & 245          & 2d 10h        &  & 13           & 3h 6m         &  & N/A               & N/A               \\
Ours-CD                 & 88           & 1d 4h         &  & 7            & 2h 30m        &  & 15                & 5h 5m             \\ \hline
\end{tabular}
\end{table*}

\begin{table*}[h]
\caption{The training time costs of different variants of our model on the MS-CMRSeg dataset. The times and numbers of epochs correspond to reaching the best DSC by each method itself, 95\% of the best DSC, or the best DSC by the best baseline (CTE-Att). d: days, h: hours, m: minutes. *: using a learning rate of $10^{-4}$ (otherwise $10^{-3}$).}
\footnotesize
\centering
\label{tab:training_time}
\begin{tabular}{ccccccccccc}
\hline
\multirow{2}{*}{Method} & \multirow{2}{*}{Encoder}&\multirow{2}{*}{\makecell{Reg. \\Module}}& \multicolumn{2}{c}{\makecell{Best\\DSC}} &  & \multicolumn{2}{c}{\makecell{95\% \\DSC}} &  & \multicolumn{2}{c}{\makecell{Best \\Baseline DSC}} \\ \cline{4-5} \cline{7-8} \cline{10-11} 
                       && & Epoch        & Time          &  & Epoch        & Time          &  & Epoch             & Time              \\ \hline

Ours-CN  & Att-UNet& Convs              & 135          & 1d 18h        &  & 9            & 2h 52m        &  & 17                & 5h 34m            \\
Ours-CD  &Att-UNet& Convs               & 88           & 1d 4h         &  & 7            & 2h 30m        &  & 15                & 5h 5m             \\ \hdashline
Ours-CN  &TransMorph& Convs& 360 & 1d 19h  &  & 10   & 1h 12m  &  & 27  & 3h 25m \\
Ours-CD  &TransMorph& Convs& 346 & 1d 23h  &  & 14   & 1h 57m  &  & 33  & 4h 36m \\
Ours-CN  &Att-UNet& PIViT& 821* & 6d 12h*  &  & 96*   & 16h 43m*  &  & 237*  & 42h 56m* \\
Ours-CD  &Att-UNet& PIViT& 837* & 4d 13h*  &  & 58*   & 7h 34m*  &  & 166*  & 21h 36m* \\
Ours-CN  &Att-UNet& ModeT& 238 & 1d 13h  &  & 7   & 1h 3m  &  & 14  & 2h 6m \\

\hline
\end{tabular}
\end{table*}

\begin{table}[h]
\caption{The inference time costs of different variants of our model on the MS-CMRSeg dataset. The test set contains 880 image groups. Inference was conducted with a batch size of 20, which consumed less than 5 GB of GPU memory.}
\footnotesize
\centering
\label{tab:inference_time}
\begin{tabular}{cccccccccccc}
\hline
\multirow{2}{*}{Method} & \multirow{2}{*}{Encoder}&\multirow{2}{*}{\makecell{Reg. \\Module}}& \multirow{2}{*}{Total Time} & \multirow{2}{*}{\makecell{Average Time \\per Image Group}} \\ \\ \hline

Ours-CN  & Att-UNet& Convs              & 38.57s  & 0.04s      \\
Ours-CD  &Att-UNet& Convs               &  39.45s & 0.04s          \\ \hdashline
Ours-CN  &TransMorph& Convs& 36.39s & 0.04s \\
Ours-CD  &TransMorph& Convs&  37.73s & 0.04s\\
Ours-CN  &Att-UNet& PIViT*& 34.22s & 0.04s\\
Ours-CD  &Att-UNet& PIViT*& 35.83s & 0.04s \\
Ours-CN  &Att-UNet& ModeT& 35.39s & 0.04s\\

\hline
\end{tabular}
\end{table}
  
\begin{table}[h]
\caption{Ablation study on the effect of reconstruction fidelity on registration performance. }
\footnotesize
\centering
\label{tab:recon_ablation}
\begin{tabular}{ccccccccccc}
\hline
\multirow{2}{*}{Kernel Size} & \multirow{2}{*}{\makecell{Weight of \\Reconstruction Loss}} & \multirow{2}{*}{DSC $\uparrow$} & \multirow{2}{*}{ASSD $\downarrow$}\\ \\ \hline
1 & 120 & $0.836\pm 0.043$ & $2.21\pm0.47$\\\hdashline
3 & 120 & $0.835\pm 0.045$ & $2.29\pm0.52$\\
1 & 80 & $0.836\pm 0.045$ & $2.26\pm0.51$\\
1 & 55 & $0.819\pm 0.048$ & $2.56\pm0.52$\\
1 & 40 & $0.793\pm 0.052$ & $2.86\pm0.55$\\

\hline
\end{tabular}
\end{table}

\end{appendices}

\begin{scriptsize}
  \bibliographystyle{IEEEtranN}
  \bibliography{main}
\end{scriptsize}

\begin{IEEEbiography}
[{\includegraphics[width=1in,height=1.25in,clip,keepaspectratio]{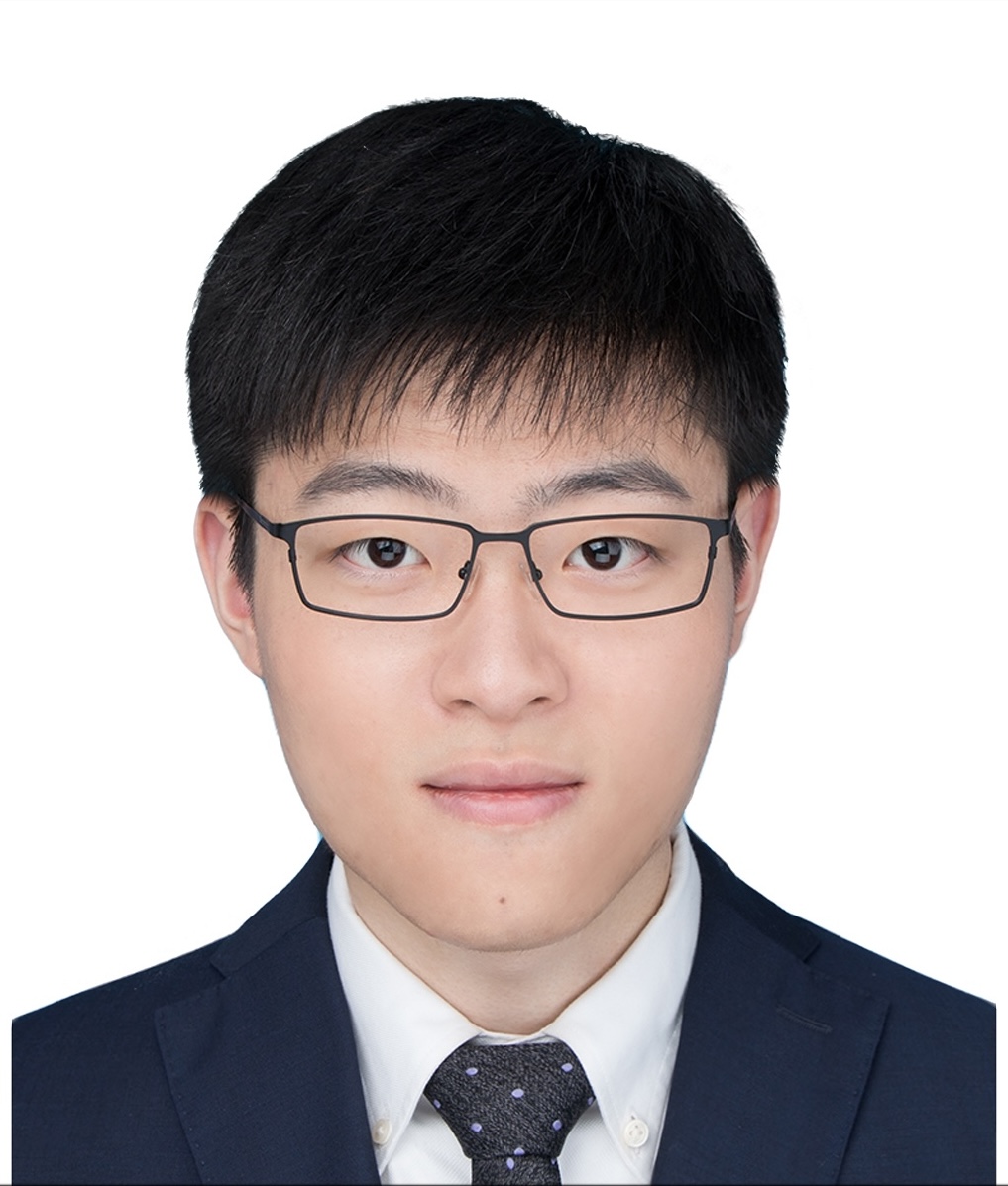}}]
{Xinzhe Luo}
  received his BSc degree in Mathematics and PhD degree in Statistics from Fudan University in 2019 and 2024, respectively.
  He is currently a postdoctoral researcher working at the Department of Electrical and Electronic Engineering and I-X, Imperial College London.
  His research interests include machine learning, medical imaging and probabilistic deep learning.
\end{IEEEbiography}

\begin{IEEEbiography}
[{\includegraphics[width=1in,height=1.25in,clip,keepaspectratio]{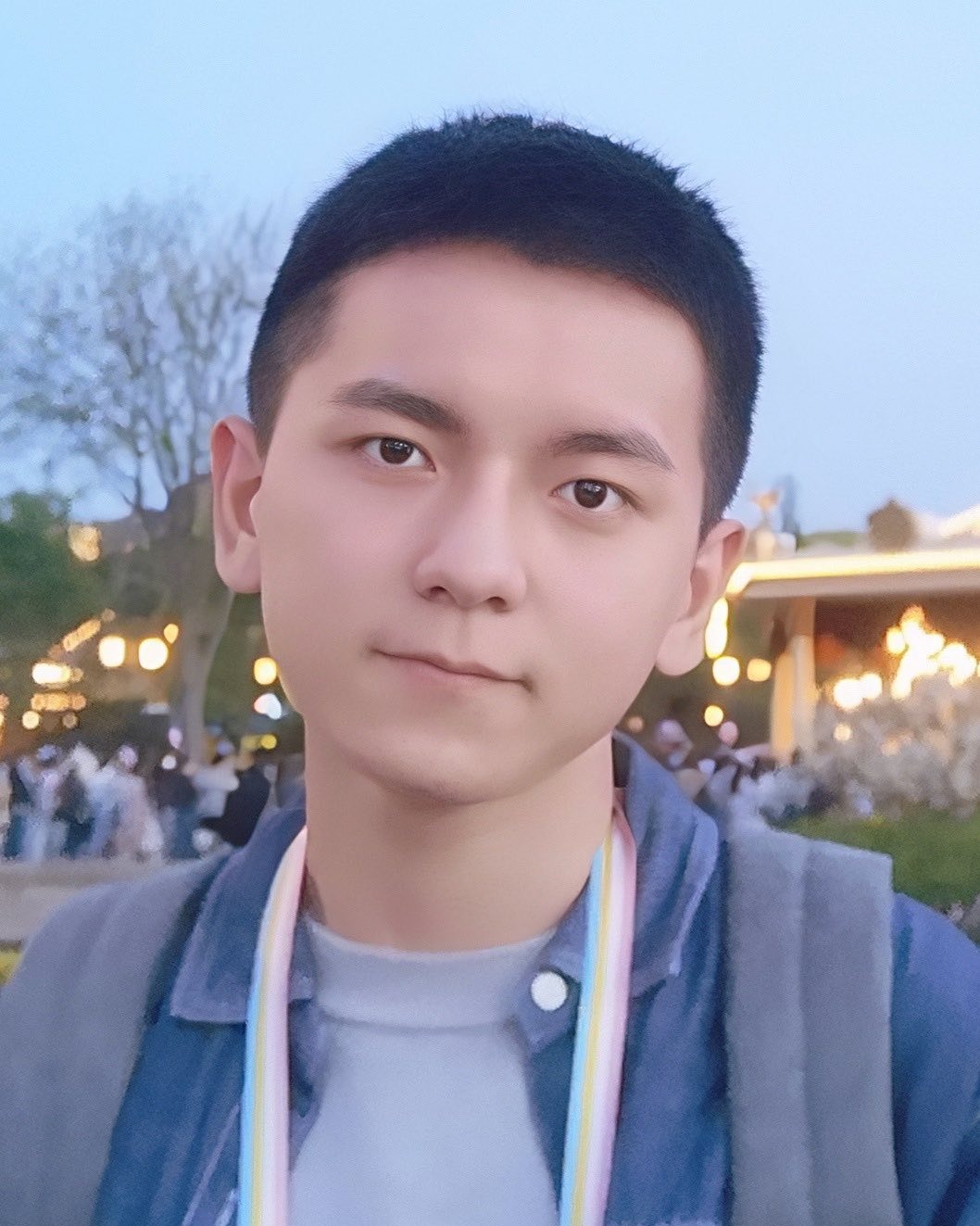}}]
{Xin Wang} received the BS degree in electrical engineering from Fudan University in 2020. He is currently pursuing a PhD degree in electrical and computer engineering at the University of Washington, supervised by Prof. Linda Shapiro and Prof. Chun Yuan. His research is centered around medical image analysis and machine learning, with a focus on statistical models for multi-modal image fusion.
\end{IEEEbiography}

\begin{IEEEbiography}
[{\includegraphics[width=1in,height=1.25in,clip,keepaspectratio]{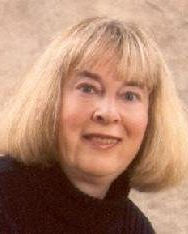}}]
{Linda Shapiro} is the Boeing Endowed Professor in Computer Science and Engineering, Professor of Electrical and Computer Engineering, and Adjunct Professor of Medical Education and Biomedical Informatics at the University of Washington. 
Her research interests
include computer vision, image database
systems, pattern recognition, and medical imaging.
She is a fellow of the IEEE and the IAPR,
and a past chair of the IEEE Computer Society
Technical Committee on Pattern Analysis and
Machine Intelligence.
\end{IEEEbiography}

\begin{IEEEbiography}
[{\includegraphics[width=1in,height=1.25in,clip,keepaspectratio]{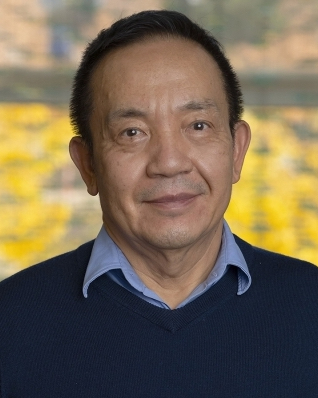}}]
{Chun Yuan} is a professor in the Department of Radiology and Imaging Sciences at the University of Utah, and a professor emeritus in the Department of Radiology at the University of Washington. He is a fellow of the ISMRM, AHA, and AIMBE. He has pioneered multiple high-resolution MRI techniques to detect vulnerable atherosclerotic plaques and led numerous MRI studies examining carotid atherosclerosis. He is a member of the editorial board for the following Journals: JACC CV imaging, Journal of Cardiovascular MR, and the Journal of Geriatric Cardiology. 
\end{IEEEbiography}

\begin{IEEEbiography}
[{\includegraphics[width=1in,height=1.25in,clip,keepaspectratio]{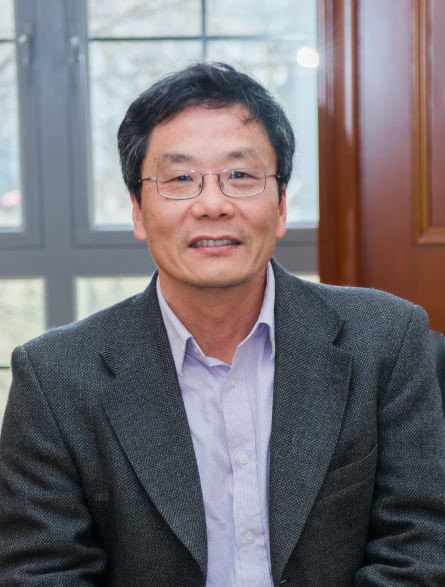}}]
{Jianfeng Feng}
  is the chair professor of Shanghai Center for Mathematical Sciences, and the Dean of Institute of Science and Technology for Brain-Inspired Intelligence and School of Data Science in Fudan University. He has been developing new mathematical, statistical and computational theories and methods to meet the challenges raised in neuroscience and mental health researches. Recently, his research interests are mainly in big data analysis and mining for neuroscience and brain diseases. He was awarded the Royal Society Wolfson Research Merit Award in 2011, as a scientist ‘being of great achievements or potentials’, and the Humboldt Research Award in 2023, for his significant contribution in psychiatry research and pioneering work in Digital Twin Brain.
\end{IEEEbiography}

\begin{IEEEbiography}
[{\includegraphics[width=1in,height=1.25in,clip,keepaspectratio]{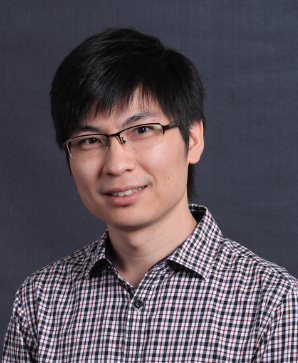}}]
{Xiahai Zhuang}
  is a professor at the School of Data Science, Fudan University. He graduated from Department of Computer Science, Tianjin University, received Master degree from Shanghai Jiao Tong University and Doctorate degree from University College London. His research interests include interpretable AI, medical image analysis and computer vision. His work won the Elsevier-MedIA 1st Prize and Medical Image Analysis MICCAI Best Paper Award 2023.
\end{IEEEbiography}

\end{document}